\documentclass{article}
 
\usepackage[final,nonatbib]{neurips_2023}

\usepackage[utf8]{inputenc} 
\usepackage[T1]{fontenc}    
\usepackage[colorlinks,citecolor=blue,urlcolor=blue,linkcolor=blue,linktocpage=true]{hyperref}

\usepackage{url}            
\usepackage{booktabs}       
\usepackage{amsfonts}       
\usepackage{nicefrac}       
\usepackage{microtype}      
\usepackage{xcolor}         

\usepackage{algorithm}
\usepackage{algorithmicx}
\usepackage[noend]{algpseudocode}
\usepackage{amsmath, amsfonts, amssymb, amsthm, amsbsy, amscd, bm, bbm,mathrsfs}    \usepackage{graphicx}
\usepackage{subcaption}
\usepackage{cleveref}
\usepackage{enumitem}
\setlist{leftmargin=5mm}
\usepackage{titlesec}
\usepackage{wrapfig}
\usepackage{multirow}

\usepackage{soul}
\newcommand{\Note}[1]{}
\renewcommand{\Note}[1]{\hl{[#1]}}

\graphicspath{{figs/}}
\sethlcolor{yellow}

\numberwithin{equation}{section}

\newtheorem{theorem}{Theorem}
\newtheorem{othertheorem}{othertheorem}[section]
\newtheorem{lemma}[othertheorem]{Lemma}

\theoremstyle{definition}
\newtheorem{definition}[othertheorem]{Definition}
\newtheorem{remark}[othertheorem]{Remark}
\newtheorem{assumption}[othertheorem]{Assumption}
\theoremstyle{definition}

\newtheorem{fact}[othertheorem]{Fact}

\renewcommand{\v}{\bm{v}}
\newcommand{\g}{\bm{g}}

\newcommand{\I}{\mathbf{I}}

\newcommand{\R}{\mathbb{R}}
\newcommand{\w}{\bm{w}}

\newcommand{\m}{\mathbf{m}}

\newcommand{\E}{\mathbb{E}}
\renewcommand{\P}{\mathbb{P}}
\newcommand{\M}{\mathcal{M}}

\title{Automatic Clipping: Differentially Private Deep Learning Made Easier and Stronger}

\author{%
  Zhiqi Bu \\
  AWS AI\\
  \texttt{zhiqibu@amazon.com}
  \\
  \And
  Yu-Xiang Wang\\
AWS AI, UC Santa Barbara\\
  \texttt{yuxiangw@cs.ucsb.edu}
  \And
  Sheng Zha\\
  AWS AI\\
  \texttt{zhasheng@amazon.com}  
  \And
 George Karypis\\
  AWS AI\\
  \texttt{gkarypis@amazon.com}  
}

\begin{document}

\maketitle



\begin{abstract}


Per-example gradient clipping is a key algorithmic step that enables practical differential private (DP) training for deep learning models. The choice of clipping threshold $R$, however, is vital for achieving high accuracy under DP. We propose an easy-to-use replacement, called automatic clipping, that eliminates the need to tune $R$ for any DP optimizers, including DP-SGD, DP-Adam, DP-LAMB and many others.
The automatic variants are as private and computationally efficient as existing DP optimizers, but require no DP-specific hyperparameters and thus make DP training as amenable as the standard non-private training. We give a rigorous convergence analysis of automatic DP-SGD in the non-convex setting, showing that it can enjoy an asymptotic convergence rate that matches the standard SGD, under a symmetric gradient noise assumption of the per-sample gradients (commonly used in the non-DP literature). We demonstrate on various language and vision tasks that automatic clipping outperforms or matches the state-of-the-art, and can be easily employed with minimal changes to existing codebases\footnote{Code for our experiments is available at FastDP library \href{https://github.com/awslabs/fast-differential-privacy}{https://github.com/awslabs/fast-differential-privacy}.}.
\end{abstract}

\section{Introduction}
Deep learning has achieved impressive progress in a wide range of tasks. These successes are made available, in part, by the collection of large datasets, sometimes containing sensitive private information of individual data points. Prior works have illustrated that deep learning models pose severe privacy risks to individual subjects in the training data and are susceptible to various practical attacks. For example, machine learning services such as Google Prediction API and Amazon
Machine Learning can leak membership information from the purchase records \cite{shokri2017membership}; the GPT2 language models auto-complete texts that contain someone's full name, phone number, email address, etc., from the training data that it memorizes, if invoked by specific prefixes \cite{carlini2021extracting}.

Differential privacy (DP) \cite{dwork2008differential,dwork2014algorithmic,dwork2006calibrating} is a formal definition of privacy that has been shown to prevent the aforementioned privacy risks in deep learning \cite{abadi2016deep}. At a high level, the key difference between the DP deep learning and the standard one is whether the gradient is privately released. In other words, while the standard optimizers update on $\sum_i \g_i$, the DP optimizers update on the \textit{private gradient}:
\vspace{-0.2cm}
\begin{align}
\text{DP Optimizer} (\{\g_i\}_{i=1}^B )&=\text{Optimizer}(\overbrace{\sum\nolimits_i \g_i\cdot \texttt{Clip}(\|\g_i\|;R)+\sigma R\cdot\mathcal{N}(0,\I)}^{\text{private gradient}})
\label{eq:DP optimizers}
\\
\text{Standard Optimizer} (\{\g_i\}_{i=1}^B )&=\text{Optimizer}(\sum\nolimits_i \g_i)
\label{eq:NonDP optimizers}
\end{align}
Here $\g_i\in\R^d$ is the per-sample gradient of loss $l_i$, $\mathcal{N}$ is the standard normal, $\sigma$ is the noise multiplier, and $R$ is the clipping threshold. The clipping function $\texttt{Clip}:\R^d\to\R$ is defined such that $\|\g_i\cdot \texttt{Clip}(\g_i;R)\|\leq R$. For instance, the DP-SGD in \cite{abadi2016deep} is 
\begin{align}
\text{DP-SGD}_\text{Abadi}:\quad
\w_{t+1}=\w_t-\eta\Big(\sum\nolimits_i \frac{\partial l_i}{\partial \w_t}\min\Big(R/\Big\|\frac{\partial l_i}{\partial \w_t}\Big\|,1\Big)+\sigma R\cdot\mathcal{N}(0,\I)\Big)
\label{eq:abadi dpsgd}
\end{align}

\vspace{-0.2cm}
In comparison to the regular training \eqref{eq:NonDP optimizers}, two additional DP-specific hyperparameters $R$ and $\sigma$ need to be determined in DP learning \eqref{eq:DP optimizers}. On the one hand, setting the noise multiplier $\sigma$ is easy and can be derived analytically prior to the training. Whenever the privacy budget $(\epsilon,\delta)$ is determined, one can apply off-the-shelf privacy accounting tools in \Cref{sec: dp intro} to determine $\sigma$, based on the subsampling probability $p$ 
and the number of iterations $T$:

\vspace{-0.4cm}
$$\text{privacy\_accountant}(\sigma, p, T;\delta)=\epsilon$$

\vspace{-0.1cm}
On the other hand, the choice of clipping threshold $R$ is crucial to the performance of DP models, yet the hyperparameter tuning is much labor-intensive. Recent advances of DP deep learning on ImageNet \cite{kurakin2022toward} and on E2E datasets \cite{li2021large}, using ResNet18 and GPT2 respectively, illustrate that the performance is very sensitive to $R$. We have reproduced their results in \Cref{fig:my motivation}. Observe that on ImageNet, ResNet18 can drop from the highest 45\% accuracy to 31\% if $R$ is chosen 2 times larger, and to 0.1\% if $R$ is chosen 4 times larger. Similar drastic drop can also be observed in \cite[Figure 3]{kurakin2022toward} even if the noise multiplier $\sigma=0$. Unlike the noise multiplier $\sigma$, the clipping threshold $R$ cannot be inferred from the privacy budget $(\epsilon,\delta)$ and have to be tuned. Consequently, DP training necessarily requires an expensive 2D grid search for $(R,\eta)$, like \Cref{fig:my motivation}, whereas the regular training only requires an easy 1D grid search for $\eta$. Even worse, the difficulty of tuning a per-layer clipping threshold vector \cite{mcmahan2017learning}, i.e. one clipping threshold for one layer, may increase exponentially as the number of layers increases.

To save the effort of tuning $R$, previous researches have proposed different approaches. In \cite{andrew2021differentially,pichapati2019adaclip,golatkar2022mixed,he2022exploring}, researchers advocate to use data-adaptive information to select $R$, such as a specified quantile of the gradient norm distribution.
These adaptive clipping methods can be a little ad-hoc: they often replace the need to tune $R$ by the need to tune one or more new hyperparameters, e.g. the quantile to use and the ratio to split the privacy budget between the quantile decision and the gradient perturbation.
Another approach used by the practitioners is to replace the single 2D grid search by multiple cheaper 1D grid searches. For example, the researchers propose, in \cite[Section 3.3]{kurakin2022toward} to fine-tune $\eta$ with non-DP SGD, fix $\eta$ and sweep over various values of the clipping threshold $R$ with DP-SGD, then further fix $R$ and do one more grid search on $\eta$. However, tuning $R$ formally in a data-dependent way (e.g. through cross-validation) introduces additional privacy loss \cite{papernot2021hyperparameter}, and most existing empirical work does not privately conduct hyperparameter tuning.

We take a completely different route by proposing a new clipping principle that removes $R$, instead of coming up with methods to find the appropriate $R$. We term our method as \emph{automatic clipping} and the DP optimizers using it as \emph{automatic DP optimizers}. Our contributions are:

\vspace{-0.15cm}
\begin{enumerate}
    \item We propose the automatic clipping in \eqref{eq:Auto without R} that expunges the clipping threshold from general DP optimizers, making DP training as amenable as regular training. In large-scale tasks (GPT-level) like \Cref{fig:my motivation}, our automatic clipping can reduce the cost of ablation study by $5\times$\footnote{The hyperparameter tuning of $(R,\eta)$ takes days (e.g. GPT2 \cite{li2021large}) to months (e.g. GPT3-175B) on large foundation models, highlighting the significance of our method to expunge the additional $R$.}.
    \item We show that automatic DP optimizers are as private and efficient as existing DP optimizers.
    \item We show in \Cref{thm:upper bounding grad norm without r} that automatic DP-SGD converges in the non-convex setting, at the same asymptotic convergence rate as the standard SGD. Our theoretical analysis successfully explains the training behaviors of deep learning in previous empirical works.
    \item We demonstrate the superiority of automatic clipping on a variety of vision and language tasks, especially with large models including ResNet, RoBERTa and GPT2.
    \item In \Cref{app:codebase}, we include simple code snippets that demonstrate how easy it is to switch from Abadi's clipping to our automatic clipping in popular codebases, e.g. Opacus and ObJAX.
\end{enumerate}

\section{Preliminaries}
\subsection{Differential Privacy}
\label{sec: dp intro}

We consider the $(\epsilon,\delta)$-DP in \Cref{def:DP}, where smaller $(\epsilon,\delta)$ means stronger privacy guarantee.

\begin{definition}[\cite{dwork2006calibrating}]\label{def:DP}
A randomized algorithm $M$ is $ (\varepsilon, \delta)$-differentially private (DP) if for any two neighboring datasets $S,S^{\prime}$ (i.e. if one can obtain $S^\prime$ by adding or removing one data point from $S$), and for any event $E$,
\begin{align}
 \mathbb{P}[M(S) \in E] \leqslant \mathrm{e}^{\varepsilon} \mathbb{P}\left[M\left(S^{\prime}\right) \in E\right]+\delta.
\end{align}
\end{definition}

In words, DP restricts the influence of an arbitrary sample, so that the information contributed by such sample is limited and less vulnerable to privacy attacks. In deep learning, DP is achieved by applying the \emph{subsampled Gaussian mechanism} to privatize the minibatch gradients during training. 

As illustrated in \Cref{eq:DP optimizers}, the subsampled Gaussian mechanism involves (I) sampling a minibatch by including each data point iid with probability $p$ (II) per-sample gradient clipping to bound the $l_2$ norm sensitivity at $R$ and (III) adding independent Gaussian noise proportional to $R$ and $\sigma$, where $\sigma$ is derived from the privacy budget $(\epsilon,\delta)$. This can be realized by leveraging a variety of modern privacy accounting tools, such as Renyi DP (or moments accountant) \cite{abadi2016deep,mironov2017renyi,wang2019subsampled}, Privacy Loss distribution (Fourier accountants) \cite{koskela2020computing,gopi2021numerical,zhu2021optimal}, or Gaussian DP \cite{dong2019gaussian,bu2020deep}. 

\subsection{Differentially Private optimizers with general clipping operations}
\label{subsec:DP optimizer with general clipping}
Privately released stochastic gradients (through the Gaussian mechanism) can be used by various off-the-shelf optimizers, including DP-SGD in \eqref{eq:abadi dpsgd}, DP-HeavyBall, DP-AdaGrad, DP-Adam, DP-FedAvg/FedSGD \cite{mcmahan2017learning}, etc. To improve the performance of DP optimizers, previous researches on the per-sample clipping can be classified into two categories. 

The first category, where the majority of researches lie in, works with Abadi's clipping and focuses on better design of $R$. To name a few examples, one can adaptively design $R_t$ for each iteration $t$ \cite{andrew2021differentially,pichapati2019adaclip,golatkar2022mixed}, or design the per-layer clipping threshold vector $\bm R\in\R^{L}$ for $L$ layers \cite{abadi2016deep,mcmahan2017learning} so as to apply a different clipping threshold for each layer.

Fewer works fall into the second category that proposes new clipping methods. In fact, any function $\texttt{Clip}:\R^d\to\R$ satisfying $\|\texttt{Clip}(\g)\cdot\g\|\leq R$ can serve as a valid clipping function besides Abadi's. For example, the global clipping \cite{bu2021convergence} proposes $\texttt{Clip}_\text{global}(\g_i):=\mathbb{I}(\|\g_i\|<R)$ to mitigate the bias of the private gradient and alleviate the mis-calibration issue of DP classifiers. Another example is the re-parameterized clipping \cite{de2022unlocking}, $\texttt{Clip}_\text{re-param}(\g_i):=\min(1/\|\g_i\|,1/R)$, which is equivalent to Abadi's clipping under a re-scaled learning rate. Our automatic clipping belongs to this category. We note that different clipping methods work orthogonally to optimizers, network architectures and gradient norm computation (see \Cref{sec:discussion}).


\section{Motivation}

\subsection{Small clipping threshold often works best}

\begin{figure}[!htb]
\centering
\vspace{-0.4cm}
    \includegraphics[width=0.36\textwidth]{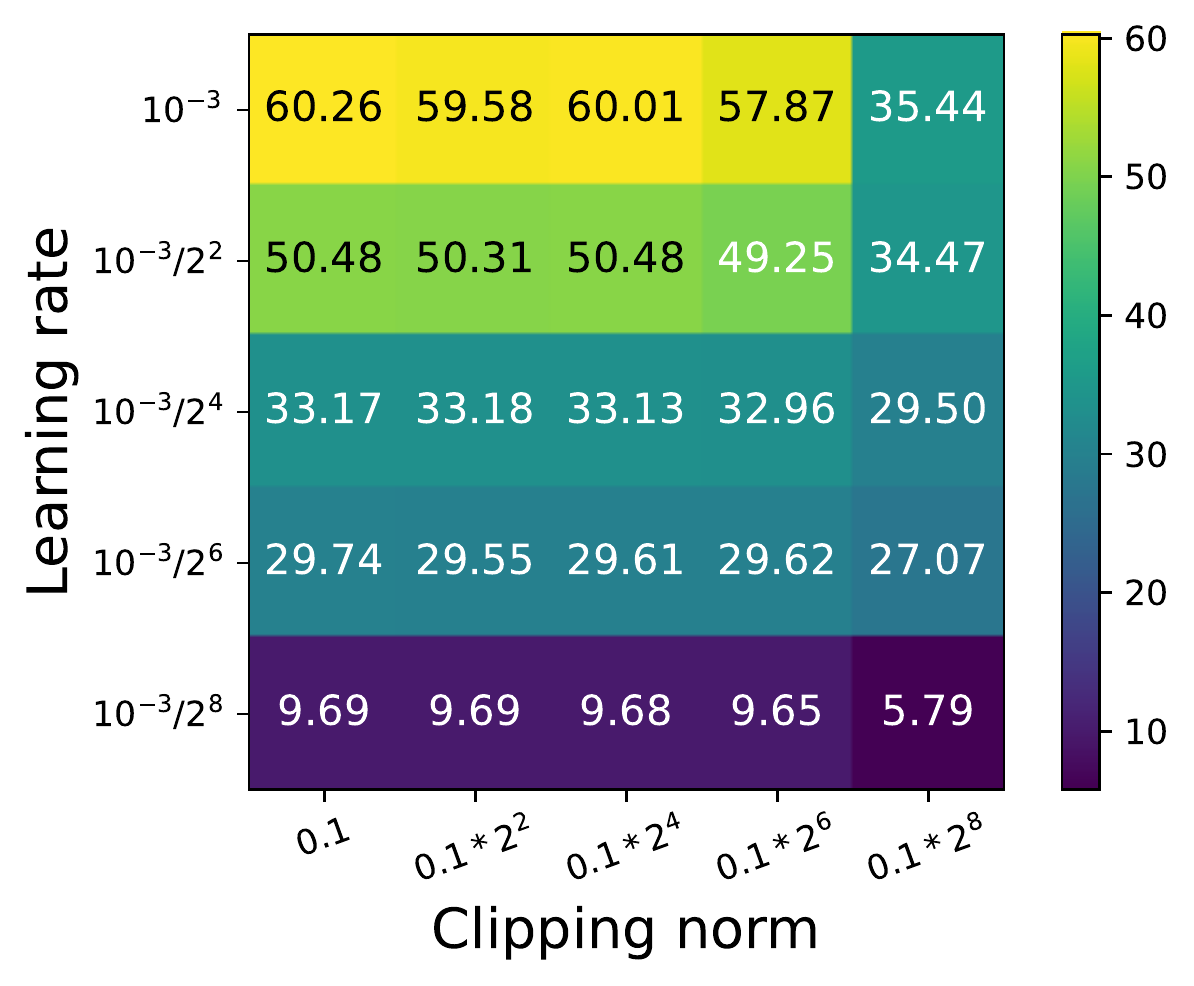}
    \includegraphics[width=0.36\textwidth]{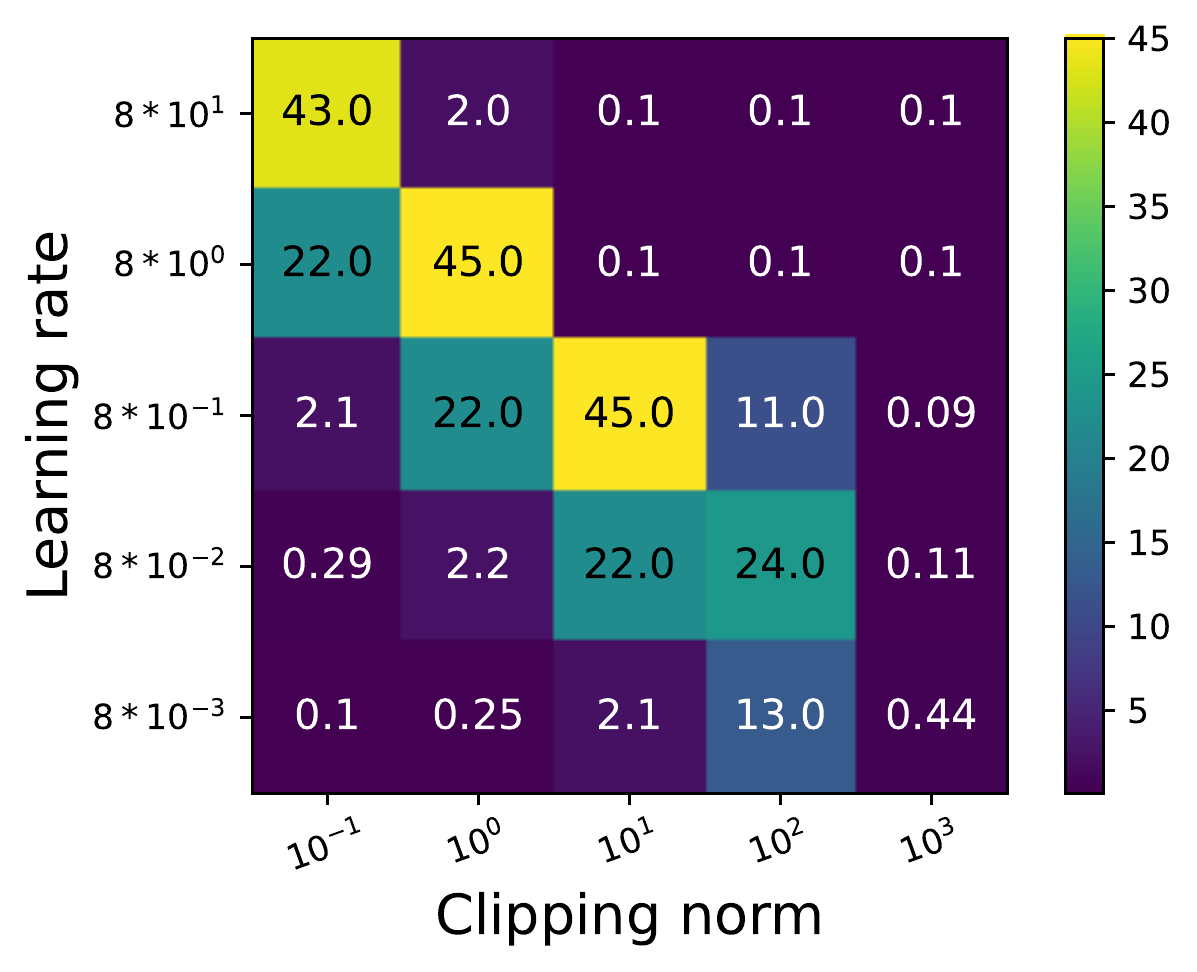}
\vspace{-0.25cm}
    \caption{Ablation study of clipping threshold and learning rate. Left: BLEU score of GPT2 on E2E dataset \cite{li2021large}, with DP-AdamW. Right: Test accuracy of ResNet18 on ImageNet \cite{kurakin2022toward}, with DP-SGD.}
    \label{fig:my motivation}
\end{figure}
\vspace{-0.2cm}
One intriguing observation that we can make about the recent studies on DP learning with large models is that the state-of-the-art (SOTA) results are often achieved with very small clipping threshold $R$.  This observation is consistent in both vision and language tasks. In \cite{li2021large}, GPT2 (about 800 million parameters) and RoBERTa models (over 300 millions parameters) achieve the best results under DP on QNLI, MNLI, SST-2, QQP, E2E, and DART datasets, with each per-sample gradient clipped to length $R=0.1$. In \cite{kurakin2022toward,de2022unlocking,mehta2022large}, ResNets and Vision Transformers achieve the best DP results on ImageNet with $R=1$; in \cite{tramer2020differentially}, the best DP results on CIFAR10 use $R=0.1$ with ResNeXt-29 and SimCLRv2 \cite{chen2020simple}. The effectiveness of small clipping threshold together with proper learning rate is depicted in \Cref{fig:my motivation}.

Intuitively, smaller $R$ implies that the Abadi's clipping \eqref{eq:Auto-V} is effective, which means $\min\big(R/\|\g_i\|,1\big)=R/\|\g_i\|$. Given that the clipping threshold $R$ is so small compared to the number of parameters in large models, and that strong DP is guaranteed when the number of training iterations is small (i.e. $\|\g_i\|$ has not converged to small values yet), we expect and empirically observe that the clipping happens on a large proportion of per-sample gradients at all iterations. For instance, we find in the GPT2 generation experiments in \cite{li2021large} that 100\% of per-sample gradients are clipped at all iterations; in classification tasks such as  QQP, QNLI, and MNLI, the percentage of clipping is about $20\sim60\%$ on average (more details in \Cref{app:freq of clipping}).

\subsection{Per-sample gradient normalization as new clipping}
\label{sec:normal motivation}
In the small clipping threshold regime, we can approximately view
\begin{align}
\texttt{Clip}_\text{Abadi}(\g_i;R)=\min\left(R/||\g_i||,1\right)\approx R/||\g_i||=:\texttt{Clip}_\text{AUTO-V}(\g_i;R)
\label{eq:Auto-V}
\end{align}
and thus derive a novel private gradient $\sum_i R\frac{\g_i}{\|\g_i\|}+\sigma R\cdot\mathcal{N}(0,\I)$. Here AUTO-V stands for the vanilla automatic clipping, which essentially performs the normalization on each per-sample gradient. As a specific example, we can write the $R$-dependent automatic DP-SGD as
\begin{align}
R\text{-dependent } \text{DP-SGD}_\text{AUTO-V}: \w_{t+1}=\w_t-\eta\Big(\sum_{i} R\frac{\partial l_i}{\partial \w_t}/\|\frac{\partial l_i}{\partial \w_t}\|+\sigma R\cdot\mathcal{N}(0,\I)\Big)
\label{eq:R-dep AUTO-V dpsgd}
\end{align}

\vspace{-0.4cm}
We may view our AUTO-V clipping as to maximize the dot-product similarity (a commonly used similarity measure, e.g. in the attention block in transformers \cite{vaswani2017attention}) between the clipped gradient and the regular gradient. Suppose we want to
\begin{align}
    \max_{C_i} \left\langle{\textstyle\sum}_i C_i \g_i,{\textstyle\sum}_j\g_j\right\rangle
    \quad
    \text{ s.t. } 0\leq C_i\leq R/\|\g_i\|
\label{eq:clipping max}
\end{align}

\vspace{-0.15cm}
Note that the constraint is a sufficient condition for clipping, as discussed in \Cref{subsec:DP optimizer with general clipping}. It is not hard to see that the optimal clipping factor (though violating DP guarantee\footnote{In DP literature, per-sample clipping depend only on individual gradient $\g_i$ separately, hence does not allow the use of $\sum_j \g_j$, which changes the sensitivity when adding or removing one data point from the mini-batch.}) regarding \eqref{eq:clipping max} is
\begin{align}
C_i =
R/\|\g_i\|\cdot\mathbb{I}(\langle\g_i,\sum\nolimits_j\g_j\rangle>0),
\end{align}

\vspace{-0.15cm}
If the per-sample gradients are indeed concentrated in the sense $\forall i, \langle \g_i,\sum_j\g_j\rangle\geq 0$, then AUTO-V is the optimal per-sample gradient clipping. We compare with Abadi's clipping in \Cref{fig:dot product}, where this similarity is significantly magnified by our AUTO-V clipping. In fact, the dot-product similarity in \eqref{eq:clipping max} closely resembles the convergence of DP optimization for \Cref{thm:upper bounding grad norm without r} in \eqref{eq:lipschitz expand}. 
\begin{figure}[!htb]
\vspace{-0.1cm}
\centering
\includegraphics[width=0.33\linewidth]{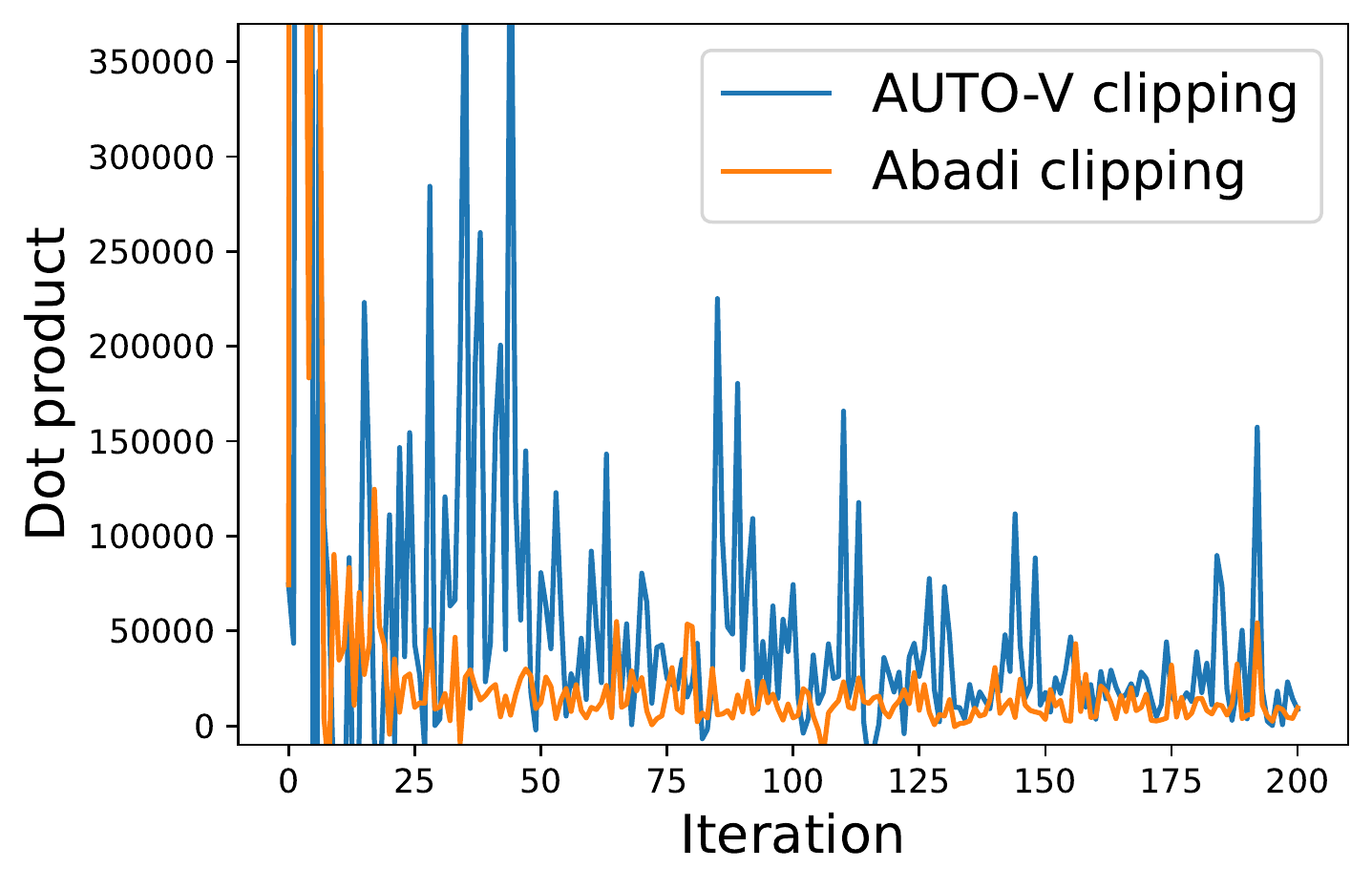}
\includegraphics[width=0.33\linewidth]{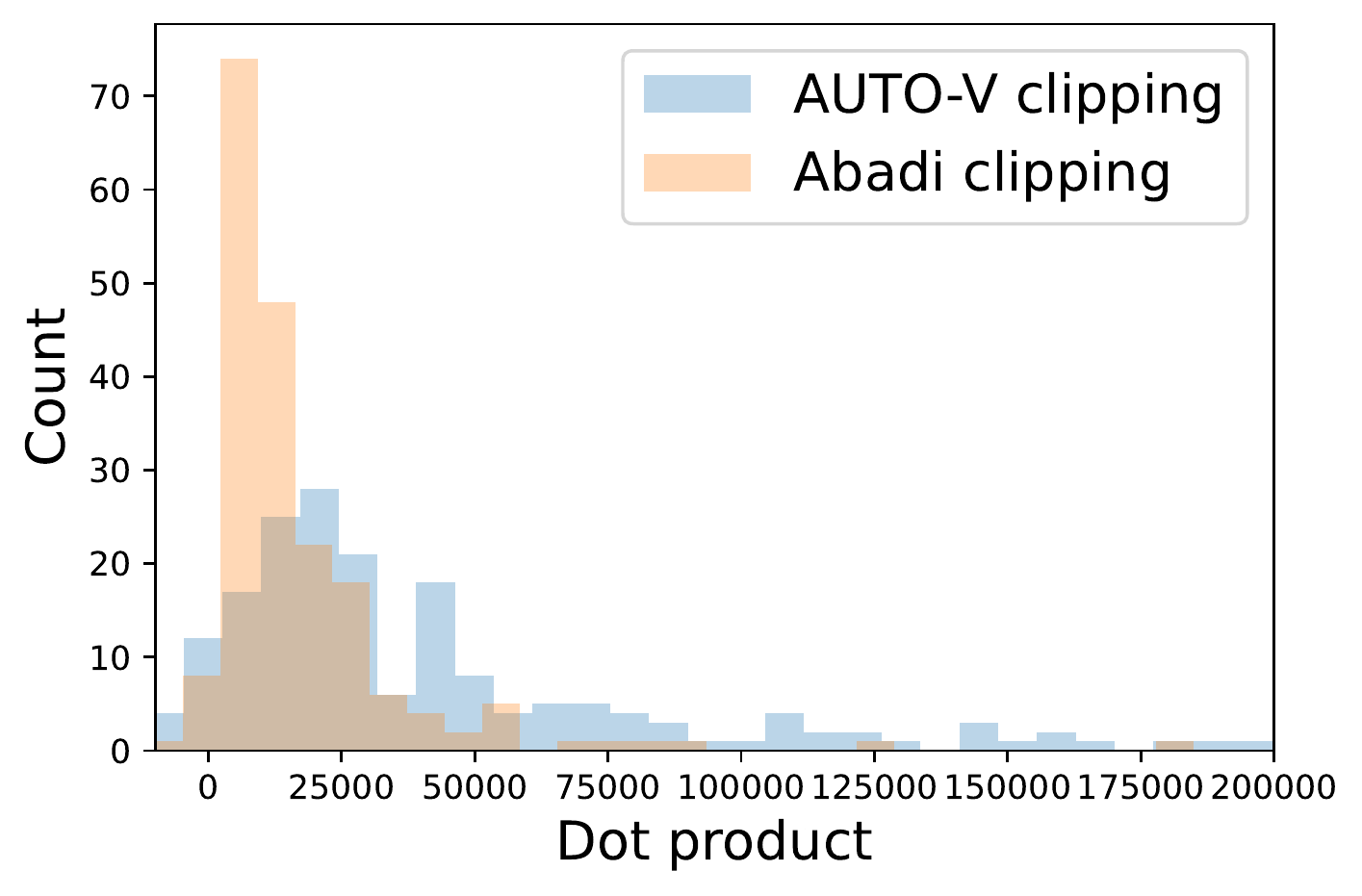}
\vspace{-0.2cm}
    \caption{RoBERTa-base with DP-Adam ($\epsilon=3$) on SST2 dataset, as in \Cref{sec:NLP classification}.}
    \label{fig:dot product}
\end{figure}

\newpage
\subsection{Stability constant breaks scale-invariance and remains stationary}
\label{subsec:lazy region and stability}

\begin{wrapfigure}{r}{0.4\textwidth}
    \vspace{-0.55cm}
    \centering
\includegraphics[width=0.93\linewidth]{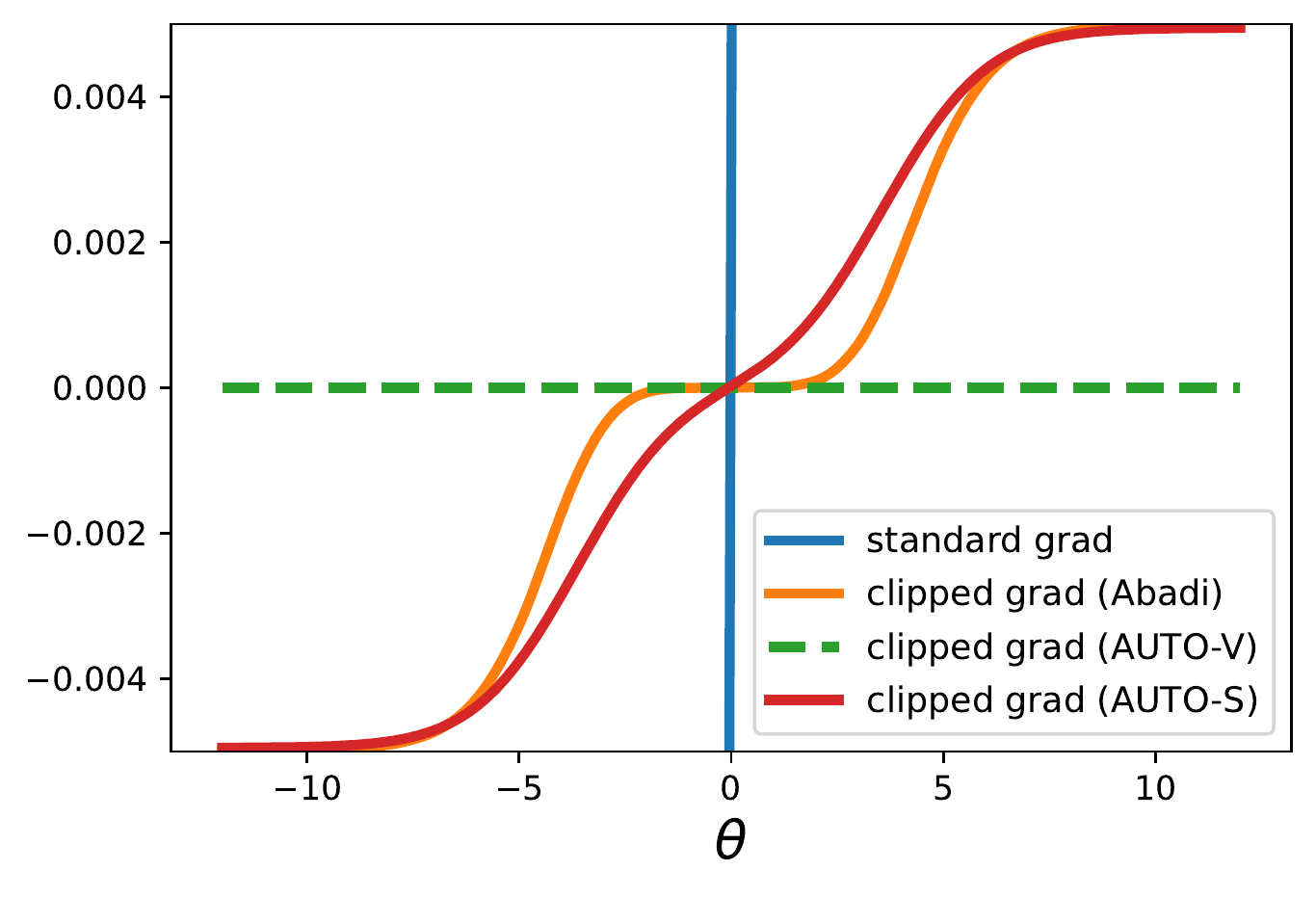}
    \vspace{-0.3cm}
    \caption{Gradient (scalar) at each $\theta$.}
    \label{fig:lazy zone}
    \vspace{-0.4cm}
\end{wrapfigure}

One potential drawback of AUTO-V clipping is that all gradients lose their magnitudes information completely, since $\|\g_i\cdot \texttt{Clip}_\text{AUTO-V}(\g_i;R)\|=R, \forall i$. This scale-invariance in AUTO-V and partially in Abadi's clipping (when $\|\g_i\|>R$) leads to the "lazy region" issue: the parameters will not be updated by DP-GD even if the true gradients are non-zero. In \Cref{fig:lazy zone}, we illustrate such issue in a logistic regression\footnote{The settings are in \Cref{app:lazy region examples}, where the lazy region issues also emerge in the mean estimation problem. We note that the lazy region is also discussed in \cite[Example 2]{chen2020understanding}.} for AUTO-V and Abadi's clipping, when the trainable parameter $\theta\in [-2,2]$, as the gradients from two classes cancel each other.

To preserve the magnitude information and thus escape the lazy region, we propose the AUTO-S clipping, with a positive stability constant $\gamma$:
\begin{align}
\texttt{Clip}_\text{AUTO-S}(\g_i;R):=R/(||\g_i||+\gamma)
\label{eq:Auto-S}
\end{align}

\vspace{-0.2cm}
We visualize in \Cref{fig:norm after clipping} that AUTO-S allows larger per-sample gradients to have larger magnitudes after the clipping, while still allowing smaller gradients to vanish after ``clipping’’. That is, as $\g_i\to 0$, the existence of $\gamma$ allows the clipped gradient $C_i\g_i\to\g_i/\gamma$ rather than having a magnitude $R$ as in AUTO-V. We elaborate this point in \Cref{sec:privacy}.
 This is critical in our convergence analysis and allows DP-SGD$_\text{AUTO-S}$ (but not DP-SGD$_\text{AUTO-V}$) to converge to zero gradient norms in \Cref{sec:convergence}.


\section{Automatic DP Training}
One may wonder why our clipping \eqref{eq:Auto-V}\eqref{eq:Auto-S} is automatic at all, if the hyperparameter $R$ is still present and there is an additional parameter $\gamma$ to choose. It turns out that any constant choice of $R>0$  is equivalent to choosing $R=1$, and common deep learning optimizers are insensitive to the choice of $\gamma$ (e.g. for any $\gamma>0$, we show that the gradient norm converges to zero at the same asymptotic rate in \Cref{thm:upper bounding grad norm without r}; see also the ablation study in \Cref{fig:easy to choose gamma}). Consequently, we set $\gamma=0.01$ as the default. Specifically, let us redefine the $R$-independent clipping function:
\vspace{-0.15cm}
\begin{align}
\texttt{Clip}_\text{AUTO-S}(\g_i):=1/(||\g_i||+\gamma).
\label{eq:Auto without R}
\end{align}
With this clipping, we can design automatic DP optimizers similar to \eqref{eq:DP optimizers}:
\begin{align}
\text{Automatic DP Optimizer} (\{\g_i\}_{i=1}^B )&=\text{Optimizer}\underbrace{\big(\sum_{i} \frac{\g_{t,i}}{||\g_{t,i}||+\gamma}+\sigma\cdot\mathcal{N}(0,\I)\big)}_{\text{denoted as }\hat\g_t}
\label{eq:trully auto DP optim}
\end{align}

\vspace{-0.2cm}
Clearly, the new private gradient $\hat \g_t$ from our automatic clipping is $R$-independent, in contrast to the one used in \eqref{eq:DP optimizers}. A concrete example (in the case of $\gamma = 0$) that is comparable to \eqref{eq:R-dep AUTO-V dpsgd} will be
\vspace{-0.15cm}
\begin{align}
R\text{-independent } \text{DP-SGD}_\text{AUTO-V}:\quad\w_{t+1}=\w_t-\eta\Big(\sum_{i} \frac{\partial l_i}{\partial \w_t}/\Big\|\frac{\partial l_i}{\partial \w_t}\Big\|+\sigma\cdot\mathcal{N}(0,\I)\Big)
\label{eq:R-indep AUTO-V dpsgd}
\end{align}

\vspace{-0.3cm}
Leveraging the private gradient $\hat \g_t$ in \eqref{eq:trully auto DP optim}, we can train DP neural networks without tuning DP-specific hyperparamters $R$ and $\sigma$, as demonstrated in \Cref{alg:dp auto procedure}.

\vspace{-0.4cm}
\begin{minipage}{\textwidth}
\begin{algorithm}[H]
	\caption{Automatic Deep Learning with DP}\label{alg:dp auto procedure}
    \textbf{Parameters:} initial weights $\w_0$, learning rate $\eta_t$, sampling probability $p$, number of iterations $T$.
\begin{algorithmic}[1]
    \State Compute $\sigma$ such that  $\epsilon_{\text{Accountant}}(\delta,\sigma,p,T) \leq \epsilon$ from any privacy accountant.
	\For{iteration $t = 1, \cdots, T$}
	\State{Sample a batch $B_t$ by including each data point i.i.d. with probability $p$ 
	}
	\State Apply automatic clipping to per-sample gradients $\{\g_{i}\}_{i\in B_t}$:  $\hat \g_i=\g_i/(\|\g_i\|_2+0.01)$.
	\State Add Gaussian noise to the sum of clipped gradients: $\hat \g=\sum_i \hat\g_i+\sigma\cdot\mathcal{N}(0,\I)$.
	\State Update $\w_t$ by any optimizer on the private gradient $\hat\g$ with learning rate $\eta_t$.
	\EndFor
	\end{algorithmic}
\end{algorithm}
\end{minipage}

We will elaborate two distinct reasons in the next sub-sections for the following statement:
$$\boxed{\text{DP Optimizer}_\text{Abadi}\approx R\text{-dependent DP Optimizer}_\text{AUTO}\equiv R\text{-independent DP Optimizer}_\text{AUTO}}$$
which expunges the DP hyperparameters, only leaving us the regular hyperparameters such as learning rate, weight decay, etc. The significant save in the tuning effort is illustrated in \Cref{fig:AUTO only 1D grid search}.
\begin{figure}[!htb]
\vspace{-0.3cm}
\centering
\includegraphics[width=0.31\linewidth]{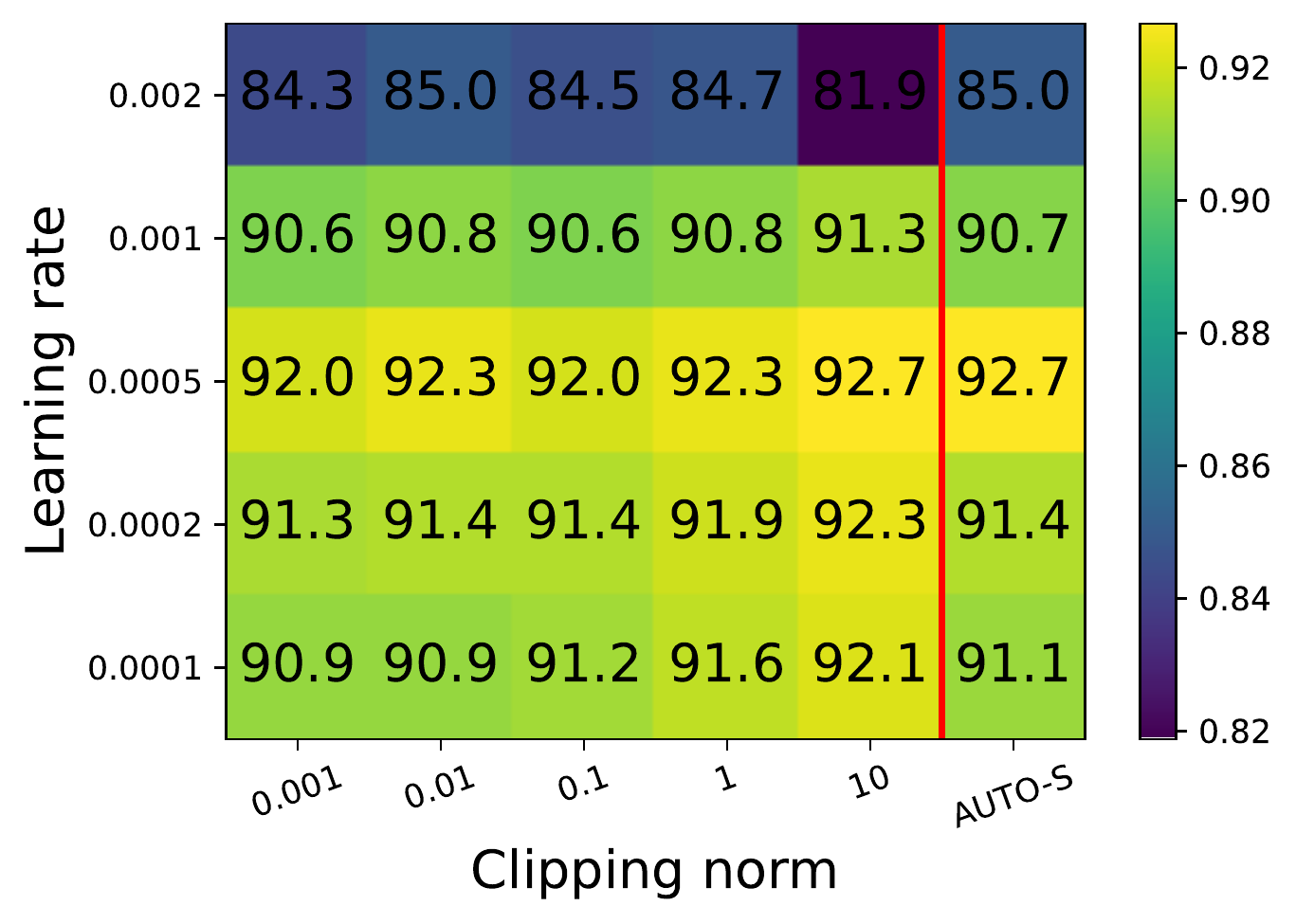}
    \includegraphics[width=0.31\linewidth]{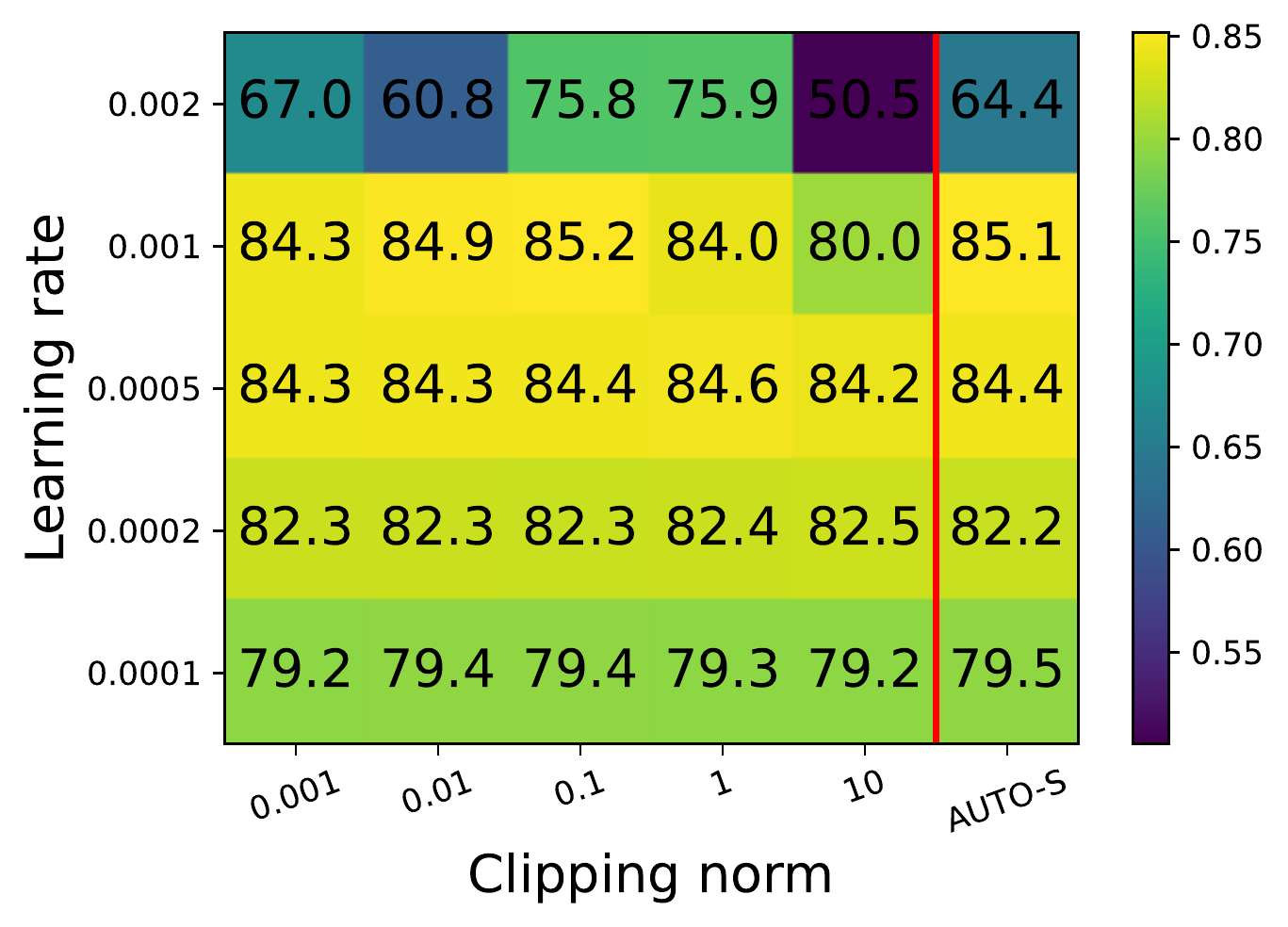}
    \includegraphics[width=0.31\linewidth]{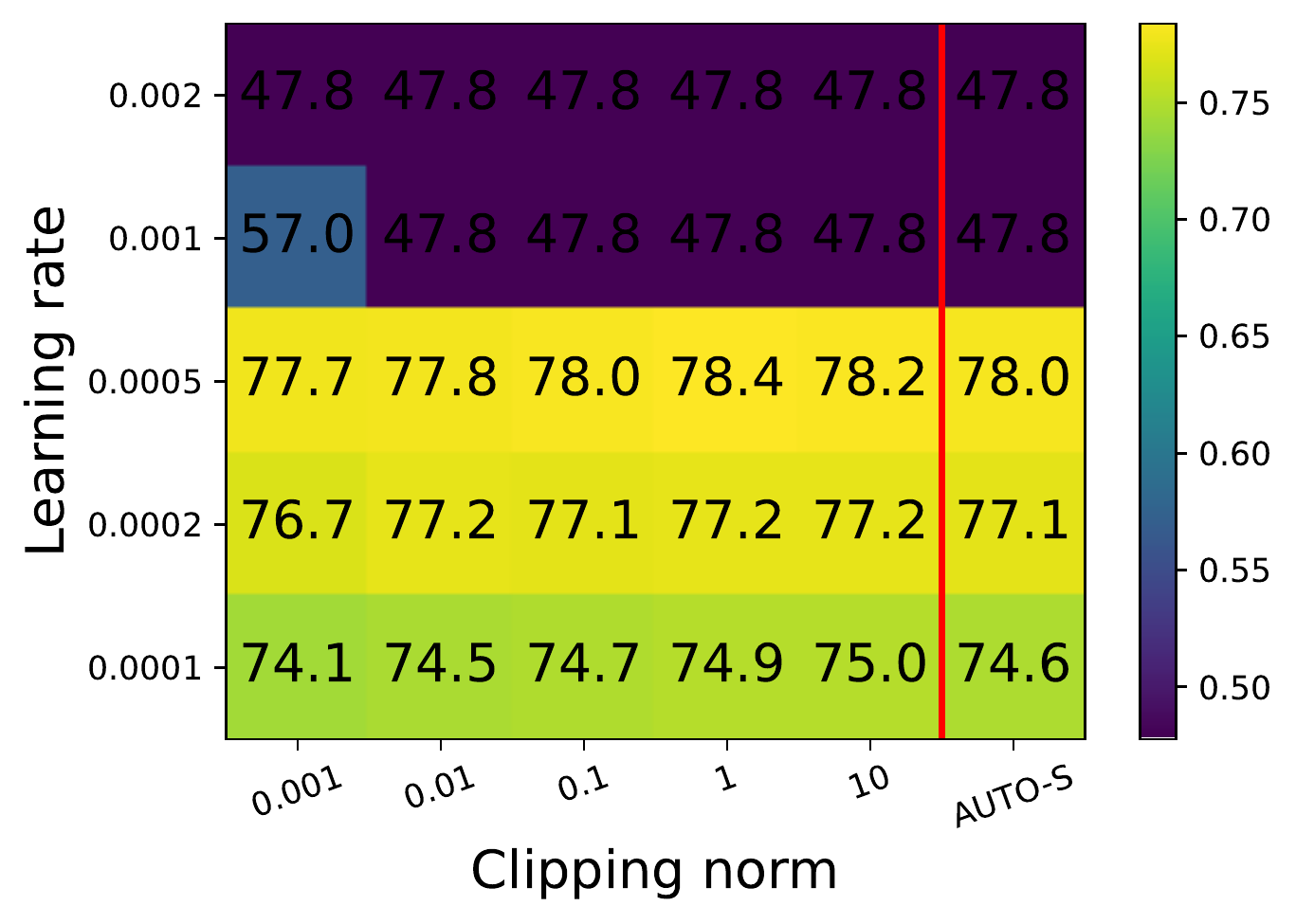}
\vspace{-0.2cm}
    \caption{Test accuracy of RoBERTa-base by different clipping thresholds $R$ and learning rates $\eta$. This is trained with DP-Adam (Abadi and AUTO-S) on SST2 (left, 3 epochs), QNLI (middle, 1 epoch), and MNLI (right, 1 epoch), under $\epsilon=3$. Notice by only searching along $\eta$, instead of over $(R,\eta)$, we can save the cost of hyperparameter tuning by $5\times$.}
    \label{fig:AUTO only 1D grid search}
\end{figure}

\subsection{Non-adaptive optimizer couples clipping threshold with learning rate}
With $R$-dependent automatic clipping, DP-SGD becomes
\vspace{-0.1cm}
$$\w_{t+1}=\w_t-\eta\Big(\sum_{i}\g_{t,i}\cdot\frac{R}{||\g_{t,i}||+\gamma}+\sigma R\cdot\mathcal{N}(0,\I)\Big)=\w_t-\eta R\hat \g_t.$$

\vspace{-0.35cm}
We can view $\eta_\text{effective}\equiv\eta R$ as a whole: increasing $R$ has the same effect as increasing $\eta$, which explains the diagonal pattern in \Cref{fig:my motivation}(lower plot) where $\text{DP-SGD}_\text{Abadi}$ is applied with small clipping threshold. We extend to general non-adaptive optimizers in \Cref{thm: non-adaptive automatic}\footnote{This coupling of $\eta$ and $R$ is also partially observed in \cite[Appendix B.1]{de2022unlocking} through a re-parameterization trick of Abadi's clipping. Unlike AUTO-S/V, the coupling is not strict (e.g. doubling $R$ is not equivalent to doubling $\eta$, thus still necessitating tuning both $(\eta,R)$), and the relationship to weight decay was not discussed.}.

\begin{theorem}
\label{thm: non-adaptive automatic}
Non-adaptive $R$-dependent automatic DP optimizers (including SGD, Heavyball\cite{polyak1964some} and NAG\cite{nesterov1983method}), with learning rate $\eta$ and weight decay $\lambda$, is equivalent to $R$-independent automatic DP optimizers, with learning rate $\eta R$ and weight decay $\lambda/R$.
\end{theorem}

\subsection{Adaptive optimizer can be insensitive to clipping threshold}

Adaptive automatic DP optimizers are different than the non-adaptive ones, as the clipping threshold cancels out instead of being coupled with learning rate. To see this, we scrutinize $\text{DP-Adam}_\text{Abadi}$ (which is similar to $\text{DP-Adam}_\text{AUTO-V}$) in \Cref{fig:my motivation}(upper plot), where columns to the left are almost identical. Further evidence is observed in \cite[Table 5]{mehta2022large} that shrinking $R$ has zero effect on LAMB. We now give a simple explanation using AdaGrad \cite{duchi2011adaptive}:
$$\w_{t+1}=\w_{t}-\eta \frac{\g_t}{\sqrt{\sum_{\tau<t}\g_\tau^2}}$$

\vspace{-0.25cm}
where $\g_t=\sum_i \g_{t,i}$ is the gradient sum. In $R$-dependent DP-AdaGrad$_\text{AUTO-V}$, the private gradient is $R\hat \g_t$ in place of the standard gradient sum $\g_t$:

\vspace{-0.3cm}
$$\w_{t+1}=\w_{t}-\eta \frac{R\hat \g_t}{\sqrt{R^2\sum_{\tau\leq t}\hat \g_\tau^2}}=\w_{t}-\eta \frac{\hat\g_t}{\sqrt{\sum_{\tau< t}\left(\hat\g_\tau\right)^2}}.$$

\vspace{-0.15cm}
We generalize to other adaptive optimizers in \Cref{thm: adaptive automatic}
and to the per-layer clipping style in \Cref{app:automatic per-layer}.

\begin{theorem}
\label{thm: adaptive automatic}
Adaptive $R$-dependent automatic DP optimizers (e.g. AdaGrad\cite{duchi2011adaptive}, AdaDelta\cite{zeiler2012adadelta}, AdaMax/Adam\cite{kingma2014adam}, NAdam\cite{dozat2016incorporating}, RAdam\cite{liu2019variance}, LARS\cite{you2017scaling}, LAMB\cite{you2019large}), with learning rate $\eta$ and weight decay $\lambda$ is equivalent to $R$-independent automatic DP optimizers with learning rate $\eta$ and weight decay $\lambda/R$. With decoupled weight decay\cite{loshchilov2017decoupled}, $R$-dependent automatic DP-AdamW is equivalent to $R$-independent automatic DP-AdamW with the same $\eta$ and $\lambda$.
\end{theorem}


\subsection{Automatic clipping is equally private and maximizes utility}
\label{sec:privacy}

\begin{wrapfigure}{r}{0.48\textwidth}
\vspace{-0.5cm}
    \centering
    \includegraphics[width=0.8\linewidth]{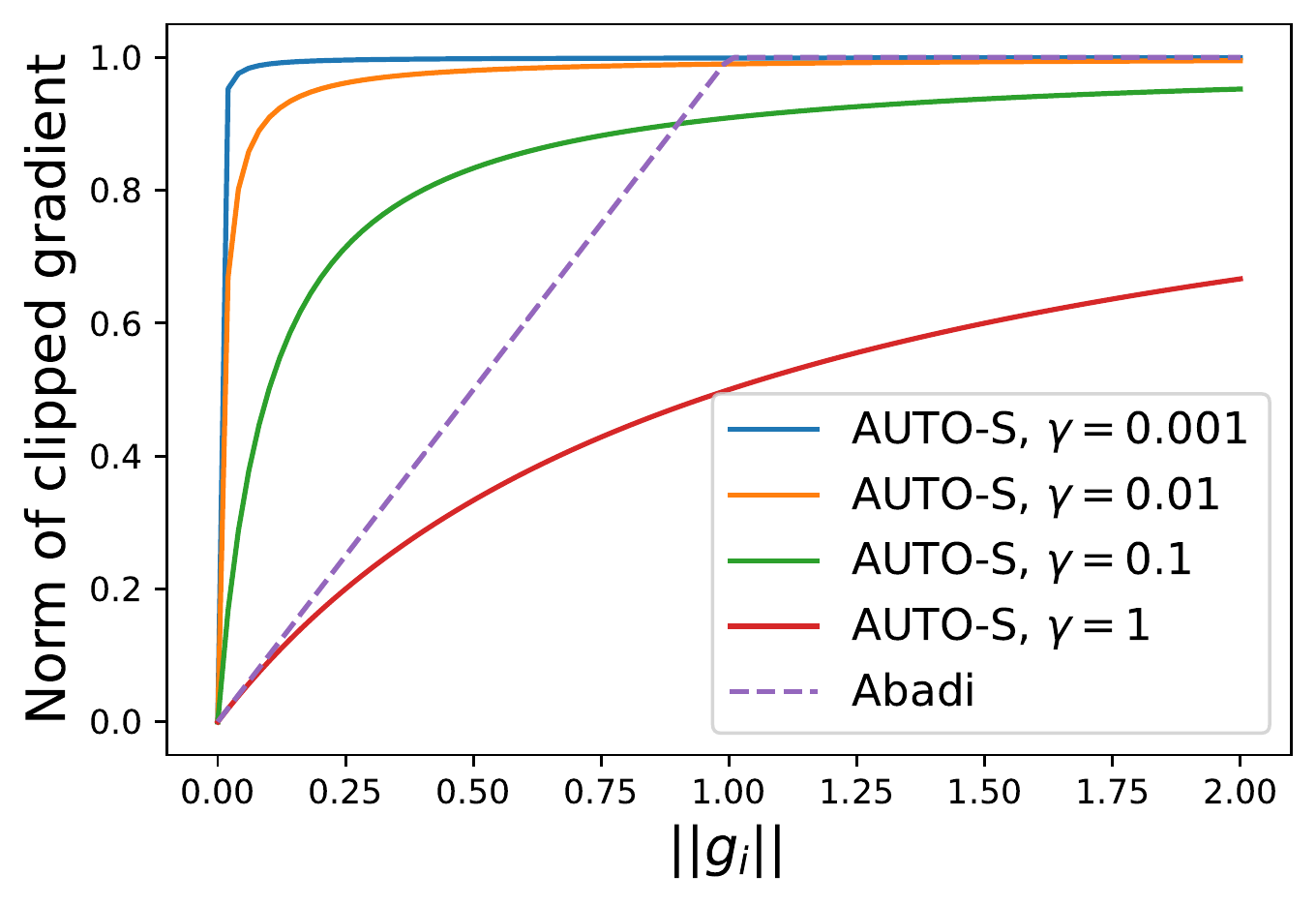}
    \vspace{-0.3cm}
    \caption{Per-sample gradient norms before and after different clippings at $R=1$.}
    \label{fig:norm after clipping}
\vspace{-0.5cm}
\end{wrapfigure}
In \Cref{thm:auto is private} (proved in \Cref{app:proof DP}), we show that the new private gradient $\hat \g_t$ in \eqref{eq:trully auto DP optim} has the same level of privacy guarantee as the existing one in \eqref{eq:DP optimizers}, since the global sensitivity remains the same (see \Cref{fig:norm after clipping}). We note that as long as $\gamma>0$, the magnitude information of per-sample gradients is preserved by AUTO-S, in the sense that $\|\g_i\|>\|\g_j\|\Longleftrightarrow \|C_i\g_i\|>\|C_j\g_j\|$, whereas this can be violated in both the AUTO-V and Abadi's clipping (as depicted by the flat curve in \Cref{fig:norm after clipping} when $\|\g_i\|>1$).
\\\\
Additionally, note that when $\gamma$ is small, almost all data points ``max out'' the signal relative to the amount of noise we add. 
To say it differently, for the same amount of noise, AUTO-S with small $\gamma$ allows more signal to be pushed through a differentially private channel. 
Towards the end of the training, i.e., 
at the limit when $\|\g_i\| \rightarrow 0$ for all $i$, then we have $\sum_i \frac{\g_{i}}{\|\g_i\| +\gamma} \rightarrow \frac{1}{\gamma}\sum_i \g_{i}$. In words, the clipped gradients become closer to the standard SGD, thus do not suffer from the instability of AUTO-V. 

\begin{theorem}
\label{thm:auto is private}
Under the noise multiplier $\sigma$, number of iterations $T$, subsampling probability $B/n$, DP optimizers using AUTO-V or AUTO-S clipping satisfy $(\epsilon_{\text{Accountant}}(\delta,\sigma,B/n,T),\delta)$-DP, where $\epsilon_{\text{Accountant}}$ is any valid privacy accountant for DP-SGD under Abadi's clipping.
\end{theorem}

\section{Convergence analysis of DP-SGD with automatic clipping}
\label{sec:convergence}

\subsection{Convergence theory of DP-SGD to stationary points}
We highlight that automatic clipping can be more amenable to analysis than Abadi's clipping in \cite{chen2020understanding}, since we no longer need to decide whether each per-sample gradient is clipped.

To analyze the convergence of automatic DP-SGD \eqref{eq:trully auto DP optim} in the non-convex setting, we follow the standard assumptions in the SGD literature \cite{ghadimi2013stochastic,allen2018natasha,bottou2018optimization}, including a symmetry assumption on the gradient noise, which is empirically verified in \cite[Figure 3]{chen2020understanding} and commonly used in the standard non-DP literature \cite{mandt2017stochastic,smith2018don,chaudhari2018stochastic,xie2020diffusion}. We refer the curious readers to \Cref{app:grad symm} for details.

\begin{assumption}[Lower bound of loss]\label{assumption: lower bounding loss}
For all $\w$ and some constant $\mathcal{L}_*$, we have $\mathcal{L}(\w)\geq\mathcal{L}_*$.
\end{assumption}

\begin{assumption}[Smoothness]
\label{assumption: Lipschitz}
Let $ \g(\w) $ denote the gradient of the objective $ \mathcal{L}(\w) $. Then $ \forall \w, \v $, there is an non-negative constant $L$ such that
\vspace{-0.2cm}
\begin{align}
 \mathcal{L}(\v)-\left[\mathcal{L}(\w)+\g(\w)^\top (\v-\w)\right]
 \leq \frac{L}{2} \|\w-\v\|^{2}.
\end{align}
\end{assumption}

\begin{assumption}[Gradient noise]
\label{assumption: tilde g}
The per-sample gradient noise $\tilde\g_{t,i}-\g_t$ is i.i.d. from some ditribution such that
$$\E(\tilde\g_{t,i}-\g_t)=0, \E\|\tilde\g_{t,i}-\g_t\|^2\leq \xi^2,$$

\vspace{-0.2cm}
and $\tilde \g_{t,i}$ is centrally symmetric about $\g_t$ in distribution:
$\tilde \g_{t,i}-\g_t\overset{\mathcal{D}}{=}\g_t-\tilde \g_{t,i}.$
\end{assumption}

We show in \Cref{thm:upper bounding grad norm without r} that DP-SGD with AUTO-S clipping allows the true gradient norm to converge to zero, though the clipped gradient may still be biased, but not so with AUTO-V clipping. 

\begin{theorem}\label{thm:upper bounding grad norm without r}
Under \Cref{assumption: lower bounding loss}, \ref{assumption: Lipschitz}, \ref{assumption: tilde g}, running DP-SGD with automatic clipping for $T$ iterations and setting the learning rate $\eta\propto 1/\sqrt{T}$ give\footnote{The upper bound takes an implicit form of $\mathcal{G}(\cdot;\xi,\gamma)$ because it is a lower envelope of functions $\frac{\xi}{r}+\mathcal{F}(\cdot;r,\xi,\gamma)$ over all possible $r>0$, whose forms are detailed in \Cref{thm: convergence DPSGD AUTO}. Notice that $\mathcal{G}$ results only from the clipping operation, not from the noise addition.}
\begin{align}
\begin{split}
\min_{0\leq t\leq T}\E(\|\g_t\|)
&\leq \mathcal{G}\left(\frac{4}{\sqrt{T}}\sqrt{(\mathcal{L}_0-\mathcal{L}_*)L\left(1+\frac{\sigma^2 d}{B^2} \right)};\xi,\gamma\right)
:=\min_{r>0} \frac{\xi}{r}+\mathcal{F}\left(\cdots;r,\xi,\gamma\right).
\end{split}
\label{eq:min grad norm without r}
\end{align}
Here $\cdots$ represents the first argument of $\mathcal{G}$, and $\mathcal{G}$ is increasing and positive. As $T\to\infty$, we have $\min_t\E(\|\g_t\|)=O(T^{-1/4})$ for AUTO-S, the same rate as the standard SGD given in \Cref{thm: convergence nonDP SGD}.
\end{theorem}
\begin{figure}[!htb]
\vspace{-0.4cm}
    \centering
    \hspace{-0.5cm}
    \includegraphics[width=0.3\linewidth]{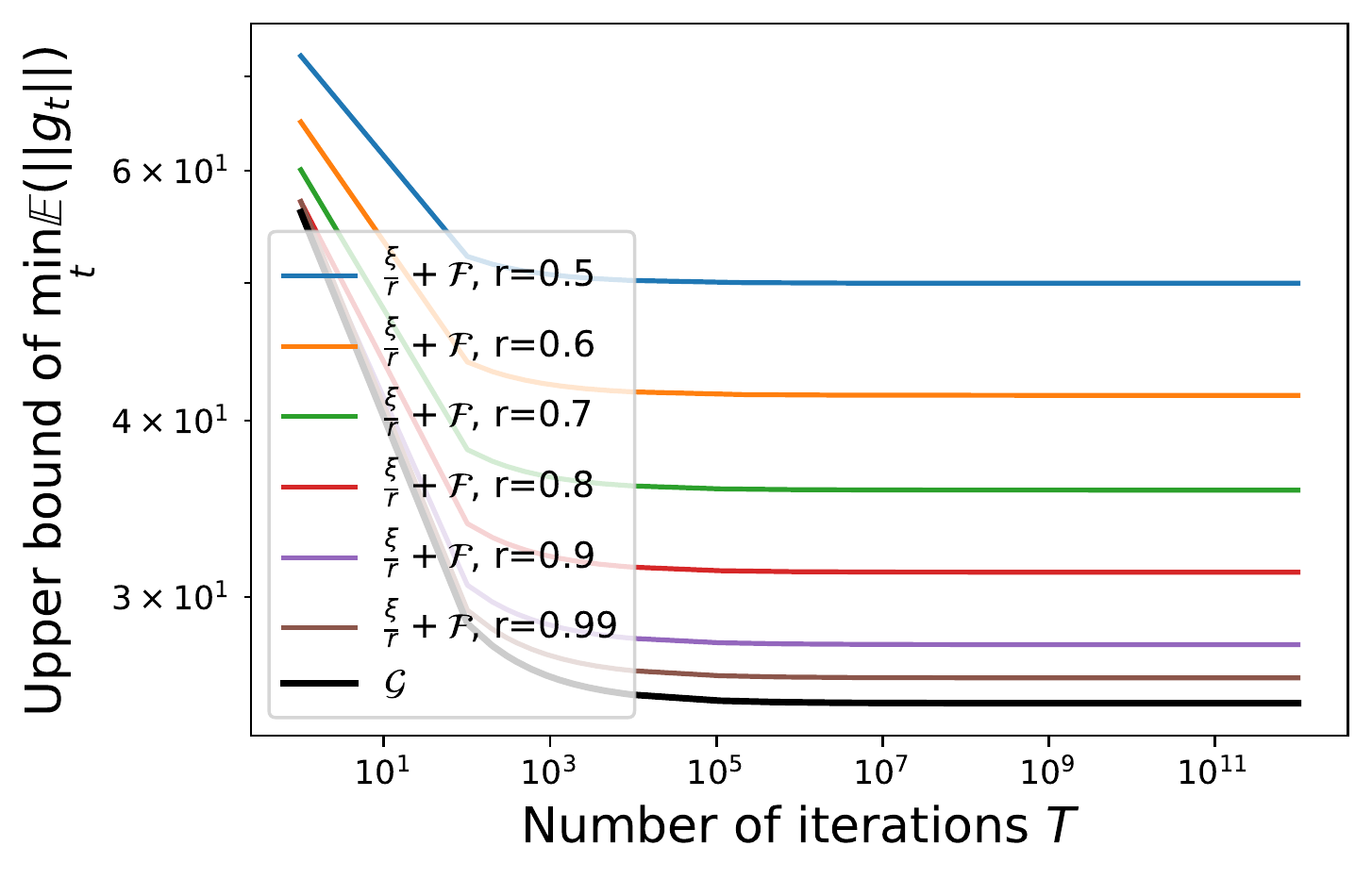}
    \hspace{-0.2cm}
    \includegraphics[width=0.29\linewidth]{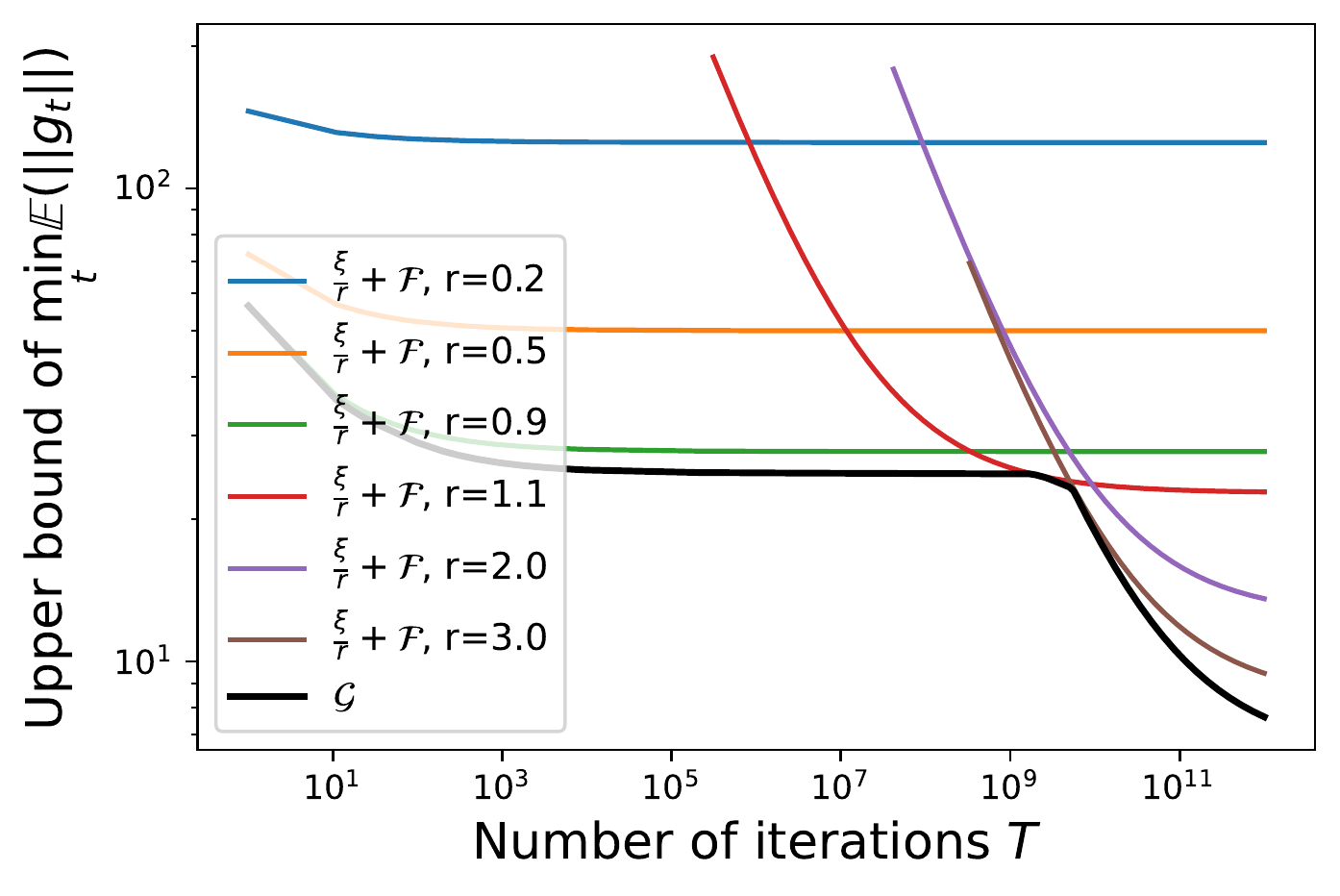}
    \hspace{-0.2cm}
    \includegraphics[width=0.295\linewidth]{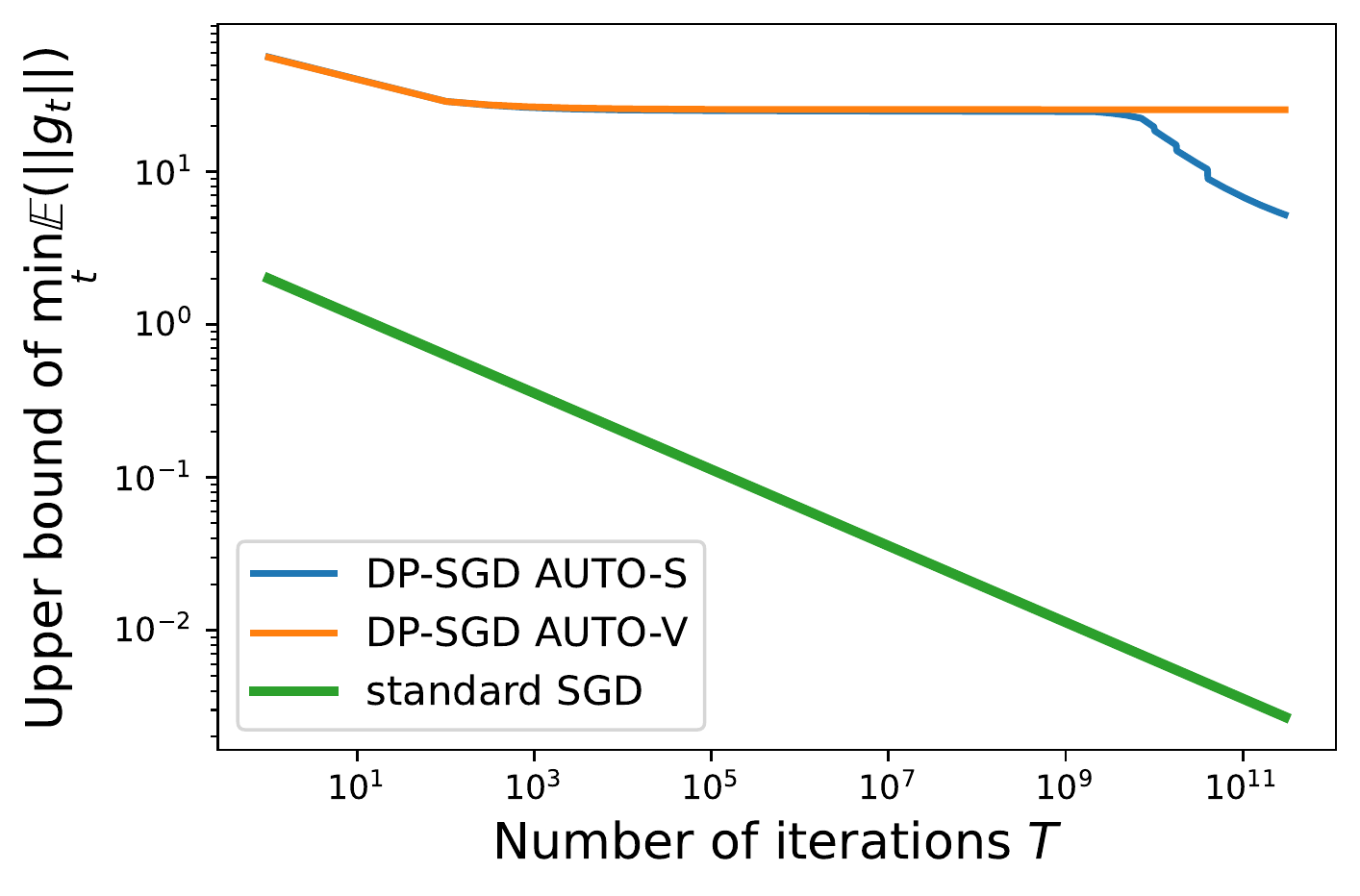}
    \hspace{-0.5cm}
\vspace{-0.2cm}
    \caption{Left: DP-SGD with AUTO-V clipping. Middle: DP-SGD with AUTO-S clipping. Right: Log-log plot of convergence rate in comparison to standard SGD. Here $\xi=25,\gamma=0.01$, and the $O(1/\sqrt{T})$ term is set to 10 for DP-SGD and to 2 for standard SGD.}
    \label{fig: G plots}
\end{figure}


\begin{remark}
\label{rem:S better than V}
We show in \Cref{thm: convergence DPSGD AUTO} and in \Cref{fig: G plots} that the upper bound \eqref{eq:min grad norm without r} has $\mathcal{G}\geq\xi$ for AUTO-V ($\gamma=0$), and $\mathcal{G}$ only reduces to zero for AUTO-S ($\gamma>0$). 
We provide real data evidence in \Cref{fig:grad norm reduced by stability} that strictly positive $\gamma$ reduces the gradient norm significantly.
\end{remark}

\subsection{Analysis of factors affecting the convergence}
We now analyze the many factors that affect the convergence in \Cref{thm:upper bounding grad norm without r}, from a unified viewpoint of both the convergence and the privacy.

We start with the stability constant $\gamma$ and the learning rate $\eta_t$, both only affect the convergence not the privacy. We empirically observe in \Cref{fig:G dependence on xi and gamma} that small $\gamma$ benefits the convergence at initial iterations (when the privacy guarantee is strong) but larger $\gamma$ converges faster asymptotically. For $\eta_t$, the optimal is in fact the miminizer of the hyperbola in \eqref{eq:DP-SGD hyperbola lr}, that is unique and tunable.

Next, we focus on the hyperparameters that affect both convergence and privacy: the batch size $B$, the noise multiplier $\sigma$, and the number of iterations $T$. These hyperparameters have to be considered along the privacy-accuracy tradeoff, not just from a convergence perspective.

{Recall that given a fixed privacy budget $(\epsilon,\delta)$, we rely on modern privacy accountant for computing the appropriate combinations of parameter $\sigma,T,B$. The exact expression of the bound as a function of $(\epsilon,\delta)$ is somewhat messy. For this reason, we illustrate our analysis in terms of the surrogate parameter $\mu$ for $\mu$-GDP} \cite{dong2019gaussian}, which implies $(\epsilon,\delta)$-DP with $\epsilon = \mu^2 + \mu\sqrt{2\log(1/\delta)})$.  \cite{bu2020deep} showed that DP-SGD's privacy guarantee asymptotically converges to $\mu$-GDP (as $T\rightarrow\infty$) with $\mu=\frac{B}{n}\sqrt{T(e^{1/\sigma^2}-1)}$.
{We can alternatively leverage $\rho$-tCDP} \cite{bun2018composable} {for similar conclusions, using $\rho$ in place of $\mu^2$ in} \eqref{eq:min grad norm without r with privacy}.

\begin{theorem}
Under \Cref{assumption: lower bounding loss}, \ref{assumption: Lipschitz}, \ref{assumption: tilde g}, fixing the asymptotic $\mu(\epsilon,\delta)$-GDP parameter, running DP-SGD with automatic clipping for $T$ iterations and setting the learning rate $\eta\propto 1/\sqrt{T}$ give
\begin{align}
\min_{0\leq t\leq T}\E(\|\g_t\|)
&\leq \mathcal{G}\left(4\sqrt{(\mathcal{L}_0-\mathcal{L}_*)L\left(\frac{1}{T}+\frac{d}{\mu^2 n^2}+O\big(\frac{1}{B^2 T}\big) \right)};\xi,\gamma\right)
\label{eq:min grad norm without r with privacy}
\end{align}
\end{theorem}

\vspace{-0.35cm}
To show that our analysis matches the training behaviors observed in SOTA empirical work \cite{li2021large,kurakin2022toward,de2022unlocking,tramer2020differentially,mehta2022large,yu2021large}, we minimize the first argument of $\mathcal{G}$ in \eqref{eq:min grad norm without r with privacy}, denoted as $X(B,T,\mu,d,L,\mathcal{L}_0)$. 
\vspace{-0.2cm}
\begin{enumerate}
    \item \textbf{[Train longer with larger noise]} Fixing the expected batch size $B$, we see that $X$ is decreasing in $T$. Hence larger $T$ and consequently larger $\sigma$ are preferred.
    \item \textbf{[Larger batch size helps]} Fixing number of iterations $T$ or epochs $E=BT/n$, we see that $X$ is decreasing in $B$. Hence larger $B$ and consequently larger $\sigma$ are preferred.
    
    \item \textbf{[Pretraining is critical]} Pretraining can boost the DP accuracy through a much smaller initial loss $\mathcal{L}_0$ and from a smooth (small $L$) and flat (small $\xi$, c.f. \Cref{fig:G dependence on xi and gamma}(left)) initialization.
    
    \item \textbf{[Learning rate needs tuning]} The optimal learning rate by minimizing \eqref{eq:DP-SGD hyperbola lr} is  
    $\sqrt{\frac{(\mathcal{L}_0-\mathcal{L}_*)\mu^2 n^2}{L(\mu^2n^2 + dT)}}$. 
    This indicates that one should use larger learning rate for smaller model $d$, weaker privacy (larger $\mu$ or small $\epsilon$), or smaller iteration budget $T$. 
\end{enumerate}


\section{Experiments}
\label{sec:experiments}
We evaluate our automatic DP training on image classification, sentence classification, and table-to-text generation tasks. Detailed settings including hyperparameters can be found in \Cref{app:experiemnt settings}.

\subsection{Image classification}
\label{sec:CV classification}
For MNIST/FashionMNIST, we use the same setup as in \cite{papernot2020tempered,tramer2020differentially,shamsabadi2021losing} with a simple CNN. For CIFAR10, we use the same setup as in \cite{tramer2020differentially} with pretrained SimCLRv2 \cite{chen2020simple}. For \href{https://github.com/fastai/imagenette/}{ImageNette}, a 10-class sub-task of ImageNet \cite{deng2009imagenet}, we use the same setup as in \cite{klause2022differentially} without the learning rate decay. For CelebA \cite{liu2015faceattributes}, the real human face dataset, we train ResNet9 \cite{he2016deep} with group normalization to replace the batch normalization. Notice that CelebA contains high-resolution (178x218) images, each with 40 labels. We consider CelebA for either multi-class classification on one label, e.g. `Smiling' and `Male', or for multi-label/multi-task problem to learn all labels simultaneously. 

\vspace{-0.6cm}
\begin{table}[H]
\centering
\setlength\tabcolsep{3pt}
\caption{Average test accuracy and 95\% confidence interval on image tasks over 5 runs.}
\resizebox{0.9\linewidth}{!}{
    \begin{tabular}{|c|c|c|c|c|c|}
          \hline
\multirow{3}{*}{Task}& \multirow{3}{*}{Model}&\multirow{3}{*}{$(\epsilon,\delta)$} &\multicolumn{3}{c|}{Accuracy \%}  \\
         \cline{4-6}
         &&&\multirow{2}{*}{Abadi's clipping}&\multirow{2}{*}{AUTO-S clipping}&non-DP
         \\
         &&&&&($\epsilon=\infty$)\\
        \hline
         MNIST& 4-layer CNN&$(3,1e$-$5)$&$98.04\pm0.09$&$98.15\pm 0.07$&$99.11\pm 0.07$
         \\
         \hline
         FashionMNIST& 4-layer CNN&$(3,1e$-$5)$&$86.04\pm0.26$&$86.36\pm 0.18$&$89.57\pm 0.13$
         \\
         \hline
         CIFAR10 pretrained& SimCLRv2&$(2,1e$-$5)$&$92.44\pm 0.13$&$92.70\pm 0.02$&$94.42\pm 0.01$
         \\
          \hline
         ImageNette& ResNet9&$(8,1e$-$4)$&$60.29\pm 0.53$&$60.71\pm 0.48$&$71.11\pm 0.37$
         \\
         \hline
         CelebA [Smiling]& ResNet9&$(8,5e$-$6)$&$90.75\pm 0.11$&$91.08\pm 0.08$&$92.61\pm0.20$
         \\
         \hline
        CelebA [Male]& ResNet9&$(8,5e$-$6)$&$95.54\pm 0.14$&$95.70\pm 0.07$&$97.90\pm 0.04$   
         \\
         \hline
         CelebA Multi-label& ResNet9&$(3,5e$-$6)$&$86.81\pm 0.03$&$87.05\pm 0.01$&$90.30\pm 0.02$
         \\
         \hline
         CelebA Multi-label& ResNet9&$(8,5e$-$6)$&$87.52\pm 0.15$&$87.58\pm 0.04$&$90.30\pm 0.02$
         \\\hline
    \end{tabular}
}
    \label{tab:CV results}
\end{table}

\vspace{-0.5cm}
In \Cref{tab:CV results}, we observe that AUTO-S clipping outperforms existing clipping in all datasets with statistical significance. Interestingly, the standard deviation from different runs is smaller for automatic DP optimizers, indicating better reproducibility and stability. We additionally experiment 40 binary classification problems on CelebA with respect to each label, and observe that the mean accuracy further improves to 91.63\% at $\epsilon=8$ for AUTO-S (see \Cref{app:celebA more experiments}).

\subsection{Sentence classification}
\label{sec:NLP classification}
On five benchmark language datasets (MNLI(m/mm)\cite{N18-1101}, QQP\cite{WinNT}, QNLI\cite{rajpurkar2016squad}, SST2\cite{socher2013recursive}), we compare our automatic DP training with re-parameterized gradient perturbation (RGP, \cite{yu2021large}) and full-parameter finetuning (full, \cite{li2021large}) using RoBERTa models \cite{liu2019roberta}. These methods use the same experimental setup. For language models, our automatic training is based on the codebase of \cite{li2021large}.

\begin{table}[H]
\hspace{-0.2cm}
\setlength\tabcolsep{2.2pt}
\caption{Test accuracy on language tasks with RoBERTa-base (12 blocks, 125 million parameters).}
\resizebox{\linewidth}{!}{
\begin{tabular}{|c|cccc|cccc|cccc|}
\hline
\multirow {2}{*}{Method}&\multicolumn {4}{c}{$\epsilon=3$}& \multicolumn{4}{c}{$\epsilon=8$}&\multicolumn{4}{c|}{$\epsilon=\infty$ (non-DP)}
\\\cline{2-13}
&MNLI&QQP&QNLI&SST2&MNLI&QQP&QNLI&SST2&MNLI&QQP&QNLI&SST2
\\\hline
RGP \cite{yu2021large} &-&-&-&-&80.5/79.6&85.5&87.2&91.6&83.6/83.2&89.3&91.3&92.9
\\\cline{1-13}
full \cite{li2021large}&82.45/82.99&85.56&\textbf{87.42}&91.86&83.20/83.46&86.08&\textbf{87.94}&92.09&\multirow{3}{*}{85.91/86.14}	&\multirow{3}{*}{87.34}&\multirow{3}{*}{91.40}&\multirow{3}{*}{94.49}
\\
full AUTO-V&81.21/82.03&84.72&86.56&91.86&82.18/82.64&\textbf{86.23}&87.24&92.09&&&&
\\
full AUTO-S&\textbf{83.22/83.21}&\textbf{85.76}&86.91&\textbf{92.32}&\textbf{83.82/83.55}&\textbf{86.58}&87.85&\textbf{92.43}&&&&
\\\hline
\end{tabular}
}
\label{tab:sentence roberta base}

\hspace{-0.2cm}
\setlength\tabcolsep{2.2pt}
\vspace{-0.2cm}
\caption{Test accuracy on language tasks with RoBERTa-large (24 blocks, 355 million parameters).}
\resizebox{\linewidth}{!}{
\begin{tabular}{|c|cccc|cccc|cccc|}
\hline
\multirow {2}{*}{Method}&\multicolumn {4}{c}{$\epsilon=3$}& \multicolumn{4}{c}{$\epsilon=8$}&\multicolumn{4}{c|}{$\epsilon=\infty$ (non-DP)}
\\\cline{2-13}
&MNLI&QQP&QNLI&SST2&MNLI&QQP&QNLI&SST2&MNLI&QQP&QNLI&SST2
\\\hline
RGP \cite{yu2021large}&-&-&-&-&86.1/86.0&86.7&90.0&93.0&-&-&-&-
\\\cline{1-13}
full \cite{li2021large}&\textbf{86.43}/86.46&86.43&90.76&93.04&87.02/\textbf{87.26}&87.47&91.10&93.81&\multirow{3}{*}{90.33/90.03}&\multirow{3}{*}{87.90}&\multirow{3}{*}{93.61}&\multirow{3}{*}{96.21}
\\
full AUTO-V&85.33/85.61&\textbf{86.61}&89.99&\textbf{93.12}&85.91/86.10&86.86&90.55&93.35&&&&
\\
full AUTO-S&86.27/\textbf{86.67}&\textbf{86.76}&\textbf{91.01}&\textbf{93.92}&\textbf{87.07}/87.16&\textbf{87.47}&\textbf{91.45}&\textbf{94.61}&&&&
\\\hline
\end{tabular}
}
    \label{tab:sentence roberta large}
\end{table}

\vspace{-0.7cm}
In \Cref{tab:sentence roberta base} and \Cref{tab:sentence roberta large}, we note that full parameter finetuning with AUTO-S outperforms or at least matches SOTA on all tasks. We use \textit{exactly the same} hyperparameters as in \cite{li2021large}.


\subsection{Table-to-text generation}
\label{sec:NLP generation}
We compare our automatic DP training with a variety of fine-tuning methods, for table-to-text generation task on E2E dataset \cite{dusek.etal2020:csl}, where the goal is to generate texts about different aspects of a restaurant's data. We measure the success on this task by BLEU, ROUGE-L (in \Cref{tab:E2E GPT selected}), METEOR, NIST, CIDEr (extended in \Cref{tab:E2E GPT extended}), with higher value meaning better model quality.

\vspace{-0.3cm}
\begin{table}[!htb]
    \centering
\setlength\tabcolsep{2pt}
\caption{Test performance on E2E dataset with GPT2. Additional performance measures are included in \Cref{tab:E2E GPT extended}. The best two GPT2 models for each row are marked in bold.}
\resizebox{\linewidth}{!}{
\begin{tabular}{|l|c|c|c|cccccccc|}
\hline
&DP&GPT2&GPT2&\multicolumn{8}{c|}{GPT2}
\\
Metric&guarantee&large&medium&&&&&&&&
\\\cline{3-12}
&&full&full&full&full&full&LoRA&RGP&prefix&top2&retrain
\\
&&AUTO-S&AUTO-S&AUTO-S&AUTO-V&\cite{li2021large}&\cite{hu2021lora}&\cite{yu2021large}&\cite{li2021prefix}&&
\\\hline
\multirow{3}{*}{BLEU}&$\epsilon=3$&\textbf{64.180}&\textbf{63.850}&\textbf{61.340}&\textbf{61.519}&\textbf{61.519}&58.153&58.482&47.772&25.920&15.457
\\
&$\epsilon=8$&\textbf{64.640}&\textbf{64.220}&\textbf{63.600}&63.189&63.189&\textbf{63.389}&58.455&49.263&26.885&24.247
\\
&non-DP&66.840&68.500&69.463&69.463&69.463&69.682&68.328&68.845&65.752&65.731
\\\hline\hline
\multirow{3}{*}{ROGUE-L}&$\epsilon=3$&\textbf{67.857}&\textbf{67.071}&\textbf{65.872}&65.670&65.670&\textbf{65.773}&65.560&58.964&44.536&35.240
\\
&$\epsilon=8$&\textbf{68.968}&\textbf{67.533}&\textbf{67.073}&66.429&66.429&\textbf{67.525}&65.030&60.730&46.421&39.951
\\
&non-DP&70.384&71.458&71.359&71.359&71.359&71.709&68.844&70.805&68.704&68.751
\\\hline
\end{tabular}
}
\label{tab:E2E GPT selected}
\end{table}

Competitive methods include low-rank adaption (LoRA), prefix-tuning (prefix), RGP, only fine-tuning the top 2 Transformer blocks
(top2), and training from scratch (retrain), as were recorded in \cite{li2021large}. Again, we use the \textit{exactly the same} hyperparameters as in \cite{li2021large}. For GPT2 (124 million parameters), GPT2 medium (355 million), and GPT2 large (774 million), \Cref{tab:E2E GPT selected} shows that AUTO-S is scalable with stronger performance on larger models. Our automatic full-parameter finetuning has the best overall performance. Additionally, we highlight that AUTO-S and methods like LoRA are not mutually exclusive and can be combined to yield strong performance, since AUTO-S modifies the optimizers and LoRA modifies the architecture.

\section{Related works}
\label{sec:related}

While other DP works also normalize the per-sample gradients (instead of clipping them) or use small clipping threshold (making the clipping similar to normalization), our work is very different in terms of theoretical analysis, algorithm design and experiments. In fact, the concurrent work \cite{yang2022normalized} gives the same algorithm as AUTO-S, although its theoretical analysis and experiment design is fundamentally different from ours. \cite{das2021convergence} proposes to normalize the per-user (not per-sample) gradient in the federated learning setting, and analyzes the convergence in a convex, non-deep-learning setting. 

On the other hand, many works apply the per-sample gradient clipping with small $R$ for good utility \cite{abadi2016deep,li2021large,mehta2022large,kurakin2022toward,de2022unlocking}. These works have led to valuable insights, but also some false or incomplete conclusions, due to the lack of rigorous theoretical analysis. For instance, since $R$ is present in the (re-parameterized) per-sample clipping, it cannot avoid the hyperparameter tuning as the choice of $R$ is not robust; even if a sufficiently small $R$ is used, the clipping does not reveal the stability constant in AUTO-S, which enjoys theoretical and empirical advantages in \Cref{rem:S better than V} and \Cref{sec:experiments}. We devote \Cref{app:more related} to more instances (e.g. Footnote 5) and a thorough comparison. 

\section{Discussion}
\label{sec:discussion}
In this work, we propose the automatic clipping as a drop-in replacement to the standard per-example clipping for differentially private training. This is the first technique that eliminates the need to tune 
the clipping threshold $R$, thus making DP deep learning as easy as regular learning. Our AUTO-S method enjoys both theoretical guarantee of convergence in non-convex problems (under various conditions), and strong empirical performance that advances DP learning on computer vision and language tasks.

We are excited about the future of automatic DP training, especially along with other working techniques, such as general optimizers (e.g. \cite{bu2021fast,du2021dp}), clipping styles (all-layer or per-layer or adaptive clipping), architecture modifications (e.g. LoRA, RGP, prefix), and data augmentation (e.g. adversarial training \cite{goodfellow2014explaining} and multiple augmentation \cite{de2022unlocking}). Thus, we expect to achieve comparable results to all SOTA in a lightweight fashion.

\section*{Acknowledgement}
We would like to thank Xuechen Li for updating his codebase and for his quick response in technical details to reproduce the results, which was crucial for benchmarking our experiments.

\bibliographystyle{plain}
\bibliography{reference}


\appendix

\clearpage
\section{Proof of differential privacy}
\label{app:proof DP}
\begin{proof}[Proof of \Cref{thm:auto is private}]
Define the \textbf{$ \ell_{2} $ sensitivity} of any function $ g $ to be
 $\Delta g=\sup _{S, S'}\|g(S)-g(S')\|_{2}$ where the supreme is over all neighboring $(S,S')$. Then the \textbf{Gaussian mechanism} $\hat g(S) =g(S) +\sigma\Delta g\cdot\mathcal{N}(0,\mathbf{I})$.

$\sigma$ denotes the ``Noise multiplier'', which corresponds to the noise-level when a Gaussian mechanism is applied to a query with sensitivity $1$.

Observe that automatic clipping (AUTO-V and AUTO-S \eqref{eq:Auto without R}) ensures the bounded global-sensitivity of the stochastic gradient as in Abadi's clipping. Aligning the noise-multiplier (rather than the noise-level itself) ensures that the the noise-to-sensitivity ratio $\frac{\sigma\Delta g}{\Delta g}=\sigma$ is fixed regardless of $\Delta g$. The Gaussian mechanism's privacy guarantees are equivalent. Thus from the privacy accountant perspective, DP-SGD with both Abadi's clipping and our autoclipping method can be equivalently represented as the adaptive composition of $T$ Poisson sampled Gaussian Mechanism with sampling probability $B/n$ and noise multiplier $\sigma$.
\end{proof}

\section{Proof of automaticity}

\subsection{Non-adaptive DP optimizers}
\label{app:proof automaticity non-adpative}

\begin{proof}[Proof of \Cref{thm: non-adaptive automatic}]
We prove \Cref{thm: non-adaptive automatic} by showing that, DP-SGD using $R$-dependent AUTO-S with learning rate $\eta$ and weight decay $\lambda$ is equivalent to $R$-independent AUTO-S with learning rate $\eta R$ and weight decay $\lambda/R$. We claim other non-adaptive optimizers such as HeavyBall and NAG can be easily shown in a similar manner.

Recall the standard SGD with weight decay is
$$
\w_{t+1}=\w_t-\eta\left(\sum_{i\in B_t} \frac{\partial l_i}{\partial \w_t}+\lambda \w_t\right)
$$

Replacing the standard gradient $\sum_i \frac{\partial l_i}{\partial \w_t}$ with the private gradient, we write the $R$-dependent case as
\begin{align*}
\w_{t+1}
&=\w_t-\eta\left(\sum_{i\in B_t} \frac{\partial l_i}{\partial \w_t}\cdot R/\|\frac{\partial l_i}{\partial \w_t}\|_2+\sigma R\cdot\mathcal{N}(0,\I)+\lambda \w_t\right)
\\
&=\w_t-\eta R\left(\sum_{i\in B_t} \frac{\partial l_i}{\partial \w_t}/\|\frac{\partial l_i}{\partial \w_t}\|_2+\sigma \cdot\mathcal{N}(0,\I)\right)-\eta \lambda\w_t
\end{align*}
which is clearly equivalent to the $R$-independent case:
\begin{align*}
\w_{t+1}
&=\w_t-\eta'\left(\sum_{i\in B_t} \frac{\partial l_i}{\partial \w_t}/\|\frac{\partial l_i}{\partial \w_t}\|_2+\sigma \cdot\mathcal{N}(0,\I)+\lambda' \w_t\right)
\end{align*}
if we use $\eta'=\eta R$ and $\lambda'=\lambda/R$.
\end{proof}

\subsection{Adaptive DP optimizers}
\label{app:proof automaticity adpative}

\begin{proof}[Proof of \Cref{thm: adaptive automatic}]
We prove \Cref{thm: adaptive automatic} by showing that, DP-AdamW using $R$-dependent AUTO-S with learning rate $\eta$ and weight decay $\lambda$ is equivalent to $R$-independent AUTO-S with the same learning rate $\eta$ and weight decay $\lambda/R$. This is the most complicated case. We claim other adaptive optimizers such as AdaDelta, Adam with weight decay (not AdamW), and NAdam can be easily shown in a similar manner.

Recall the standard AdamW is
$$\w_{t+1}=\w_t-\eta\left(\frac{\m_t/(1-\beta_1)}{\sqrt{\v_t/(1-\beta_2)}}+\lambda\w_t\right)$$
where $\beta_1,\beta_2$ are constants, $\g_t:=\sum_i \frac{\partial l_i}{\partial \w_t}$ is the standard gradient,
$$\m_t=\beta_1 \m_{t-1}+(1-\beta_1)\g_t \longrightarrow \m_t=\sum_\tau \beta_1^{t-\tau}(1-\beta_1)\g_\tau,$$
$$\v_t=\beta_2 \v_{t-1}+(1-\beta_2)\g_t^2 \longrightarrow \v_t=\sum_\tau \beta_2^{t-\tau}(1-\beta_2)\g_\tau^2.$$

Replacing the standard gradient with the private gradient $R\tilde \g_t:=R(\sum_i \frac{\partial l_i}{\partial \w_t}/\|\frac{\partial l_i}{\partial \w_t}\|_2+\sigma \cdot\mathcal{N}(0,I))$, we write the $R$-dependent DP-AdamW as
$$\w_{t+1}=\w_t-\eta\left(\frac{\tilde\m_t/(1-\beta_1)}{\sqrt{\tilde\v_t/(1-\beta_2)}}+\lambda\w_t\right)$$
where 
$$\tilde\m_t=\beta_1 \tilde\m_{t-1}+(1-\beta_1)R\tilde\g_t \longrightarrow \tilde\m_t=\sum_\tau \beta_1^{t-\tau}(1-\beta_1)R\tilde\g_\tau,$$
$$\tilde\v_t=\beta_2 \tilde\v_{t-1}+(1-\beta_2)R^2\tilde\g_t^2 \longrightarrow \tilde\v_t=\sum_\tau \beta_2^{t-\tau}(1-\beta_2)R^2\tilde\g_\tau^2.$$

Clearly, the $R$ factor in the numerator and denominator of $\frac{\tilde\m_t/(1-\beta_1)}{\sqrt{\tilde\v_t/(1-\beta_2)}}$ cancel each other. Therefore we claim that the $R$-dependent DP-AdamW is in fact completely independent of $R$.
\end{proof}

\subsection{Automatic per-layer clipping}
\label{app:automatic per-layer}
In some cases, the per-layer clipping is desired, where we use a clipping threshold vector $\bm{R}=[R_1,\cdots,R_L]$ and each layer uses a different clipping threshold. We claim that DP optimizers under automatic clipping works with the per-layer clipping when $\bm{R}$ is tuned proportionally, e.g. $\bm{R}=R\cdot[a_1,\cdots,a_L]$, but not entry-wise (see counter-example in \Cref{fact:per-layer may not compatible with AUTO}). One special case is the \textit{uniform per-layer clipping} when $R_1=\cdots=R_L=R/\sqrt{L}$. This is widely applied as only one norm $R$ requires tuning, instead of $L$ norms in $\bm{R}$, particularly in the case of deep models with hundreds of layers. The corresponding DP-SGD with AUTO-S in \eqref{eq:Auto-S} gives
$$\w_{t+1}^{(l)}=\w_{t}^{(l)}-\eta\left(\sum_{i\in B_t} \frac{R}{\sqrt{L}}\frac{\g_{t,i}^{(l)}}{||\g_{t,i}^{(l)}||+\gamma}+\sigma R\cdot\mathcal{N}(0,\I)\right)$$
Here the superscript $(l)$ is the layer index. Clearly $R$ couples with the learning rate $\eta$ and the same analysis as in \Cref{thm: non-adaptive automatic} follows. The adaptive optimizers can be similarly analyzed from \Cref{thm: adaptive automatic}.

\begin{fact}\label{fact:per-layer may not compatible with AUTO}
Changing one clipping threshold in the clipping threshold vector $\bm R$ (i.e. not proportionally) can break the coupling with learning rate.
\end{fact}
\begin{proof}[Proof of \Cref{fact:per-layer may not compatible with AUTO}]
We prove by a counter-example of $\bm R$ in $\R^2$. Consider DP-SGD with per-layer clipping thresholds $(R_1,R_2)=(9,12)$:
$$\w_{t+1}^{(l)}=\w_{t}^{(l)}-\eta \left(\sum_{i\in B} \frac{R_l\g_{t,i,l}}{||\g_{t,i,l}||}+\sigma\sqrt{R_1^2+R_2^2}\cdot\mathcal{N}(0,\I)\right)$$
Increasing $R_1$ from 9 to 16 changes the update for the first layer
$$\eta\left(\sum_{i\in B} \frac{9\g_{t,i,l}}{||\g_{t,i,l}||}+15\sigma\cdot\mathcal{N}(0,1)\right)\to \eta\left(\sum_{i\in B} \frac{16\g_{t,i,l}}{||\g_{t,i,l}||}+20\sigma\cdot\mathcal{N}(0,\I)\right)$$
The noise-to-signal ratio decreases from 5/3 to 5/4 for this layer, and increases from 5/4 to 5/3 for the second layer. This breaks the coupling with learning rate, since the coupling does not change the noise-to-signal ratio.
\end{proof}

\section{Main results of convergence for DP-SGD with automatic clipping}
\subsection{Main proof of convergence for DP-SGD (the envelope version)}
\label{app: DPSGD convergence with AUTO}

\begin{proof}[Proof of \Cref{thm:upper bounding grad norm without r}]
In this section, we prove two parts of \Cref{thm:upper bounding grad norm without r}. 

The first part of \Cref{thm:upper bounding grad norm without r} is the upper bound on $\min_t \E(\|\g_t\|)$, which is a direct result following from \Cref{thm: convergence DPSGD AUTO}, and we prove it in \Cref{app: DPSGD convergence with AUTO(non-envelope)}.

\begin{theorem}
\label{thm: convergence DPSGD AUTO}
Under \Cref{assumption: lower bounding loss}, \ref{assumption: Lipschitz}, \ref{assumption: tilde g}, running DP-SGD with automatic clipping for $T$ iterations gives
\begin{align}
\min_t\E(\|\g_t\|)\leq\frac{\xi}{r}+ \mathcal{F}\left(\frac{4}{\sqrt{T}}\sqrt{(\mathcal{L}_0-\mathcal{L}_*)L\left(1+\frac{\sigma^2 d}{B^2} \right)};r,\xi,\gamma\right)
\label{eq:min grad norm with r}
\end{align}
where
\begin{itemize}
    \item for $r< 1, \gamma=0$ and $\eta \propto 1/\sqrt{T}$, 
    $\mathcal{F}(x)=\frac{x}{\min_{0<c<1} f(c,r)}$
    and $f(c,r):=\frac{(1+rc)}{\sqrt{r^2+2rc+1}}+\frac{(1-rc)}{\sqrt{r^2-2rc+1}}$; for $r\geq 1, \gamma=0$ and $\eta \propto 1/\sqrt{T}$, 
    $\mathcal{F}(x)=\infty$;
    \item for $r\geq 1, \gamma>0$ and $\eta \propto 1/\sqrt{T}$, $\mathcal{F}$ is the convex envelope of \eqref{eq:M inverse complex formula}, and is strictly increasing.
\end{itemize}
\end{theorem}

\begin{figure}[!htb]
    \centering
    \includegraphics[width=0.4\linewidth]{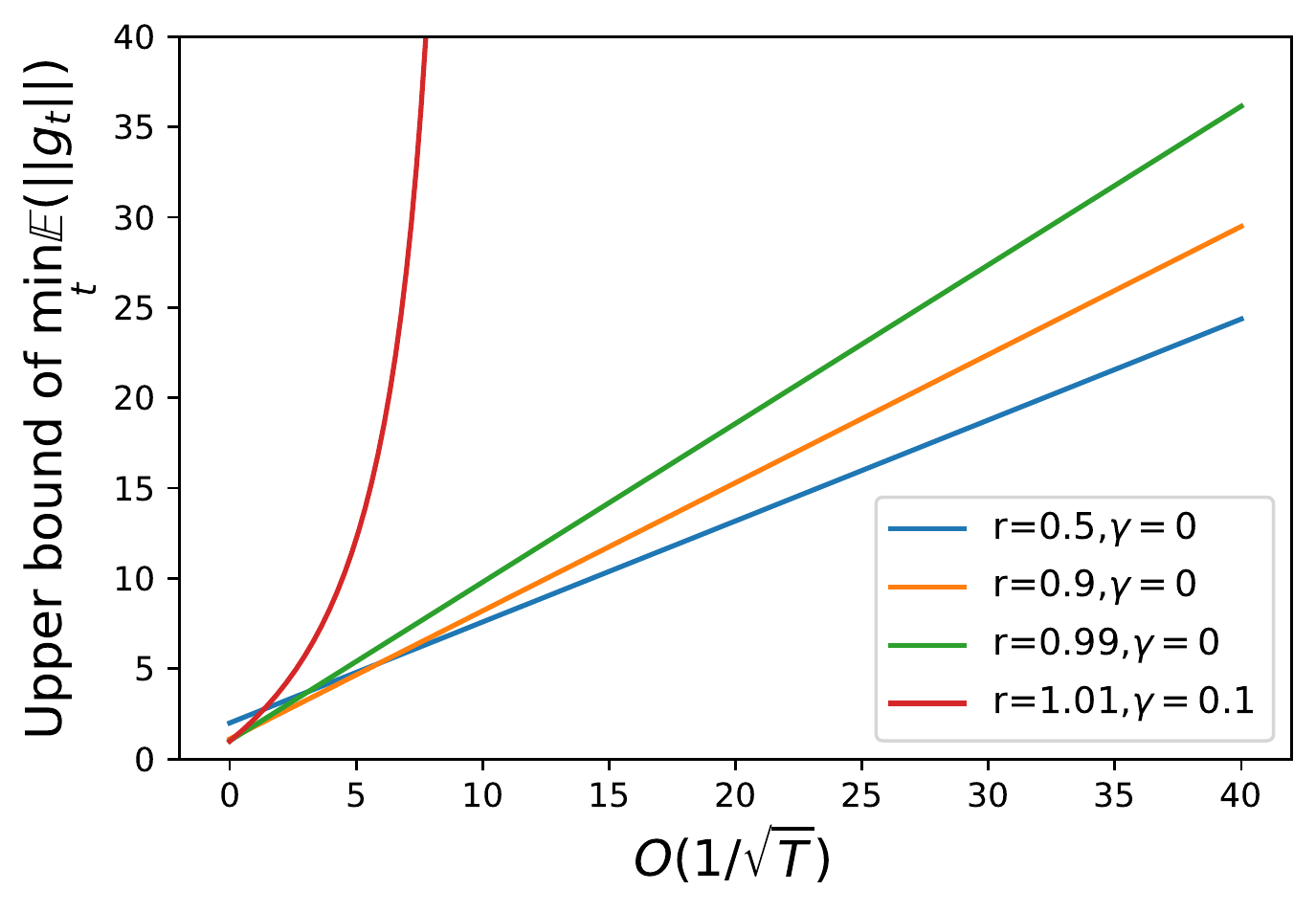}
\includegraphics[width=0.4\linewidth]{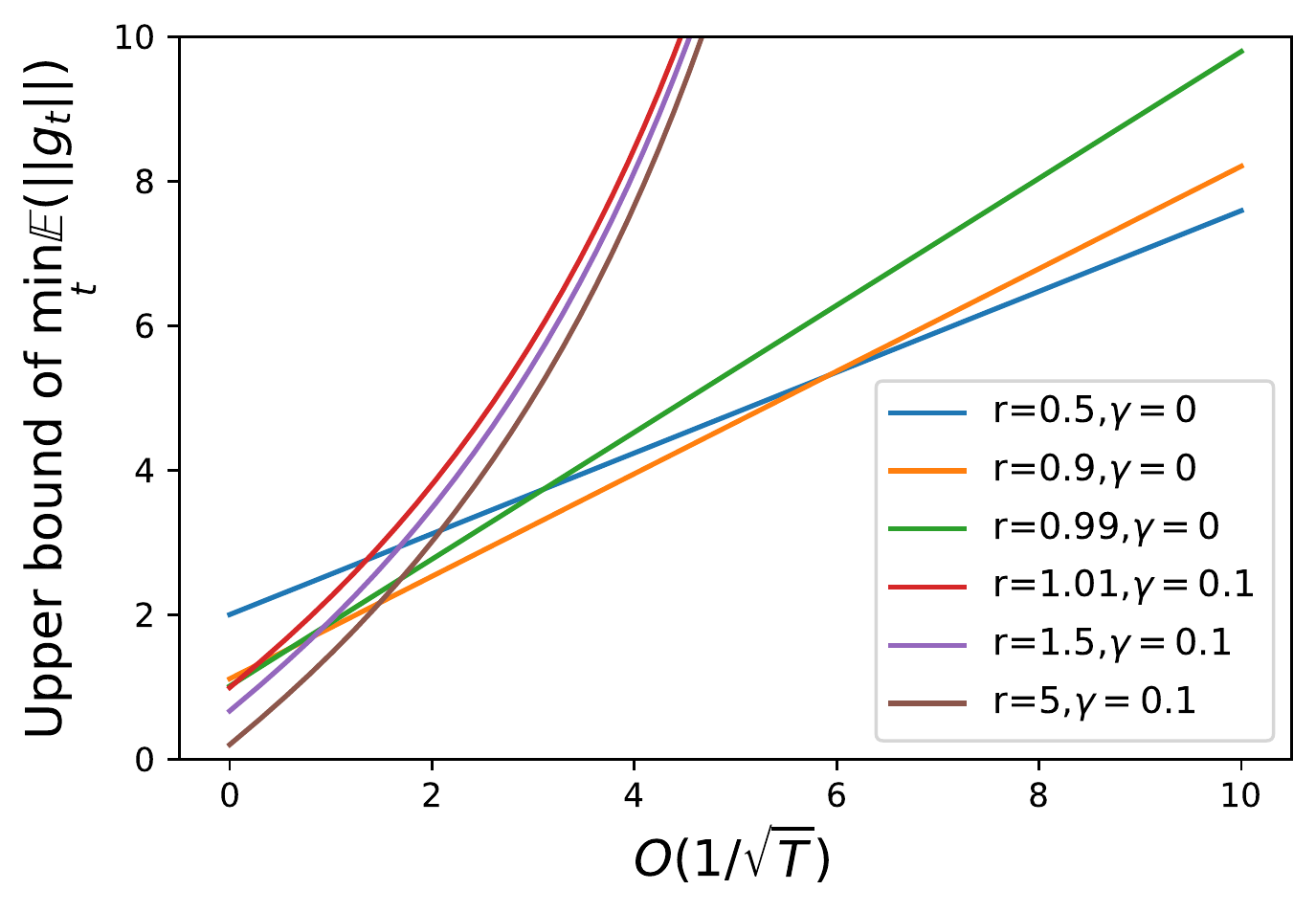}
    \caption{Visualization of upper bound $\frac{\xi}{r}+\mathcal{F}\left(O(1/\sqrt{T});r,\xi,\gamma\right)$ for gradient norm, with $O(1/\sqrt{T})$ in \eqref{eq:min grad norm with r}. Here $\xi=1$. The right plot is a zoom-in (with additional lines) of the left one.}
    \label{fig:min grad norm with r}
\end{figure}
Notice that, \eqref{eq:min grad norm with r} holds for any $r>0$. However, we have to consider an envelope curve over $r$ in \eqref{eq:min grad norm with r} to reduce the upper bound: with AUTO-V clipping ($\gamma=0$), the upper bound in \eqref{eq:min grad norm with r} is always larger than $\xi$ as $r<1$; we must use AUTO-S clipping ($\gamma>0$) to reduce the upper bound to zero, as can be seen from \Cref{fig:min grad norm with r}. In fact, larger $T$ needs larger $r$ to reduce the upper bound. 

All in all, we specifically focus on $r\geq 1$ and $\gamma>0$, which is the only scenario that \eqref{eq:min grad norm with r} can converge to zero. This scenario is also where we prove the second part of \Cref{thm:upper bounding grad norm without r}.

The second part of \Cref{thm:upper bounding grad norm without r} is the asymptotic convergence rate $O(T^{-1/4})$ of DP-SGD, only possible under $r\geq 1$ and $\gamma>0$.

By \eqref{eq:min grad norm with r} in \Cref{thm: convergence DPSGD AUTO}, our upper bound $\mathcal{G}$ from \Cref{thm:upper bounding grad norm without r} can be simplified to
$$\min_{r>0} \frac{\xi}{r}+(\mathcal{M}^{-1})_{ccv}\left(\frac{4}{\sqrt{T}}\sqrt{(\mathcal{L}_0-\mathcal{L}_*)L\left(1+\frac{\sigma^2 d}{B^2} \right)};r,\xi,\gamma\right)$$
where the function $\mathcal{M}^{-1}$ is explicitly defined in \eqref{eq:M inverse complex formula} and the subscript $ccv$ means the upper concave envelope. Clearly, as $T\to\infty$, $\mathcal{M}^{-1}(\frac{1}{\sqrt{T}})\to 0$. We will next show that the convergence rate of $\mathcal{M}^{-1}$ is indeed $O(\frac{1}{\sqrt{T}})$ and the minimization over $r$ makes the overall convergence rate $O(T^{-1/4})$.

Starting from \eqref{eq:M inverse complex formula}, we denote $x=\frac{4}{\sqrt{T}}\sqrt{(\mathcal{L}_0-\mathcal{L}_*)L\left(1+\frac{\sigma^2 d}{B^2} \right)}$ and write
\begin{align*}
\M^{-1}(x;r,\xi,\gamma)
&=
\frac{-\frac{\xi}{r}\gamma+(r^2-1)\frac{\xi}{r}x+r\gamma x+\gamma\sqrt{(\frac{\xi}{r})^2+2\xi x+2\gamma x+x^2}}{2\gamma-(r^2-1)x}
\\
&=\left(-\frac{\gamma\xi}{r}+(r^2-1)\frac{\xi}{r}x+r\gamma x+\gamma\sqrt{(\frac{\xi}{r})^2+2\xi x+2\gamma x+x^2}\right)
\\
&\quad\cdot\frac{1+\frac{r^2-1}{2\gamma}x+O(x^2)}{2\gamma}
\\
&=\frac{1}{2\gamma}\left(-\frac{\gamma\xi}{r}+(r^2-1)\frac{\xi}{r}x+r\gamma x+\frac{\gamma\xi}{r}\sqrt{1+\frac{2(\xi+\gamma)r^2 x}{\xi^2}+O(x^2)}\right)
\\
&\quad\cdot(1+\frac{r^2-1}{2\gamma}x+O(x^2))
\\
&=\frac{1}{2\gamma}\left(-\frac{\gamma\xi}{r}+(r^2-1)\frac{\xi}{r}x+r\gamma x+\frac{\gamma\xi}{r}\left(1+\frac{(\xi+\gamma)r^2 x}{\xi^2}+O(x^2)\right)\right)
\\
&\quad\cdot(1+\frac{r^2-1}{2\gamma}x+O(x^2))
\\
&=\frac{1}{2\gamma}\left((r^2-1)\frac{\xi}{r}x+r\gamma x+\frac{\gamma(\xi+\gamma)r x}{\xi}+O(x^2)\right)\cdot(1+\frac{r^2-1}{2\gamma}x+O(x^2))
\\
&=\frac{1}{2\gamma}\left((r^2-1)\frac{\xi}{r}+r\gamma +\frac{\gamma(\xi+\gamma)r }{\xi}\right)\cdot x+O(x^2)
\\
&=\frac{1}{2\gamma}\left(\frac{(\xi+\gamma)^2}{\xi}r-\frac{\xi}{r}\right)\cdot x+O(x^2)
\end{align*}

Since $\mathcal{M}^{-1}$ is asymptotically linear as $x\to 0$, we instead study
$$\min_{r>0} \frac{\xi}{r}+\mathcal{M}^{-1}\left(x;r,\xi,\gamma\right)\equiv\min_{r>0} \frac{\xi}{r}+\frac{1}{2\gamma}\left(\frac{(\xi+\gamma)^2}{\xi}r-\frac{\xi}{r}\right)\cdot x+O(x^2).$$

That is, ignoring the higher order term for the asymptotic analysis, the $\mathcal{M}^{-1}$ part converges as $O(x)=O(1/\sqrt{T})$, and we visualize this in \Cref{fig:overT}.

Although DP-SGD converges faster than SGD, the former converges to $\xi/r$ and the latter converges to 0. Thus, taking $\xi/r$ into consideration, the objective reduces to a hyperbola
$$\frac{\left(\xi(1-\frac{x}{2\gamma})\right)}{r}+\frac{x(\xi+\gamma)^2}{2\gamma\xi}\cdot r$$
whose minimum over $r$ is obviously $2\sqrt{\xi(1-\frac{x}{2\gamma})\frac{x(\xi+\gamma)^2}{2\gamma\xi}}=O(\sqrt{x})=O(T^{-1/4})$.
\end{proof}

To give more details about the upper bound in \eqref{eq:min grad norm without r}, we demonstrate its dependence on $\xi$ and $\gamma$ in \Cref{fig:G dependence on xi and gamma}.
\begin{figure}[!htb]
    \centering
    \includegraphics[width=0.4\linewidth]{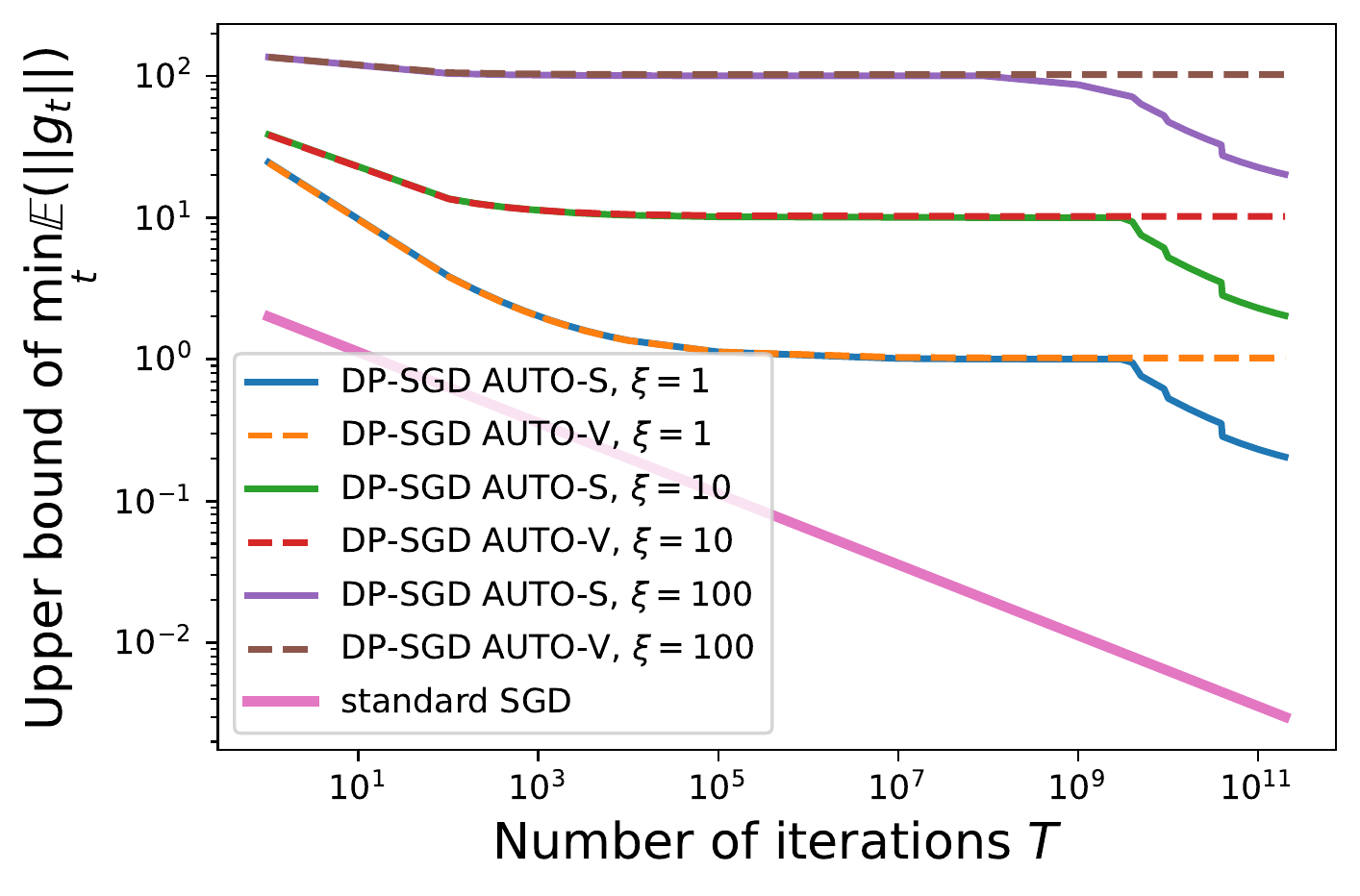}
    \includegraphics[width=0.4\linewidth]{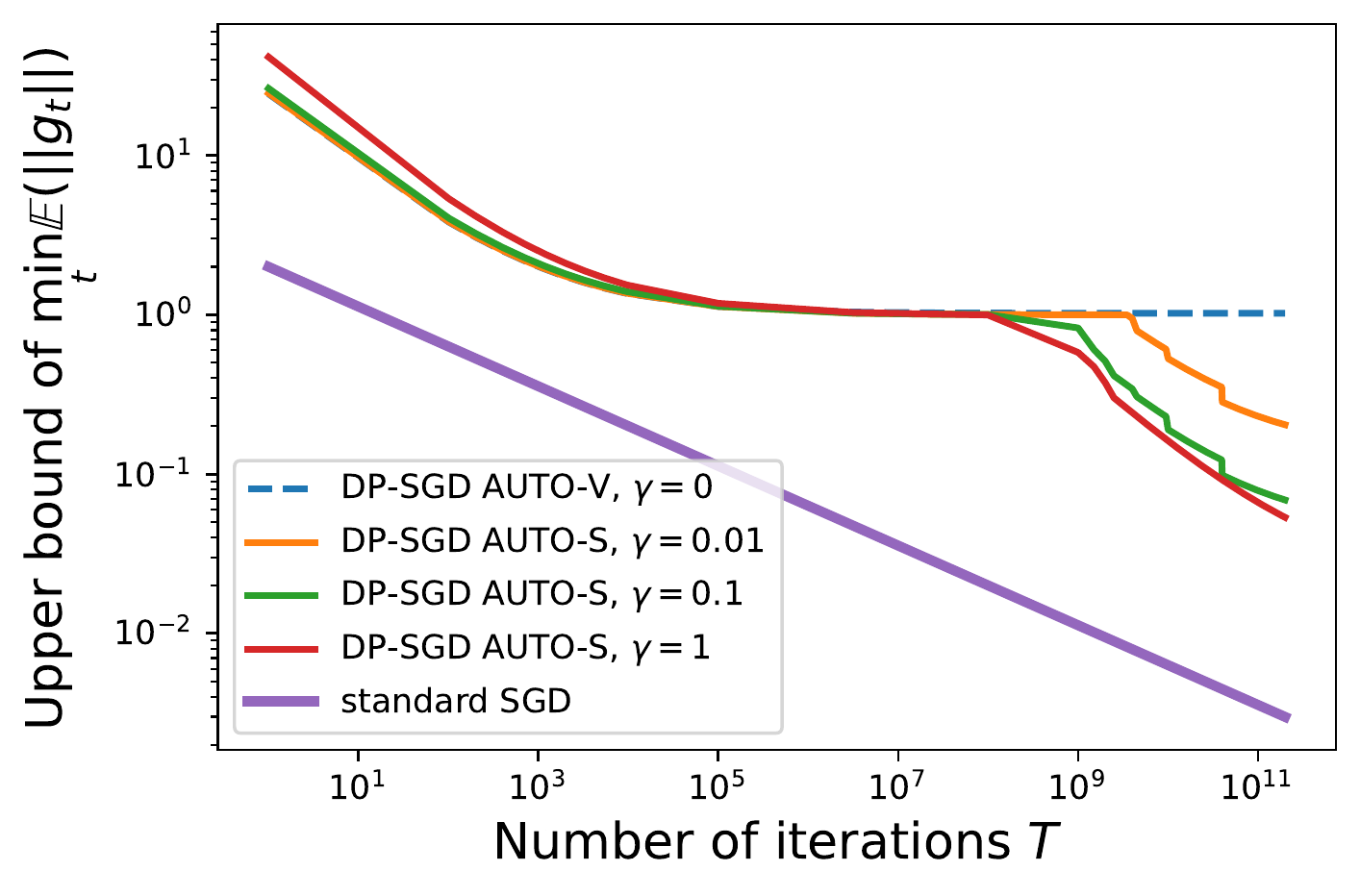}
    \caption{Dependence of the upper bound $\mathcal{G}$ on $\xi$ (left) and $\gamma$ (right). Here the $O(1/\sqrt{T})$ term is set to 10 and either $\gamma=0.01$ (left) or $\xi=1$ (right).}
    \label{fig:G dependence on xi and gamma}
\end{figure}

\begin{figure}[!htb]
    \centering
    \includegraphics[width=0.5\linewidth]{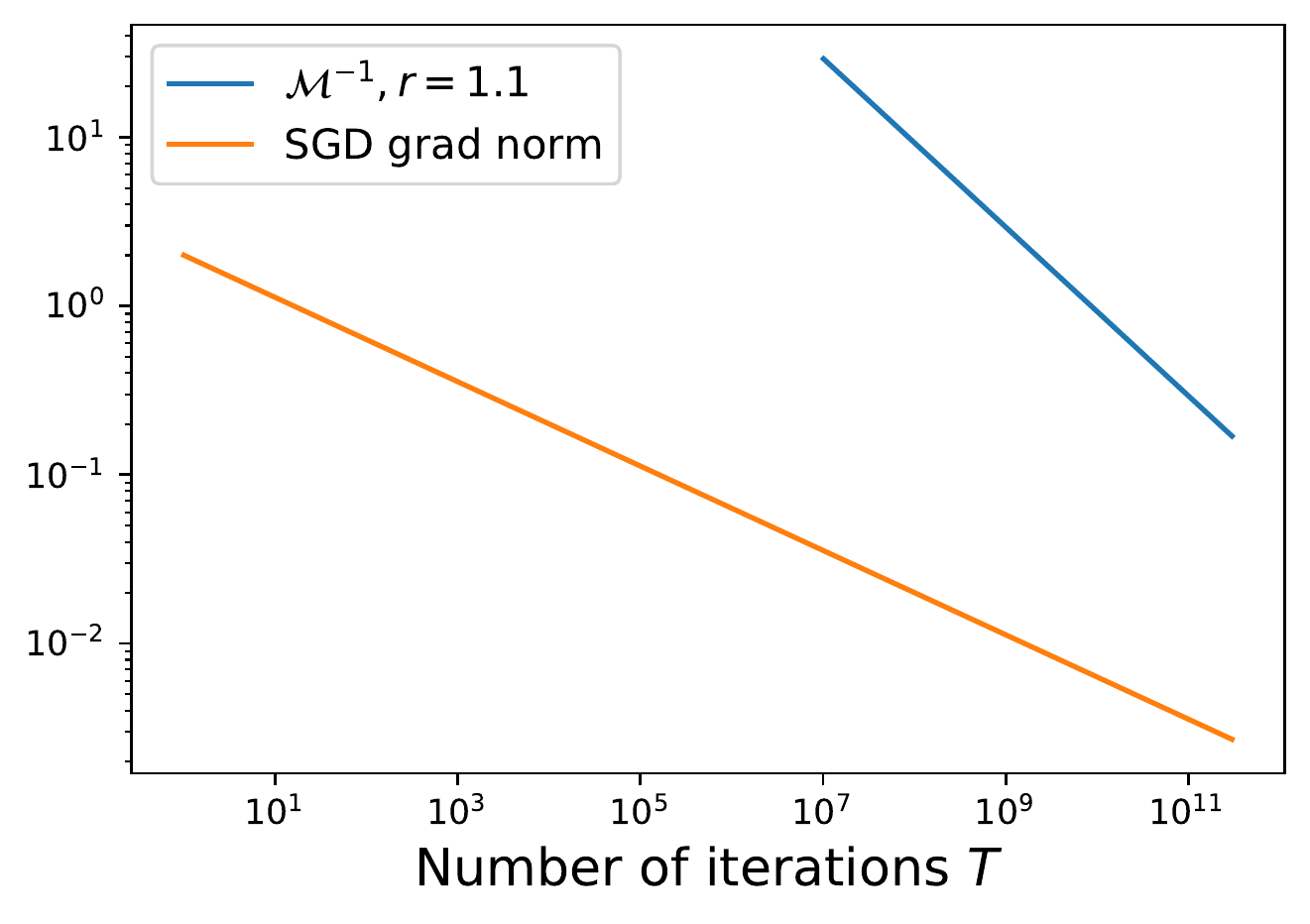}
    \caption{Convergence with respect to $T$. Same setting as \Cref{fig: G plots}.}
    \label{fig:overT}
\end{figure}

\subsection{Main proof of convergence for DP-SGD (the non-envelope version)}
\label{app: DPSGD convergence with AUTO(non-envelope)}

\begin{proof}[Proof of \Cref{thm: convergence DPSGD AUTO}]

Consider DP-SGD with AUTO-S clipping
$$\w_{t+1}=\w_{t}-\eta \left(\sum_i \frac{\tilde \g_{t,i}}{\|\tilde \g_{t,i}\|+\gamma}+\sigma \mathcal{N}(0,\I)\right)$$
where $\tilde \g_{t,i}$ is i.i.d. samples of $\tilde\g_t$, an unbiased estimate of $\g_t$, with a bounded variance as described in \Cref{assumption: tilde g}.

By Lipschitz smoothness in \Cref{assumption: Lipschitz}, and denoting $Z=\mathcal{N}(0,\I)$, we have
\begin{align}
\begin{split}
\mathcal{L}_{t+1}-\mathcal{L}_t
&\leq \g_t^\top (\w_{t+1}-\w_t)+\frac{L}{2}\|\w_{t+1}-\w_t\|^2
\\
&=-\eta \g_t^\top\left(\sum_i C_i\tilde \g_{t,i}+\sigma Z\right)+\frac{L\eta^2}{2}\left\|\sum_i C_i\tilde \g_{t,i}+\sigma Z\right\|^2
\\
&\leq -\eta   \g_t^\top\left(\sum_i C_i\tilde \g_{t,i}+\sigma Z\right)
+L\eta^2 \left(\left\|\sum_i C_i\tilde \g_{t,i}\right\|^2+\sigma^2 \|Z\|^2\right)
\\
&\leq -\eta   \g_t^\top\left(\sum_i C_i\tilde \g_{t,i}+\sigma Z\right)
+L\eta^2 \left(B^2+\sigma^2 \|Z\|^2\right)
\end{split}
\label{eq:lipschitz expand}
\end{align}
where the second last inequality follows from Cauchy Schwartz, and the last inequality follows from the fact that $\|C_i\tilde \g_{t,i}\|\leq 1$, e.g. $C_i$ is $\|\tilde \g_{t,i}/(\|\tilde \g_{t,i}\|+\gamma)\|$ or the re-parameterized clipping in \cite{de2022unlocking}. 

Notice that in the last equality, the first term (ignoring $\g_t^\top Z$ for its zero expectation) can be written in the same form as \eqref{eq:clipping max}, which supports our motivation in \Cref{sec:normal motivation}; the second term is independent of clipping functions. Note that the last inequality is tight if and only if $C_i=1$. This empirically holds in \Cref{app:freq of clipping}, especially for GPT2.

Given the fact that $\|\tilde \g_{t,i}/(\|\tilde \g_{t,i}\|+\gamma)\|\leq 1$, the expected improvement at one iteration is
\begin{align}
\begin{split}
\E (\mathcal{L}_{t+1}-\mathcal{L}_t|\w_t)
&\leq -\eta   \g_t^\top\E\left(\sum_i\frac{ \tilde \g_{t,i}}{\|\tilde \g_{t,i}\|+\gamma}\right)+L\eta^2 \left(B^2+\sigma^2 d \right)
\\
&=-\eta B\g_t^\top\E\left(\frac{\tilde \g_{t}}{\|\tilde \g_{t}\|+\gamma}\right)+L\eta^2 \left(B^2+\sigma^2 d \right)
\end{split}
\label{eq:expected improvement}
\end{align}

Now we want to lower bound $\g_t^\top\E\left(\frac{\tilde \g_{t}}{\|\tilde \g_{t}\|+\gamma}\right)$ in \eqref{eq:expected improvement}.

Write $\tilde \g_{t}=\g_t+\Delta_t$ where the gradient noise $\Delta_t$ follows $\E\Delta_t=0, \E\|\Delta_t\|<\xi$ by
\Cref{assumption: tilde g}. Then 
\begin{align*}
\g_t^\top\E\left(\frac{\tilde \g_{t}}{\|\tilde \g_{t}\|+\gamma}\right)
&= \E\left(\frac{\|\g_t\|^2
+\g_t^\top \Delta_{t}}{\|\g_{t}+\Delta_t\|+\gamma}\right)
\\
&=\frac{1}{2}\E\left(\frac{\|\g_t\|^2
+\g_t^\top \Delta_{t}}{\|\g_{t}+\Delta_t\|+\gamma}\Big|\Delta_t\in H_+\right)+\frac{1}{2}\E\left(\frac{\|\g_t\|^2
+\g_t^\top \Delta_{t}}{\|\g_{t}+\Delta_t\|+\gamma}\Big|\Delta_t\in H_-\right)
\\
&=\frac{1}{2}\E\left(\frac{\|\g_t\|^2
+\g_t^\top \Delta_{t}}{\|\g_{t}+\Delta_t\|+\gamma}\Big|\Delta_t\in H_+\right)+\frac{1}{2}\E\left(\frac{\|\g_t\|^2
-\g_t^\top \Delta_{t}}{\|\g_{t}-\Delta_t\|+\gamma}\Big|\Delta_t\in H_+\right)
\end{align*}
where we use the hyperplane perpendicular to $\g_t$ to divide the support of $\Delta_t$ into two half-spaces:
\begin{align*}
H_+:=\{\v: \g_t^\top \v>0\},
\quad
H_-:=\{\v: \g_t^\top \v<0\}.
\end{align*}
We use the symmetry assumption in \Cref{assumption: tilde g} to get
$$\P(\Delta_t\in H_+)=\P(\Delta_t\in H_-)=\frac{1}{2}$$
and notice that
$\Delta_t\overset{D}{=}-\Delta_t,$
i.e., if $\Delta_t\in H_+$, then $-\Delta_t\in H_-$ with the same distribution.

The next result further gives a lower bound for $\g_t^\top\E\left(\frac{\tilde \g_{t}}{\|\tilde \g_{t}\|+\gamma}\right)$ using $\|\g_t\|$.

\begin{lemma}
\label{lem:magic inequality2}
$$\E\left(\frac{\|\g_t\|^2
+\g_t^\top \Delta_{t}}{\|\g_{t}+\Delta_t\|+\gamma}+\frac{\|\g_t\|^2
-\g_t^\top \Delta_{t}}{\|\g_{t}-\Delta_t\|+\gamma}\Big|\Delta_t\in H_+\right)\geq \min_{0<c\leq 1} f(c,r;\frac{\gamma}{\|\g_t\|})\cdot(\|\g_t\|-\xi/r)$$
for any $r>0$ and $f(c,r;\Gamma)= \frac{(1+rc)}{\sqrt{r^2+2rc+1}+\Gamma}+\frac{(1-rc)}{\sqrt{r^2-2rc+1}+\Gamma}$.
\end{lemma}

For the simplicity of notation, we denote the distance measure
\begin{align}
\M(\|\g_t\|-\xi/r;r,\xi,\gamma)=\min_{0<c\leq 1} f\left(c,r;\frac{\gamma}{\|\g_t\|}\right)\cdot(\|\g_t\|-\xi/r)
\label{eq:final distance}
\end{align}
and leave the fine-grained analysis (e.g. its explicit form in some scenarios) at the end of this section.

Using the lower bound from \Cref{lem:magic inequality2}, the expected improvement \eqref{eq:expected improvement} becomes
\begin{align*}
\E (\mathcal{L}_{t+1}-\mathcal{L}_t|\w_t)
&\leq -\frac{\eta B }{2} \M(\|\g_t\|-\xi/r)+L\eta^2 B^2\left(1+\frac{\sigma^2 d}{B^2} \right)
\end{align*}

Now extend the expectation over randomness in the trajectory, and perform a telescoping sum over the iterations
\begin{align*}
\mathcal{L}_0-\mathcal{L}_*
&\geq \mathcal{L}_0-\E \mathcal{L}_T=\sum_t \E(\mathcal{L}_t-\mathcal{L}_{t+1})
\\
&\geq \frac{\eta B }{2}\E\left(\sum_t\M(\|\g_t\|-\xi/r)\right)-TL\eta^2 B^2\left(1+\frac{\sigma^2 d}{B^2} \right)
\end{align*}

Substituting $\eta B =\eta_\text{0}/\sqrt{T}$ where $\eta_\text{0}$ is a base learning rate, we have
\begin{align*}
2(\mathcal{L}_0-\mathcal{L}_*)
\geq \sqrt{T}\eta_\text{0} \E\left(\frac{1}{T}\sum_t\M(\|\g_t\|-\xi/r)\right)-2L\eta^2_\text{0}\left(1+\frac{\sigma^2 d}{B^2} \right)
\end{align*}
and finally
\begin{align}
\E\left(\frac{1}{T}\sum_t\M(\|\g_t\|-\xi/r)\right)
\leq \frac{1}{\sqrt{T}}\left[\frac{2(\mathcal{L}_0-\mathcal{L}_*)}{\eta_\text{0}}
+2L\eta_\text{0} \left(1+\frac{\sigma^2 d}{B^2} \right)\right]
\label{eq:DP-SGD hyperbola lr}
\end{align}
With $\eta_\text{0}$ chosen properly at $\eta_0=\sqrt{\frac{\mathcal{L}_0-\mathcal{L}_*}{L \left(1+\frac{\sigma^2 d}{B^2} \right)}}$, the hyperbola on the right hand side in \eqref{eq:DP-SGD hyperbola lr} is minimized to $4\sqrt{(\mathcal{L}_0-\mathcal{L}_*)L\left(1+\frac{\sigma^2 d}{B^2} \right)}$, and we obtain
\begin{align*}
\E\left(\frac{1}{T}\sum_t\M(\|\g_t\|-\xi/r)\right)
\leq \frac{4}{\sqrt{T}}\sqrt{(\mathcal{L}_0-\mathcal{L}_*)L\left(1+\frac{\sigma^2 d}{B^2} \right)}
\end{align*}

Since the minimum of a sequence is smaller than the average, we have
\begin{align}
\min_t \E(\M(\|\g_t\|-\xi/r))\leq\frac{1}{T}\sum_t\E\left(\M(\|\g_t\|-\xi/r)\right)
\leq \frac{4}{\sqrt{T}}\sqrt{(\mathcal{L}_0-\mathcal{L}_*)L\left(1+\frac{\sigma^2 d}{B^2} \right)}
\label{eq:ready concave}
\end{align}

We claim that $\M$ may not be concave or convex. Therefore we use $\M_{cvx}$ to denote its lower convex envelope, i.e. the largest convex function that is smaller than $\M$. Then by Jensen's inequality \eqref{eq:ready concave} becomes
\begin{align}
\min_t \M_{cvx}(\E(\|\g_t\|-\xi/r))\leq\min_t \E(\M_{cvx}(\|\g_t\|-\xi/r))\leq \frac{4}{\sqrt{T}}\sqrt{(\mathcal{L}_0-\mathcal{L}_*)L\left(1+\frac{\sigma^2 d}{B^2} \right)}
\label{eq:concave}
\end{align}

It is obvious that $\M_{cvx}$ is increasing as $\M$ is increasing by \Cref{thm: min f explicit}. Hence, $(\M_{cvx})^{-1}$ is also increasing, as the inverse of $\M_{cvx}$. We write \eqref{eq:concave} as
\begin{align*}
\min_t \E(\|\g_t\|-\xi/r)\leq (\M_{cvx})^{-1}\left(\frac{4}{\sqrt{T}}\sqrt{(\mathcal{L}_0-\mathcal{L}_*)L\left(1+\frac{\sigma^2 d}{B^2} \right)}\right)
\end{align*}
and equivalently
\begin{align}
\min_t \E(\|\g_t\|)\leq \frac{\xi}{r}+(\M_{cvx})^{-1}\left(\frac{4}{\sqrt{T}}\sqrt{(\mathcal{L}_0-\mathcal{L}_*)L\left(1+\frac{\sigma^2 d}{B^2} \right)}\right)
\label{eq:concave inverse}
\end{align}

Finally, we derive the explicit properties of $\M(\|\g_t\|-\xi/r)$ in \Cref{thm: min f explicit}. These properties allow us to further analyze on the convergence of $\M(\|\g_t\|-\xi/r)$, based on AUTO-V and AUTO-S, respectively.

\paragraph{1. DP-SGD with AUTO-V clipping.}
By \Cref{thm: min f explicit}, we write 
$$\M(x;r)=\min_{c\in(0,1]} f(c,r;0)\cdot x$$
This is a linear function and thus $\M_{cvx}=\M=1/\M_{cvx}^{-1}$. As a result, we have
\begin{align*}
\min_t \E(\|\g_t\|)
\leq \frac{\xi}{r}+\frac{1}{\min_{c\in(0,1]} f(c,r;0)}\cdot\frac{4}{\sqrt{T}}\sqrt{(\mathcal{L}_0-\mathcal{L}_*)L\left(1+\frac{\sigma^2 d}{B^2} \right)}
\end{align*}

We note here $r$ plays an important role under AUTO-V clipping: when $r<1$, we spend more iterations to converge to better and smaller gradient norm $\xi/r$; when $r\geq 1$, $\min_c f(c,r;0)=f(1,r;0)=0$ and it takes forever to converge. This is demonstrated in the left plot of \Cref{fig: G plots}.

\paragraph{2. DP-SGD with AUTO-S clipping.}
By \Cref{thm: min f explicit} and for $r>1$, we write $$\M(x;r,\xi,\gamma)=\left(\frac{\gamma}{(r-1)(x+\xi/r)+\gamma}-\frac{\gamma}{(r+1)(x+\xi/r)+\gamma}\right)\cdot x.$$

Notice that the inverse of a lower convex envelope is equivalent to the upper concave envelope (denoted by the subscript $ccv$) of an inverse. Therefore we can derive $(\M_{cvx})^{-1}=(\M^{-1})_{ccv}$ with the explicit form
\begin{align}
\M^{-1}(x;r,\xi,\gamma)=\frac{-\frac{\xi}{r}\gamma+(r^2-1)\frac{\xi}{r}x+r\gamma x+\gamma\sqrt{(\frac{\xi}{r})^2+2\xi x+2\gamma x+x^2}}{2\gamma-(r^2-1)x}.
\label{eq:M inverse complex formula}
\end{align}
we can derive it based on $r,\xi,\gamma$ and substitute back to \eqref{eq:concave inverse}.

Note that the domain of $\mathcal{M}^{-1}$ (or the image of $\mathcal{M}$) is $[0,\frac{\gamma}{r-1}-\frac{\gamma}{r+1})$.

In comparison to the AUTO-V clipping, $\mathcal{M}^{-1}$ takes a much more complicated form, as depicted in the middle plot of \Cref{fig: G plots}, where $r>1$ plays an important role for the gradient norm to converge to zero.
\end{proof}

\subsection{Proof of \Cref{lem:magic inequality2}}
\begin{proof}[Proof of \Cref{lem:magic inequality2}]
We want to lower bound
\begin{align}
\E\left(\frac{\|\g_t\|^2
+\g_t^\top \Delta_{t}}{\|\g_{t}+\Delta_t\|+\gamma}+\frac{\|\g_t\|^2
-\g_t^\top \Delta_{t}}{\|\g_{t}-\Delta_t\|+\gamma}\Big|\Delta_t\in H_+\right)
\label{eq:inside3}
\end{align}

To simplify the notation, we denote noise-to-signal ratio $S:=\frac{\|\Delta_t\|}{\|\g_t\|}$ and $c:=\cos\theta=\frac{\g_t^\top \Delta_t}{\|\g_t\| \|\Delta_t\|}$, with $\theta$ be the random angle between $\g_t$ and $\Delta_t$. Note that $0< c\leq 1$ when $\Delta_t\in H_+$. 

The term inside the conditional expectation in \eqref{eq:inside3} can be written as
\begin{align*}
&\frac{(1+Sc)\|\g_t\|^2}{\sqrt{S^2+2Sc+1}\|\g_t\|+\gamma}+\frac{(1-Sc)\|\g_t\|^2}{\sqrt{S^2-2Sc+1}\|\g_t\|+\gamma}
\\
=&\|\g_t\|\left(\frac{(1+Sc)}{\sqrt{S^2+2Sc+1}+\gamma/\|\g_t\|}+\frac{(1-Sc)}{\sqrt{S^2-2Sc+1}+\gamma/\|\g_t\|}\right)
\end{align*}

Defining $\Gamma=\gamma/\|\g_t\|$ and
\begin{align}
f(c,S;\Gamma):=\frac{(1+Sc)}{\sqrt{S^2+2Sc+1}+\Gamma}+\frac{(1-Sc)}{\sqrt{S^2-2Sc+1}+\Gamma},
\label{eq: f formula}
\end{align}
we turn the conditional expectation in \eqref{eq:inside3} into
\begin{align}
\E\left(\frac{\|\g_t\|^2
+\g_t^\top \Delta_{t}}{\|\g_{t}+\Delta_t\|+\gamma}+\frac{\|\g_t\|^2
-\g_t^\top \Delta_{t}}{\|\g_{t}-\Delta_t\|+\gamma}\Big|\Delta_t\in H_+\right)=\|\g_t\|\E(f(c,S;\Gamma)|\Delta_t\in H_+)
\label{eq: |g|E3}    
\end{align}
for which we want to lower bound $f(c,S;\Gamma)$ over $0<c\leq 1,S>0, \Gamma>0$. We use the next theorem to prepare some helpful properties. The proof can be found in \Cref{app:thm f proof}.

\begin{theorem}
\label{thm: min f}
For $f$ defined in \eqref{eq: f formula}, we have
\begin{enumerate}
    \item $f(c,S;\Gamma)$ is strictly decreasing in $S$ for all $0< c< 1$ and $\Gamma>0$.
    \item Consequently, $\min_{c\in (0,1)} f(c,S;\Gamma)$ is strictly decreasing in $S$.
    \item $f(c,S;\Gamma)$ is strictly decreasing in $c$ for all $S> 1$ and $\Gamma>0$.
    \end{enumerate}
\end{theorem}

We consider a thresholding ratio $r> 0$ and we will focus on the regime that $S<r$. This $r$ will turn out to measure the minimum gradient norm at convergence: informally speaking, $\|\g_t\|$ converges to $\xi/r$.

By the law of total expectation, \eqref{eq: |g|E3} can be relaxed as follows.
\begin{align}
\begin{split}
&\|\g_t\|\E\left( f(c,S;\Gamma)\Big|\Delta\in H_+\right)
\\
=&\|\g_t\|\E\left( f(c,S;\Gamma)\Big|\Delta\in H_+,S<r\right)\P(r\|\g_t\|>\|\Delta\|\Big|\Delta\in H_+)
\\
&+\|\g_t\|\E\left( f(c,S;\Gamma)\Big|\Delta\in H_+,S>r\right)\P(r\|\g_t\|<\|\Delta\|\Big|\Delta\in H_+)
\\
\geq&\|\g_t\|\E\left( f(c,S;\Gamma)\Big|\Delta\in H_+,S<r\right)\P(r\|\g_t\|>\|\Delta\|\Big|\Delta\in H_+)
\\
\geq&\|\g_t\|\E\left( f(c,r;\Gamma)\Big|\Delta\in H_+,S<r\right)\P(r\|\g_t\|>\|\Delta\|\Big|\Delta\in H_+)
\\
=&\|\g_t\|\E\left( f(c,r;\Gamma)\Big|\Delta\in H_+,S<r\right)\P(r\|\g_t\|>\|\Delta\|)
\\
\geq&\min_{c\in(0,1]} f(c,r;\Gamma)\cdot\underbrace{\|\g_t\|\P(r\|\g_t\|>\|\Delta\|)}_{\textcircled{$\star$}}
\end{split}
\label{eq: EP2}
\end{align}
where in the first inequality, the ignoring of last term is justified by $f(c,S;\Gamma)\geq \min_{c\in(0,1]} f(c,S;\Gamma)\geq \min_{c\in(0,1]} f(c,\infty;\Gamma)=0$, from the monotonicity (second statement) in \Cref{thm: min f}.

We first lower bound \textcircled{$\star$} by applying the Markov's inequality:
$$\P(r\|\g_t\|>\|\Delta_t\|)\geq 1-\frac{\E\|\Delta_t\|}{r\|\g_t\|}$$
and hence by \Cref{assumption: tilde g},
$$\|\g_t\|\P(r\|\g_t\|>\|\Delta_t\|)\geq\|\g_t\|-\E\|\Delta\|/r\geq\|\g_t\|-\xi/r.$$

Finally, the conditional expectation of interest in \eqref{eq:inside3} gives
\begin{align*}
&\E\left(\frac{\|\g_t\|^2
+\g_t^\top \Delta_{t}}{\|\g_{t}+\Delta_t\|}+\frac{\|\g_t\|^2
-\g_t^\top \Delta_{t}}{\|\g_{t}-\Delta_t\|}\Big|\Delta_t\in H_+\right)
\geq \min_{0<c\leq 1} f(c,r;\frac{\gamma}{\|\g_t\|})\cdot(\|\g_t\|-\xi/r)
\end{align*}

\end{proof}

\subsection{Proof of \Cref{thm: min f explicit}}

To derive some properties of $\min_c f(c,r;\Gamma)$, we need to compute separately for AUTO-V (without the stability constant, $\Gamma=0$) and for AUTO-S (with the stability constant, $\Gamma>0$), as shown in \Cref{thm: min f explicit}. As we will show, as the number of training iterations $T\to\infty$, DP-SGD with AUTO-V clipping can only compress $\|\g_t\|$ to $\xi/r$ for $r<1$. However, DP-SGD with AUTO-S clipping can compress $\|\g_t\|$ to $\xi/r$ to any $r>1$.

\begin{theorem}
\label{thm: min f explicit}
\quad
\begin{enumerate}
    \item For $0<r<1$ and $\Gamma=0$, we have $\min_{c\in(0,1]} f(c,r;0)>0$. 
    Then \Cref{eq: |g|E3} is lower bounded by
    $$\min_{c\in(0,1]} f(c,r;0)\cdot(\|\g_t\|-\xi/r)$$
    which is increasing in $\|g\|-\xi/r$.
    \item For $r\geq 1$ and $\Gamma=0$, we have $\min_{c\in(0,1]} f(c,r;\Gamma)=f(1,r;0)=0$. In words, \eqref{eq:inside3} has a trivial lower bound and \Cref{thm: convergence DPSGD AUTO} cannot compress $\|\g_t\|$ to $\xi/r$.
    \item For $r\geq 1$ and $\Gamma>0$, we have $\min_{c\in(0,1]} f(c,r;\Gamma)=f(1,r;\Gamma)=\left(\frac{\Gamma}{r+\Gamma-1}-\frac{\Gamma}{r+\Gamma+1}\right)$. Then \Cref{eq: |g|E3} is lower bounded by
    $$\left(\frac{\gamma}{(r-1)\|\g_t\|+\gamma}-\frac{\gamma}{(r+1)\|\g_t\|+\gamma}\right)\cdot(\|\g_t\|-\xi/r)$$
    which is increasing in $\|\g_t\|-\xi/r$.
\end{enumerate}
\end{theorem}

\begin{proof}
To prove statement 1, we use the second statement from \Cref{thm: min f} and show that $\min_c f(c,r;0)>\min_c f(c,\infty;0)=0$. To prove statement 2 and 3, we use the third statement from \Cref{thm: min f} and see that $\min_c f(c,r;\Gamma)=f(1,r;\Gamma)$ with an explicit formula.
\end{proof}

\section{Convergence rate of standard SGD}
\begin{theorem}
\label{thm: convergence nonDP SGD}
Under \Cref{assumption: lower bounding loss}, \ref{assumption: Lipschitz}, \ref{assumption: tilde g} (without the symmetry assumption), running the standard non-DP SGD for $T$ iterations gives, for $\eta\propto 1/\sqrt{T}$,
$$\min_t\E\left(\|\g_t\|\right)
\leq \frac{1}{T^{1/4}}\sqrt{2(\mathcal{L}_0-\mathcal{L}_*)L
+\frac{\xi^2}{B}}$$
\end{theorem}

\begin{proof}[Proof of \Cref{thm: convergence nonDP SGD}]
Consider the standard SGD
$$\w_{t+1}=\w_{t}-\eta \frac{\sum_i \tilde \g_{t,i}}{B}$$
where $\tilde \g_{t,i}$ is i.i.d. unbiased estimate of $\g_t$, with a bounded variance as described in \Cref{assumption: tilde g}.

By Lipschitz smoothness assumption in \Cref{assumption: Lipschitz},
\begin{align*}
\mathcal{L}_{t+1}-\mathcal{L}_t
&\leq \g_t^\top (\w_{t+1}-\w_t)+\frac{L}{2}\|\w_{t+1}-\w_t\|^2
=-\eta \g_t^\top\left(\sum_i\frac{1}{B}\tilde \g_{t,i}\right)+\frac{L\eta^2}{2}\left\|\sum_i \frac{1}{B}\tilde \g_{t,i}\right\|^2
\end{align*}

The expected improvement at one iteration is
\begin{align}
\begin{split}
\E (\mathcal{L}_{t+1}-\mathcal{L}_t|\w_t)
&\leq -\eta \g_t^\top\E\tilde \g_{t,i}+\frac{L\eta^2}{2}\E \|\sum_i\frac{1}{B}\tilde \g_{t,i}\|^2
\\
&\leq -\eta   \|\g_t\|^2+\frac{L\eta^2}{2}\left(\|\g_t\|^2+\frac{\xi^2}{B}\right)
\end{split}
\label{eq:expected improvement for nonDP}
\end{align}

Now we extend the expectation over randomness in the trajectory, and perform a telescoping sum over the iterations
\begin{align*}
\mathcal{L}_0-\mathcal{L}_*
&\geq \mathcal{L}_0-\E \mathcal{L}_T=\sum_t \E(\mathcal{L}_t-\mathcal{L}_{t+1})
\geq \left(\eta-\frac{L\eta^2  }{2}\right)\E(\sum_t\|\g_t\|^2)-\frac{TL\eta^2 \xi^2}{2B}
\end{align*}
Notice that we do not need the symmetry assumption in \Cref{assumption: tilde g} in the non-DP SGD analysis.

We apply the same learning rate as in \cite{bernstein2018signsgd}, $\eta=\frac{1}{L\sqrt{T}}$,
\begin{align*}
2(\mathcal{L}_0-\mathcal{L}_*)
\geq \left(\frac{2}{L\sqrt{T}}-\frac{1}{LT}\right)\E\left(\sum_t\|\g_t\|^2\right)-\frac{T \xi^2}{BLT}
\geq \frac{\sqrt{T}}{L}\E\left(\frac{1}{T}\sum_t\|\g_t\|^2\right)-\frac{\xi^2}{BL}
\end{align*}
and finally
\begin{align*}
\min_t\E\left(\|\g_t\|^2\right)
\leq \E\left(\frac{1}{T}\sum_t\|\g_t\|^2\right)
\leq \frac{1}{\sqrt{T}}\left[2(\mathcal{L}_0-\mathcal{L}_*)L
+\frac{\xi^2}{B}\right]
\end{align*}
Using the Jensen's inequality, we can have
\begin{align*}
\min_t\E\left(\|\g_t\|\right)
\leq \frac{1}{T^{1/4}}\sqrt{2(\mathcal{L}_0-\mathcal{L}_*)L
+\frac{\xi^2}{B}}
\end{align*}
\end{proof}

\section{Auxiliary proofs}

\subsection{Proof of \Cref{thm: min f}}
\label{app:thm f proof}
\begin{proof}
We first show $\frac{df(c,S;\Gamma)}{dS}<0$ for all $0<c<1, \Gamma>0$ and $S>0$, as visualized in the left plot of \Cref{fig:f over S}. We can explicitly write down the derivative, by WolframAlpha
\begin{align}
\frac{df(c,S;\Gamma)}{dS}=\frac{-(A\Gamma^2+B\Gamma+C)}{\sqrt{S^2-2 c S + 1}\sqrt{S^2+2 c S + 1}(\Gamma+\sqrt{S^2-2 c S + 1})^2(\Gamma+\sqrt{S^2+2 c S + 1})^2}
\label{eq:complicated df dS}
\end{align}
with
\begin{align*}
A(c,S)&=\sqrt{S^2+2cS+1}\left(3c^2 S-2c(S^2+1)+S\right)+\sqrt{S^2-2cS+1}\left(3c^2 S+2c(S^2+1)+S\right)
\\
B(c,S)&=4S\left[(S^2+1)(1-c^2)+c^2\sqrt{S^2+2cS+1}\sqrt{S^2-2cS+1}\right]
\\
C(c,S)&=(1-c^2)S\left[(S^2-2cS+1)^{3/2}+(S^2+2cS+1)^{3/2}\right]
\end{align*}

It is obvious that, since $c<1$,
\begin{align}
S^2\pm 2 c S + 1>S^2\pm 2 c S + c^2=(S\pm c)^2\geq 0.
\label{eq:obvious}    
\end{align}
From \eqref{eq:obvious}, the denominator in \eqref{eq:complicated df dS} is positive and it suffices to show $A\Gamma^2+B\Gamma+C>0$ for all $0<c<1$ and $S>0$, in order to show $\frac{df}{dS}<0$.

Also from \eqref{eq:obvious}, we can easily see $B(c,S)>0$ and $C(c,S)>0$. We will show that $A(c,S)>0$ in \Cref{lem: A ugly}, after very heavy algebraic computation. 

Now we can claim that $A\Gamma^2+B\Gamma+C>0$ by \Cref{fact: Ax^2+Bx+C}, and complete the proof of the first statement.

To further see that $\min_c f(c,S;\Gamma)$ is decreasing in $S$, let us denote $c^*(x;\Gamma):=\text{arg min}_{c\in[0,1]} f(c,x;\Gamma)$. Then considering $S< S'$, we prove the second statement by observing
$$\min_c f(c,S;\Gamma)=f(c^*(S;\Gamma),S;\Gamma)> f(c^*(S;\Gamma),S';\Gamma)\geq \min_c f(c,S';\Gamma).$$
This statement is also visualized in the right plot of \Cref{fig:f over S}.

We next show $\frac{df(c,S;\Gamma)}{dc}<0$ for all $0<c<1, \Gamma>0$ and $S>1$. We can explicitly write down the derivative, by WolframAlpha
\begin{align}
\frac{df(c,S;\Gamma)}{dc}=\frac{-S(A'\Gamma^2+B'\Gamma+C')}{\sqrt{S^2-2 c S + 1}\sqrt{S^2+2 c S + 1}(\Gamma+\sqrt{S^2-2 c S + 1})^2(\Gamma+\sqrt{S^2+2 c S + 1})^2}
\label{eq:complicated df dc}
\end{align}
with
\begin{align*}
A'(c,S)&=\left[(S^2+3cS+2)\sqrt{S^2-2cS+1}-(S^2-3cS+2)\sqrt{S^2+2cS+1}\right]
\\
B'(c,S)&=4Sc\left[\sqrt{S^2+2cS+1}\sqrt{S^2-2cS+1}+(S^2-1)\right]
\\
C'(c,S)&=S\left[(c+S)(S^2-2cS+1)^{3/2}+(c-S)(S^2+2cS+1)^{3/2}\right]
\end{align*}
Clearly $B'(c,S)>0$ and $C'(c,S)>0$, since $S^2+2cS+1>S^2-2cS+c^2=(S-c)^2\geq 0$. And we will show $A'(c,S)>0$ in \Cref{lem:A' ugly}, after some algebra. 

We again claim that $A'\Gamma^2+B'\Gamma+C'>0$ by \Cref{fact: Ax^2+Bx+C}, which guarantees that the numerator in \eqref{eq:complicated df dc} is negative and that $\frac{df}{dc}<0$. This is visualized in \Cref{fig:f over c}.
\end{proof}

\begin{figure}[!htb]
    \centering
    \includegraphics[width=0.4\linewidth]{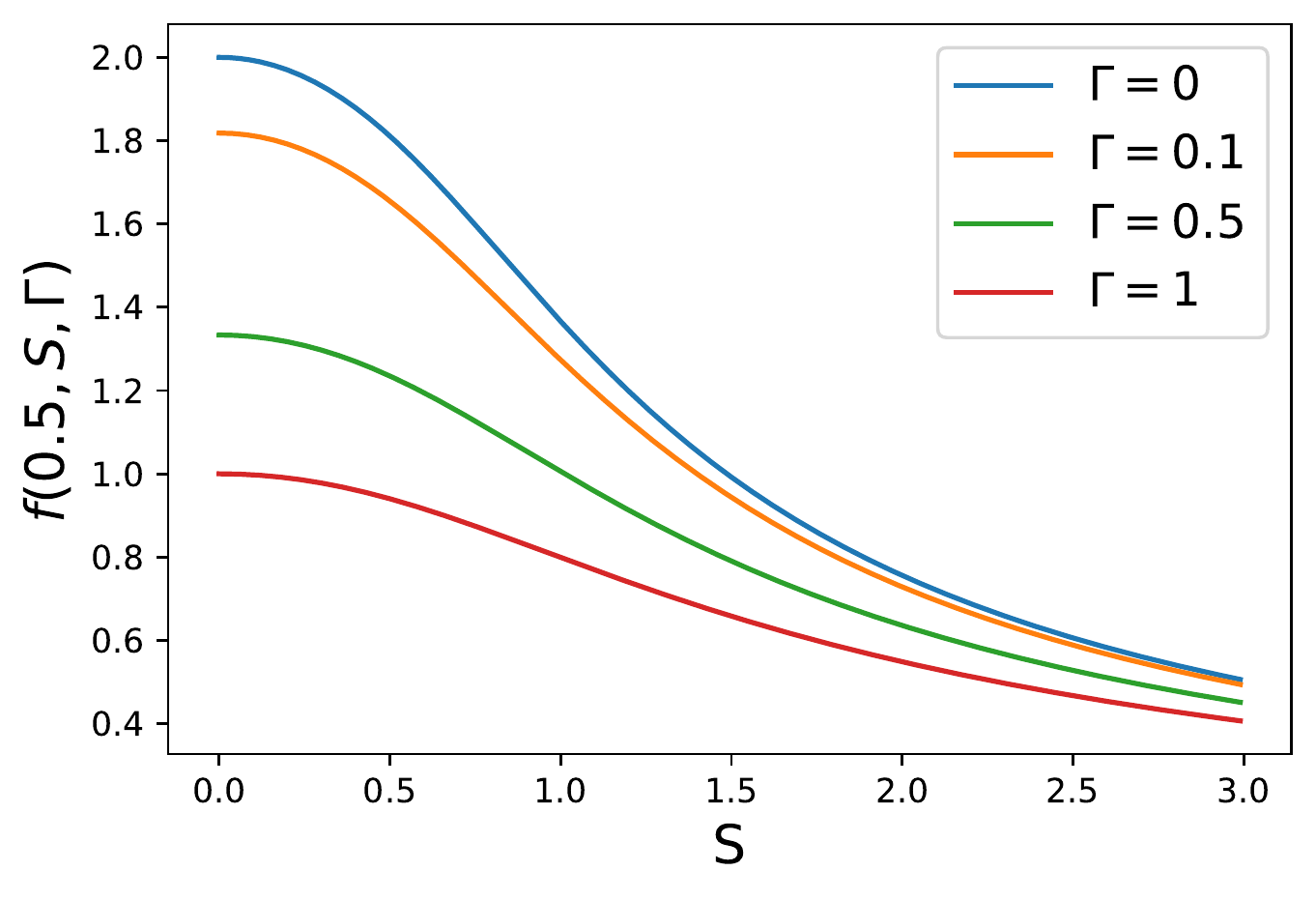}
    \includegraphics[width=0.41\linewidth]{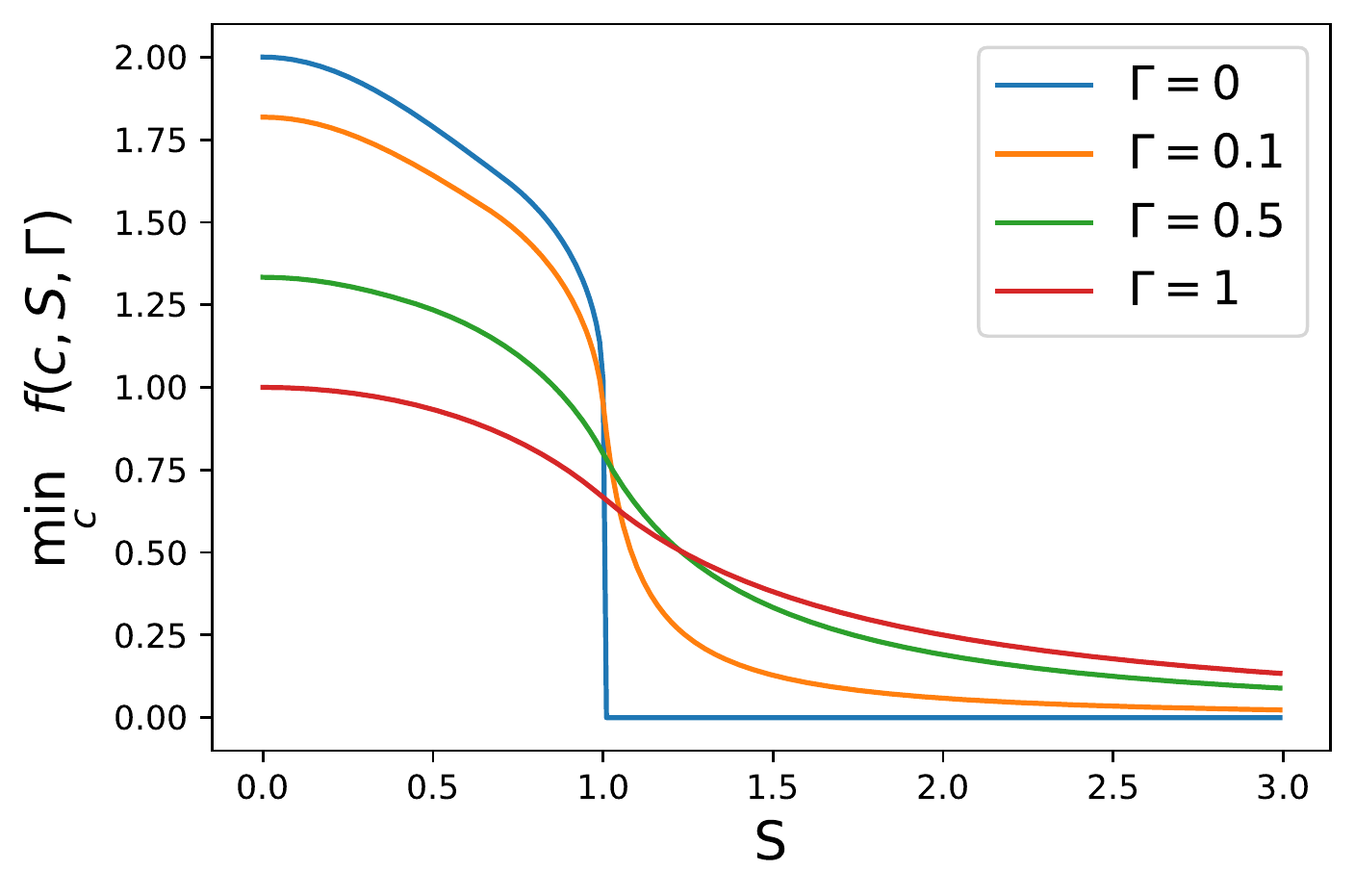}
    \caption{Visualization of $f(0.5,S,\Gamma)$ (left) and $\min_{0\leq c\leq 1}f(c,S,\Gamma)$ over $S>0$.}
    \label{fig:f over S}
\end{figure}

\begin{figure}[!htb]
    \centering
    \includegraphics[width=0.4\linewidth]{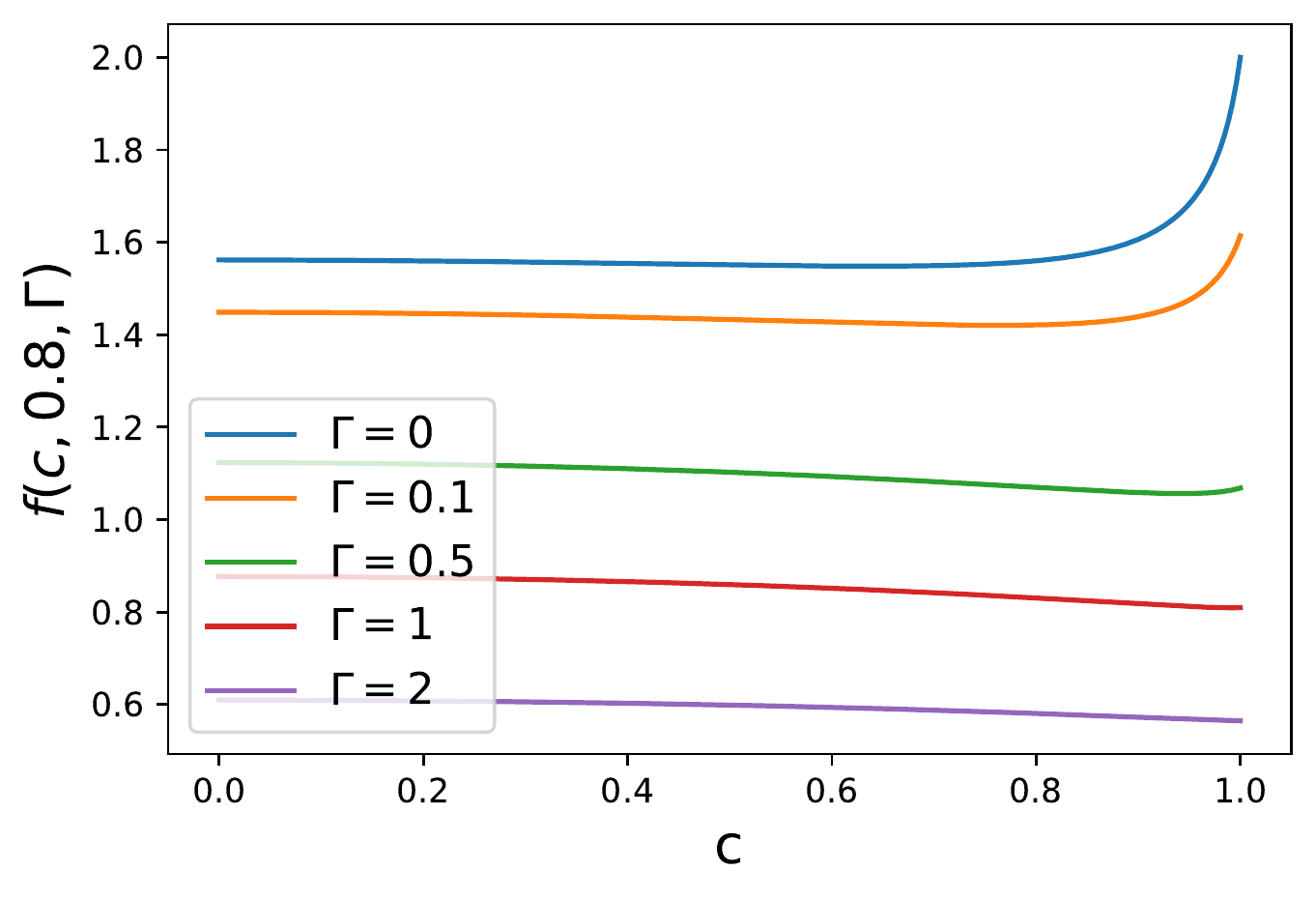}
    \includegraphics[width=0.4\linewidth]{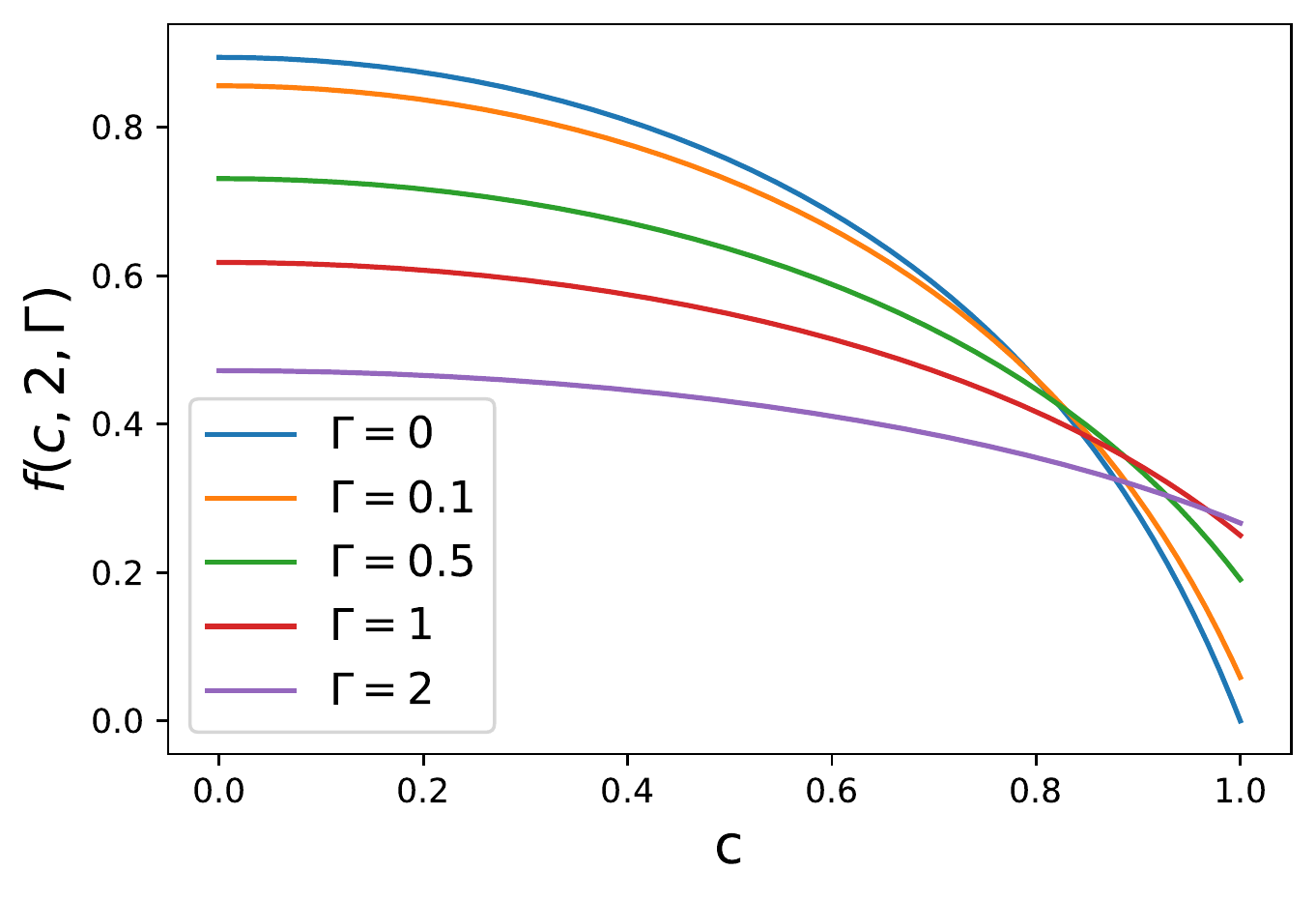}
    \caption{Visualization of $f(c,0.8,\Gamma)$ (left) and $f(c,2,\Gamma)$ over $0\leq c\leq 1$.}
    \label{fig:f over c}
\end{figure}

\subsection{Proof of \Cref{lem: A ugly}}
\begin{lemma}
\label{lem: A ugly}
For all $0<c<1$ and $S>0$,
$$
A:=\sqrt{S^2+2cS+1}\left(3c^2 S-2c(S^2+1)+S\right)+\sqrt{S^2-2cS+1}\left(3c^2 S+2c(S^2+1)+S\right)>0.
$$
\end{lemma}
\begin{proof}
We prove by contradiction. Suppose
$$\sqrt{S^2+2cS+1}\left(3c^2 S-2c(S^2+1)+S\right)+\sqrt{S^2-2cS+1}\left(3c^2 S+2c(S^2+1)+S\right)<0.$$
Then
$$0<\sqrt{S^2-2cS+1}\left(3c^2 S+2c(S^2+1)+S\right)<-\sqrt{S^2+2cS+1}\left(3c^2 S-2c(S^2+1)+S\right).$$
where the first inequality comes from $S^2-2cS+1>S^2-2cS+c^2=(S-c)^2\geq 0$.

Squaring everything gives
$$(S^2-2cS+1)\left(3c^2 S+2c(S^2+1)+S\right)^2<(S^2+2cS+1)\left(3c^2 S-2c(S^2+1)+S\right)^2.$$
Taking the difference gives
$$4 c S (2 + 3 S^2 - 9 c^4 S^2 + 2 S^4 + 2 c^2 (1 - S^2 + S^4))<0$$
Given that $c>0, S>0$, we have
$$2 + 3 S^2 - 9 c^4 S^2 + 2 S^4 + 2 c^2 (1 - S^2 + S^4)<0$$
Denoting $X:=S^2$ and viewing the above as a quadratic polynomial of $X$, we have
$$\underbrace{(2c^2+2)X^2+(3-2c^2-9c^4)X+(2c^2+2)}_{\textcircled{1}} <0
$$
Using the closed-form minimizer of quadratic polynomial $\textcircled{1}$, after some heavy algebra, one can check the minimum of $\textcircled{1}$ is
$$\frac{(1 + 3 c^2)^2 (1-c^2)(7+9c^2)}{8 (1 + c^2)}$$
which is clearly positive. Contradiction!
\end{proof}

\subsection{Proof of \Cref{lem:A' ugly}}
\begin{lemma}
\label{lem:A' ugly}
For all $0<c<1$ and $S>1$,
$$
(S^2+3cS+2)\sqrt{S^2-2cS+1}-(S^2-3cS+2)\sqrt{S^2+2cS+1}>0.
$$
\end{lemma}
\begin{proof}
Notice that $(S^2+3cS+2)>S^2+2>0$ and $\sqrt{S^2\pm 2cS+1}>0$. Therefore if $S^2-3cS+2\leq 0$, we are done.

Otherwise, we prove by contradiction and suppose
$$0<(S^2+3cS+2)\sqrt{S^2-2cS+1}<(S^2-3cS+2)\sqrt{S^2+2cS+1}.$$
under the condition that $S^2-3cS+2>0$.

Squaring everything gives
$$(S^2+3cS+2)^2 (S^2-2cS+1)<(S^2-3cS+2)^2 (S^2+2cS+1).$$
Taking the difference gives
$$cS(8 + 20 S^2 - 36 c^2 S^2 + 8 S^4)<0$$
Given that $c>0, S>0$, we have
$$2 + 5 S^2 - 9 c^2 S^2 + 2 S^4<0$$
Denoting $X:=S^2$ and viewing the above as a quadratic polynomial of $X$, we have, for $X>1$,
$$\underbrace{2X^2+(5-9c^2)X+2}_{\textcircled{2}} <0
$$

The closed-form minimizer of quadratic polynomial $\textcircled{2}$ is $\frac{(9c^2-5)}{4}$. Given that $0<c<1$, we must have $-\frac{5}{4}<\frac{9c^2-5}{4}<1$. Hence the minimizer is not within the feasible domain $(1,\infty)$ of $X$. Thus the minimum of \textcircled{2} is achieved with $X=1$ at $9(1-c^2)$. This is positive. Contradiction!
\end{proof}

\subsection{Proof of \Cref{fact: Ax^2+Bx+C}}
\begin{fact}
\label{fact: Ax^2+Bx+C}
For a quadratic polynomial $Ax^2+Bx+C$ with $A,B,C>0$, the minimum value on the domain $x\geq 0$ is $C$, at $x=0$. Therefore $Ax^2+Bx+C>0$.
\end{fact}
\begin{proof}
Since $A>0$, the quadratic polynomial is convex and increasing on the domain $x>-\frac{B}{2A}$. Since $B>0$ as well, we know $-\frac{B}{2A}<0$ and hence the quadratic polynomial is strictly increasing on $x>0$. Therefore the minimum value is achieved when $x=0$, and we obtain $Ax^2+Bx+C\geq C>0$ for all $x\geq 0$.
\end{proof}

\subsection{Assumption of symmetric gradient noise}
\label{app:grad symm}

We show that \Cref{assumption: tilde g} is actually relaxed from and less strict than the assumptions used in the non-DP literature. In words, \Cref{assumption: tilde g} allows our DP convergence to be comparable to the standard convergence (as in \Cref{thm: convergence nonDP SGD}), because our assumption does not enforce extra constraint.

In standard non-DP analysis \cite{mandt2017stochastic,smith2018don,chaudhari2018stochastic,xie2020diffusion}, the mini-batch gradient is assumed to be an unbiased estimate of the oracle gradient $\g_t=\frac{\partial \mathcal{L}}{\partial \w}$:
$$\frac{1}{B}\sum_{i=1}^B\tilde \g_{t,i}\sim \g_t+\bm\xi(\w)$$
and $\bm\xi$ is the random gradient noise with $\bm\xi\sim N(\bm 0,\Sigma(\w)/B)$. Since this assumption holds for any batch size $B$, we can set $B=1$ to recover the per-sample gradient noise: $\bm\xi=\tilde \g_{t,i}-\g_t$ is i.i.d. and symmetric because a zero-mean Gaussian is symmetric.

In fact, we can further relax our \Cref{assumption: tilde g}: besides assuming the central symmetry, the same proof of convergence will follow if we instead assume the mirror symmetry about the hyperplane normal to $\g_t$, that is $\{\v: \g_t^\top \v=0\}$.

\section{Examples of lazy regions}
\label{app:lazy region examples}
\subsection{Balanced binary classification}
We describe the data generation in \Cref{subsec:lazy region and stability}. The label is uniformly $\pm 1$, that is $\P(y_i=+1)=\P(y_i=-1)=0.5$. We have 10000 positive and negative samples $x_i\sim\mathcal{N}(y_i,1)$. We consider a logistic regression model  $\P(Y=y|x)=\mathbb{I}(y=1)\cdot\text{Sigmoid}(x+\theta)+\mathbb{I}(y=-1)\cdot(1-\text{Sigmoid}(x+\theta))=\frac {1}{1+e^{-y(\theta+x)}}$, where $\theta\in\R$ is the intercept. The gradient with respect to this only trainable parameter is $\frac{\partial \mathcal{L}_i}{\partial \theta}=-y\left(1-\frac {1}{1+e^{-y(\theta+x)}}\right)$. We set the clipping threshold $R=0.01$ and the stability constant $\gamma=0.01$.

\subsection{Mean estimation on Gaussian mixture data}
We also observe the lazy region issue in the mean estimation problem $\min_\theta \frac{1}{2}\|\theta-x_i\|^2$. Here $\P(x_i\sim\mathcal{N}(4,1))=\P(x_i\sim\mathcal{N}(4,1))=0.5$. We have 10000 samples from each Gaussian distribution. The regular minimum is clearly $\sum_i x_i\to 0$, where the regular gradient and AUTO-S clipped gradient vanish. Yet both AUTO-V and Abadi's clipping lose motivation to update the mean estimator on the interval $(-1,1)$. We set the clipping threshold $R=0.01$ and the stability constant $\gamma=0.1$.

\begin{figure}[!htb]
    \centering
    \includegraphics[width=0.5\linewidth]{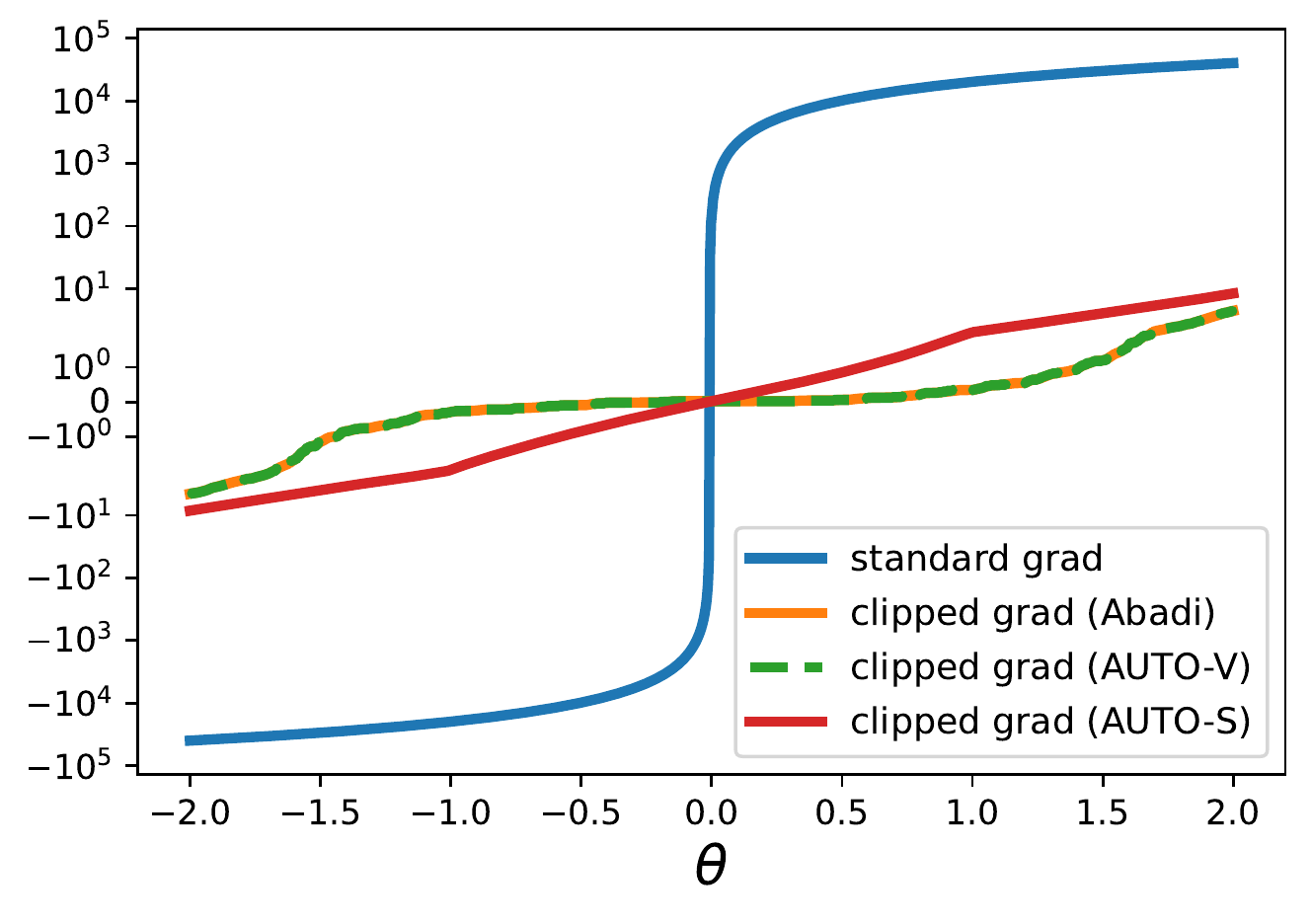}
    \vspace{-0.2cm}
    \caption{Scalar gradient $\frac{\partial \mathcal{L}}{\partial\theta}$ at each $\theta$.}
    \label{fig:lazy zone mean est}
\end{figure}

\section{Experiments settings}
\label{app:experiemnt settings}
\subsection{Image classification settings}
\label{app:CV settings}
We give the experiments settings for computer vision tasks in \Cref{tab:CV results}.
\begin{itemize}
    \item \textbf{MNIST}: We use the network architecture from \cite{papernot2020tempered,tramer2020differentially,shamsabadi2021losing}, with 40 epochs, 512 batch size, 0.5 learning rate (or 0.005 non-DP learning rate), 0.1 clipping threshold, DP-SGD with 0.9 momentum, and without pretraining. This setting is the same as \cite{tramer2020differentially}.
    
    \item \textbf{FashionMNIST}: We use the same network architecture as MNIST, with 40 epochs, 2048 batch size, 4 learning rate (or 0.04 non-DP learning rate), DP-SGD with 0.9 momentum, and without pretraining. This setting is the same as \cite{tramer2020differentially}.

    
    \item \textbf{CIFAR10 pretrained}: We use the SimCLR model from \cite{chen2020simple}\footnote{See implementation in \url{https://github.com/google-research/simclr.}}, with 50 epochs, 1024 batch size, 4 learning rate (or 0.04 non-DP learning rate), 0.1 clipping threshold, and DP-SGD with 0.9 momentum. The SimCLR model is pretrained on unlabelled ImageNet dataset. After pretraining, we obtain a feature of dimension 4096 on which a linear classifier is trained privately. This setting is the same as \cite{tramer2020differentially}.
    
    \item \textbf{ImageNette}: We use the ResNet9 (2.5 million parameters) with Mish activation function \cite{misra2019mish}. We set 50 epochs, 1000 batch size, 0.0005 learning rate (or 0.000005 non-DP learning rate), 1.5 clipping threshold, and use DP-NAdam, without pretraining. This setting is the same as \cite{klause2022differentially} except we did not apply the learning rate decaying scheduler. 
    \item \textbf{CelebA (Smiling and Male and Multi-label)} We use the same ResNet9 as above, with 10 epochs, 500 batch size, 0.001 DP learning rate (or 0.00001 non-DP learning rate), 0.1 clipping threshold, and use DP-Adam, without pretraining. We use the labels `Smiling' and `Male' for two binary classification tasks, with cross-entropy loss. For the multi-label task uses a scalar loss by summing up the 40 binary cross-entropy losses from each label.
    \end{itemize}

We refer the code for MNIST, FashionMNIST, CIFAR10, CIFAR10 pretrained to \url{https://github.com/ftramer/Handcrafted-DP} by \cite{tramer2020differentially}. ResNet9 can be found in \url{https://github.com/cbenitez81/Resnet9}.

Throughout all experiments, we do not apply tricks such as random data augmentation (single or multiple times \cite{de2022unlocking}), weight standardization \cite{qiao2019micro}, or parameter averaging \cite{polyak1992acceleration}.


\subsection{Sentence classification settings}
We experiment on five datasets in \Cref{tab:sentence roberta base} and \Cref{tab:sentence roberta large}.
\begin{itemize}
    \item \textbf{MNLI(m)} MNLI-matched, the matched validation and test splits from Multi-Genre Natural Language Inference Corpus.
    \item \textbf{MNLI(mm)} MNLI-mismatched, the matched validation and test splits from Multi-Genre Natural Language Inference Corpus.
    \item \textbf{QQP} The Quora Question Pairs2 dataset.
    \item \textbf{QNLI} The Stanford Question Answering dataset.
    \item \textbf{SST2} The Stanford Sentiment Treebank dataset.
\end{itemize}
The datasets are processed and loaded from Huggingface \cite{lhoest-etal-2021-datasets}, as described in
\url{https://huggingface.co/datasets/glue}. We follow the same setup as \cite{yu2021large} and \cite{li2021large}. We refer the interested readers to Appendix G,H,I,K,N of \cite{li2021large} for more details.

We emphasize that our automatic clipping uses exactly the same hyperparameters as the Abadi's clipping in \cite{li2021large}, which is released in their Private-Transformers library \footnote{See \url{https://github.com/lxuechen/private-transformers/blob/main/examples/classification/run_wrapper.py}}.
\begin{table}[!htb]
    \centering
    \begin{tabular}{c|cccc}
     Dataset& MNLI(m/mm)&QQP&QNLI&SST2  \\
 Epoch&18&18&6&3\\
 Batch size&6000&6000&2000&1000\\
 clipping threshold $R$&0.1&0.1&0.1&0.1\\
 DP learning rate &5e-4&5e-4&5e-4&5e-4\\
 non-DP learning rate &5e-5&5e-5&5e-5&5e-5\\
 learning rate decay&Yes&Yes&Yes&Yes\\
 AdamW weight decay&0&0&0&0\\
 Max sequence length&256&256&256&256\\
    \end{tabular}
    \caption{Hyperparameters of automatic clipping and Abadi's clipping, for sentence classification in \Cref{tab:sentence roberta base} and \Cref{tab:sentence roberta large}, using either RoBERTa base or large.}
\end{table}

Notice that we use DP learning rate 5e-4 across tasks for the $R$-dependent automatic DP-Adam, which is equivalent to $R$-independent automatic DP-Adam with the same learning rate. We demonstrate that the results are not sensitive to learning rates around the optimal choice. That is, the automatic clipping does not eliminate $R$ at the cost of more difficult tuning of learning rate.
\begin{table}[!htb]
    \centering
    \begin{tabular}{c|c|c|c|c|c}
 learning rate& 1e-4& 3e-4& 5e-4& 8e-4& 1e-3 \\\hline
 RoBERTa-base& 93.92& 94.38& 94.49& 94.72& 93.35\\
RoBERTa-large&95.76& 96.21& 96.21& 96.33& 95.99
    \end{tabular}
    \caption{SST2 accuracy with respect to learning rate.}
\end{table}

\subsection{Table-to-text generation settings}
We experiment multiple GPT2 models on E2E dataset from Huggingface \cite{lhoest-etal-2021-datasets} in \Cref{tab:E2E GPT selected}. We follow the same setup as \cite{li2021large}, and our automatic clipping uses exactly the same hyperparameters as the Abadi's clipping in \cite{li2021large}, which is released in their Private-Transformer library \footnote{See \url{https://github.com/lxuechen/private-transformers/blob/main/examples/table2text/run.sh}}.

\begin{table}[!htb]
    \centering
    \begin{tabular}{c|ccc}
     Model& GPT2&GPT2 medium&GPT2 large  \\
 Epoch&10&10&10\\
 Batch size&1024&1024&1024\\
 clipping threshold $R$&0.1&0.1&0.1\\
 DP learning rate &2e-3&2e-3&2e-3\\
 non-DP learning rate &2e-4&1e-4&1e-4\\
 learning rate decay&No&No&No\\
 AdamW weight decay&0.01&0.01&0.01\\
 Max sequence length&100&100&100\\
    \end{tabular}
    \caption{Hyperparameters of automatic clipping and Abadi's clipping, for the E2E generation task in \Cref{tab:E2E GPT selected}.}
\end{table}

\section{Figure zoo}
\subsection{Frequency of clipping}
\label{app:freq of clipping}
We show that in all sentence classification tasks, Abadi's clipping happens on a large proportion of per-sample gradients. This supports the similarity between Abadi's clipping and AUTO-V in \eqref{eq:Auto-V}.

\begin{figure}[!htb]
    \centering
    \includegraphics[width=0.45\linewidth]{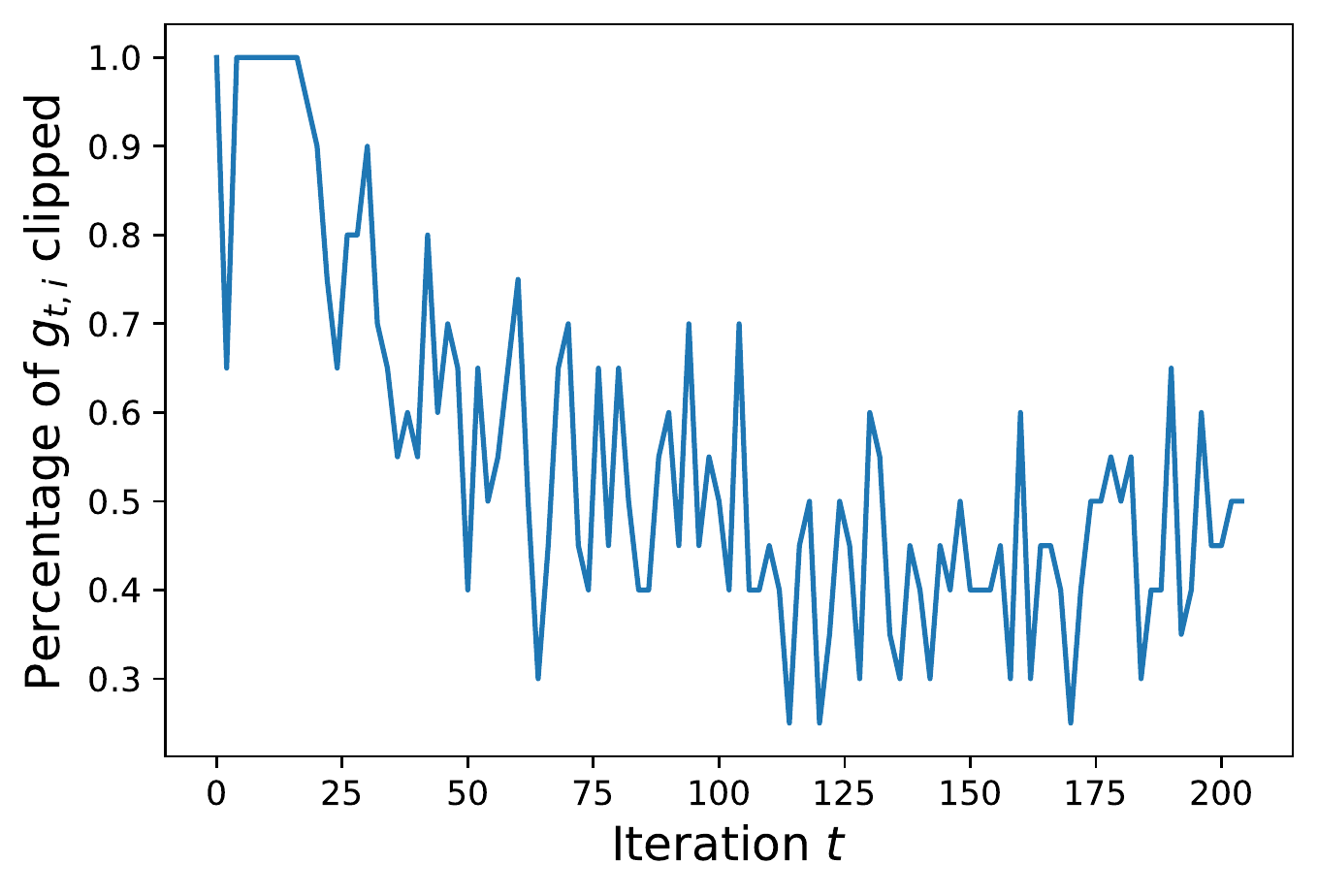}
    \includegraphics[width=0.45\linewidth]{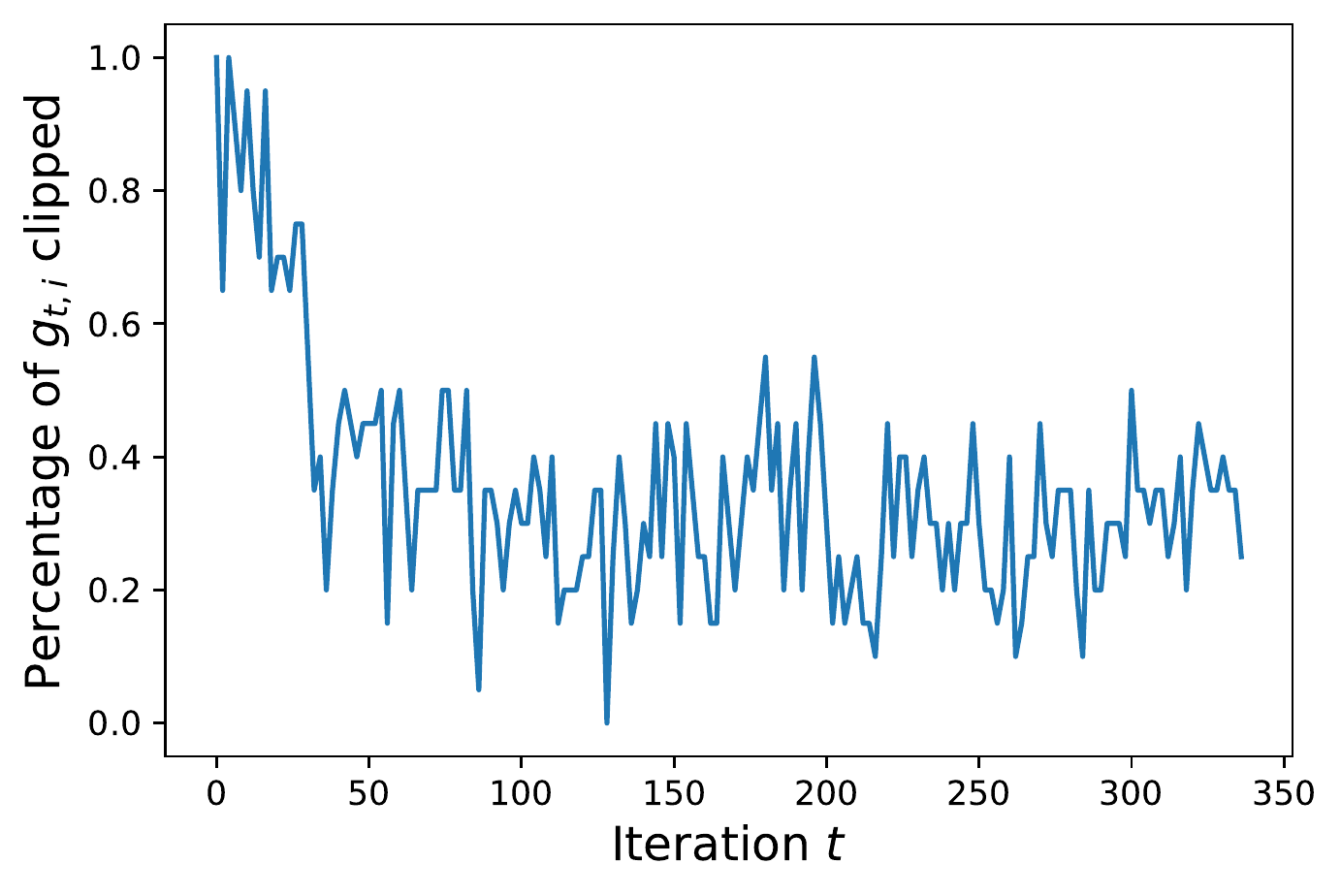}  \includegraphics[width=0.45\linewidth]{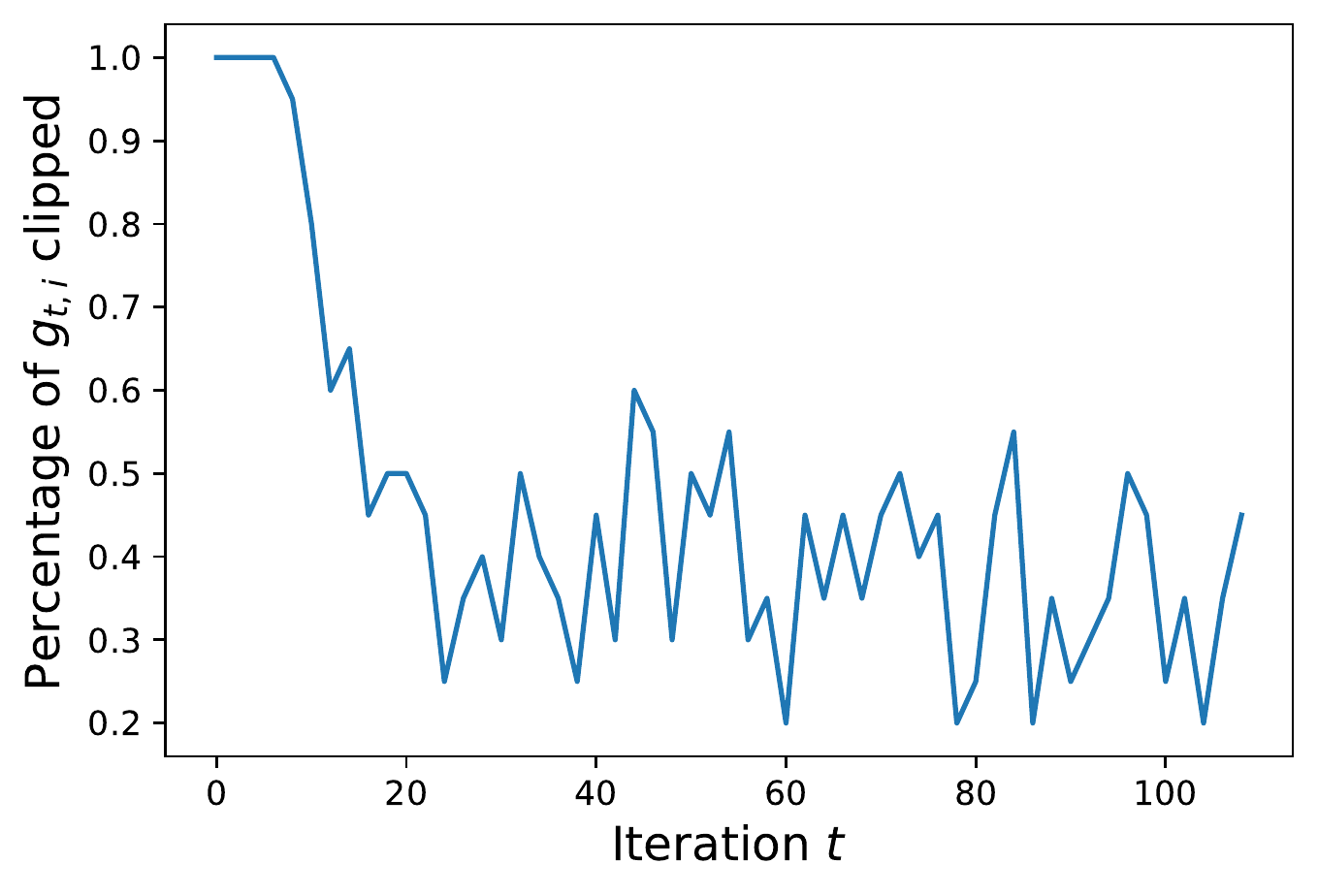}
    \includegraphics[width=0.45\linewidth]{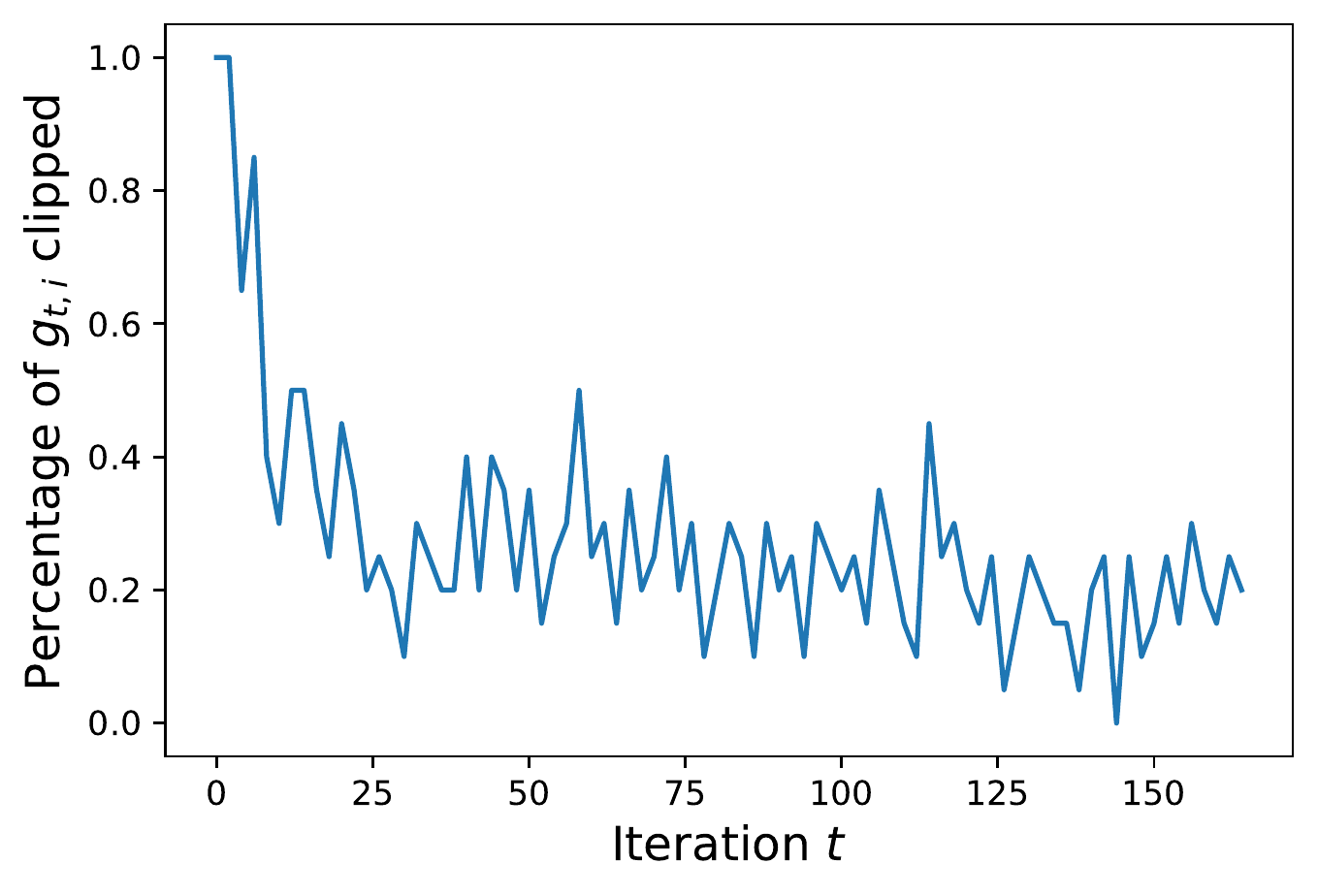}  \includegraphics[width=0.45\linewidth]{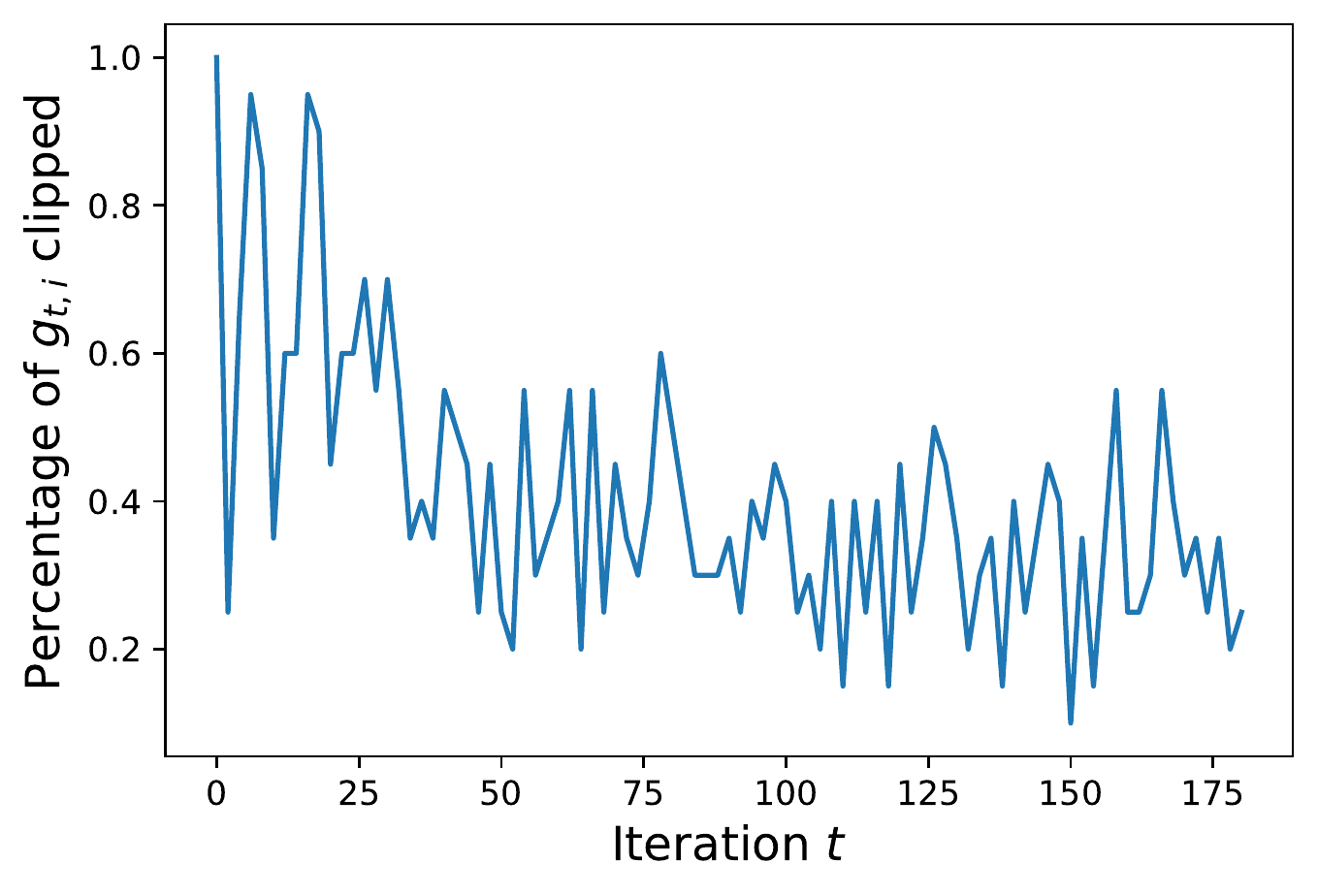}
    \includegraphics[width=0.45\linewidth]{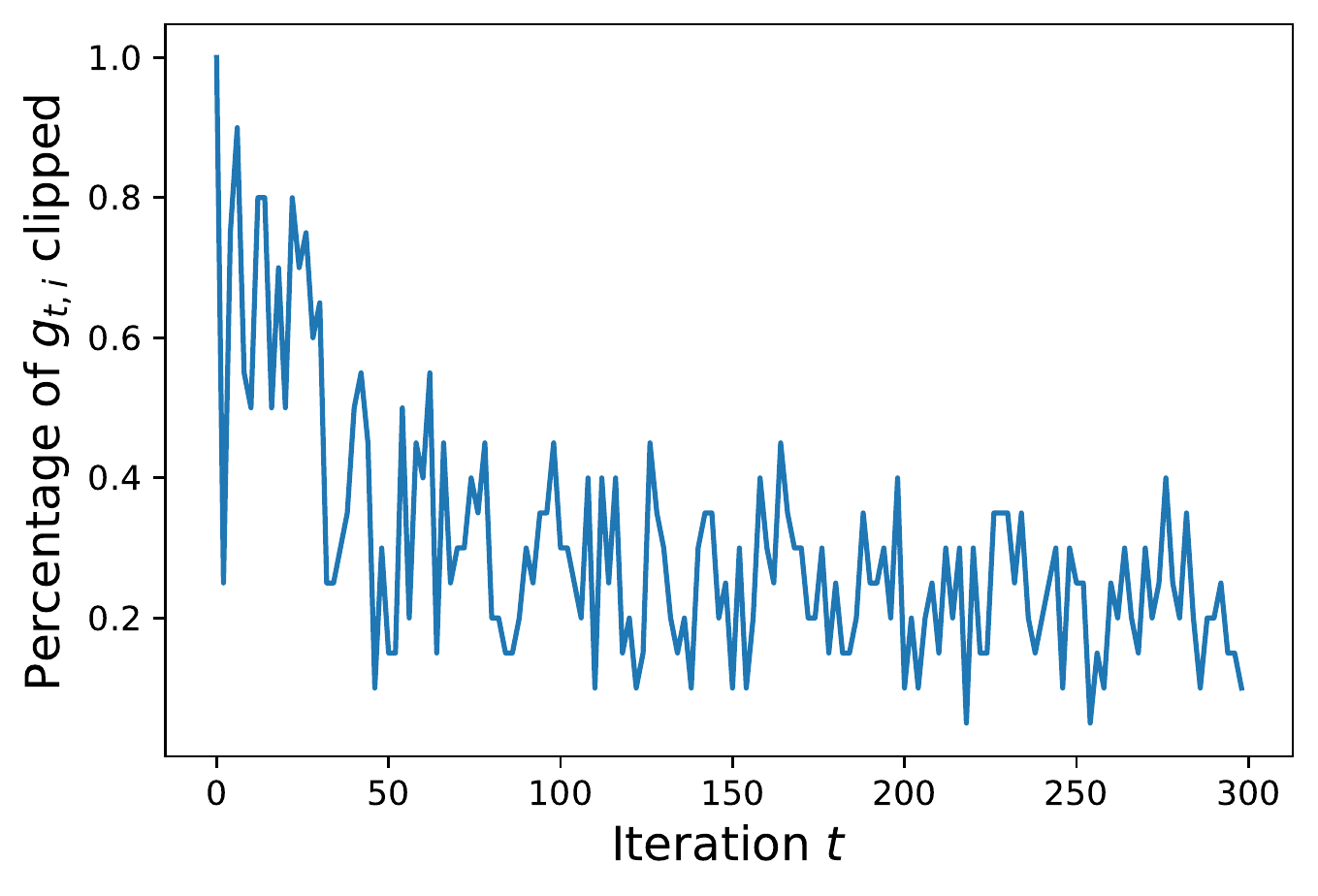}  \vspace{-0.4cm}
\caption{Percentage of clipped per-sample gradients when training with DP-Adam$_\text{Abadi}$ ($\epsilon=3$), as in \Cref{sec:NLP classification}. Left panel is RoBERTa-base and right panel is RoBERTa-large. Top row: MNLI. Middle row: QNLI. Bottom row: QQP.}
\label{fig:percentage of clipped}
\end{figure}

We note that for GPT2, GPT2 medium and GPT2 large, empirically in all iterations 100\% of the per-sample gradients are clipped by the Abadi's clipping, making the performance of Abadi's clipping equivalent to AUTO-V clipping, as shown in \Cref{tab:E2E GPT selected}.

\subsection{Stability constant helps AUTO clipping reduce gradient norm}
To corroborate our claim in \Cref{thm: convergence DPSGD AUTO}, that the stability $\gamma$ reduces the gradient norm, we plot the actual gradient norm by iteration.
\begin{figure}[!htb]
    \centering
    \includegraphics[width=0.35\linewidth]{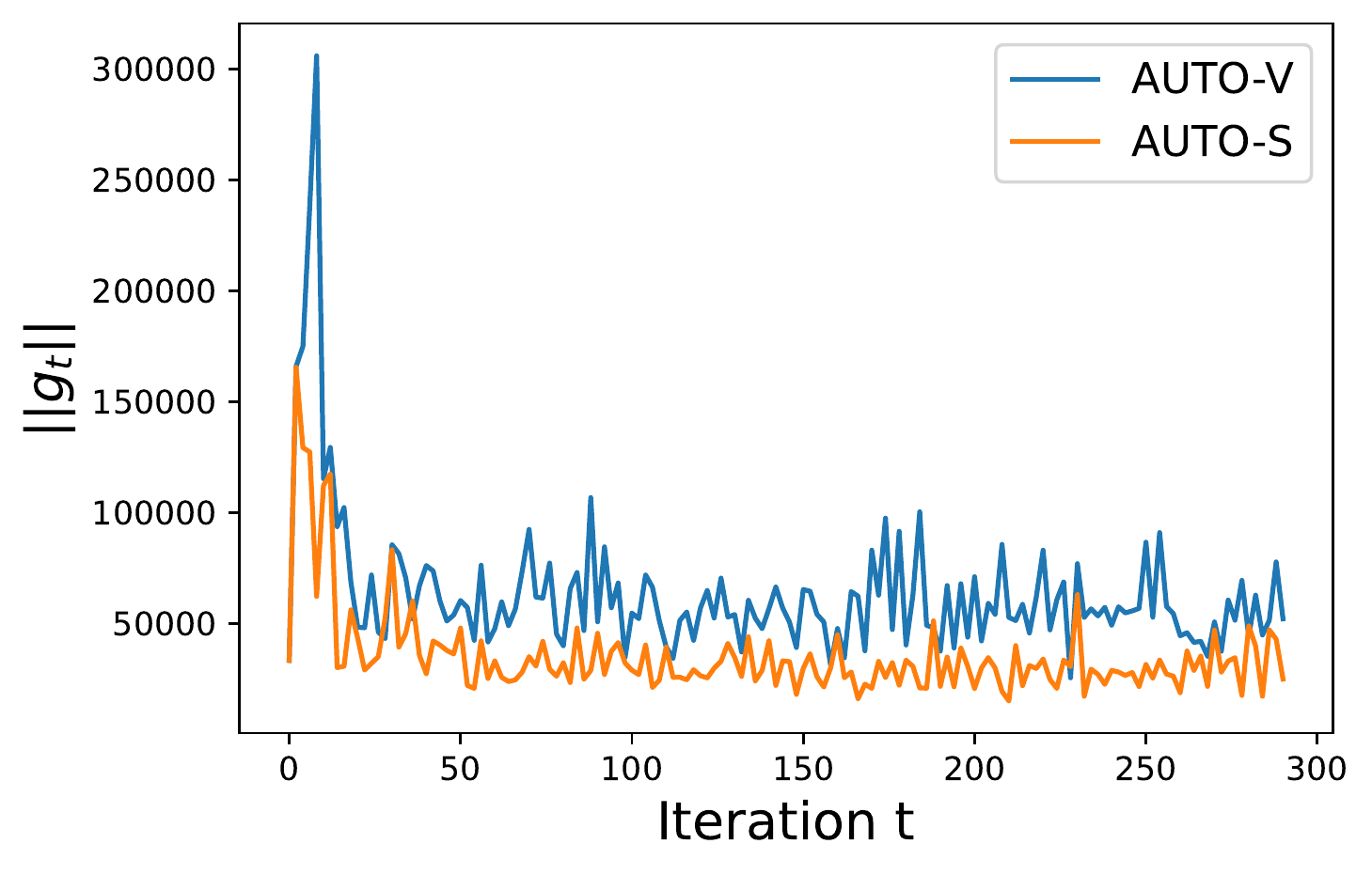}
    \includegraphics[width=0.35\linewidth]{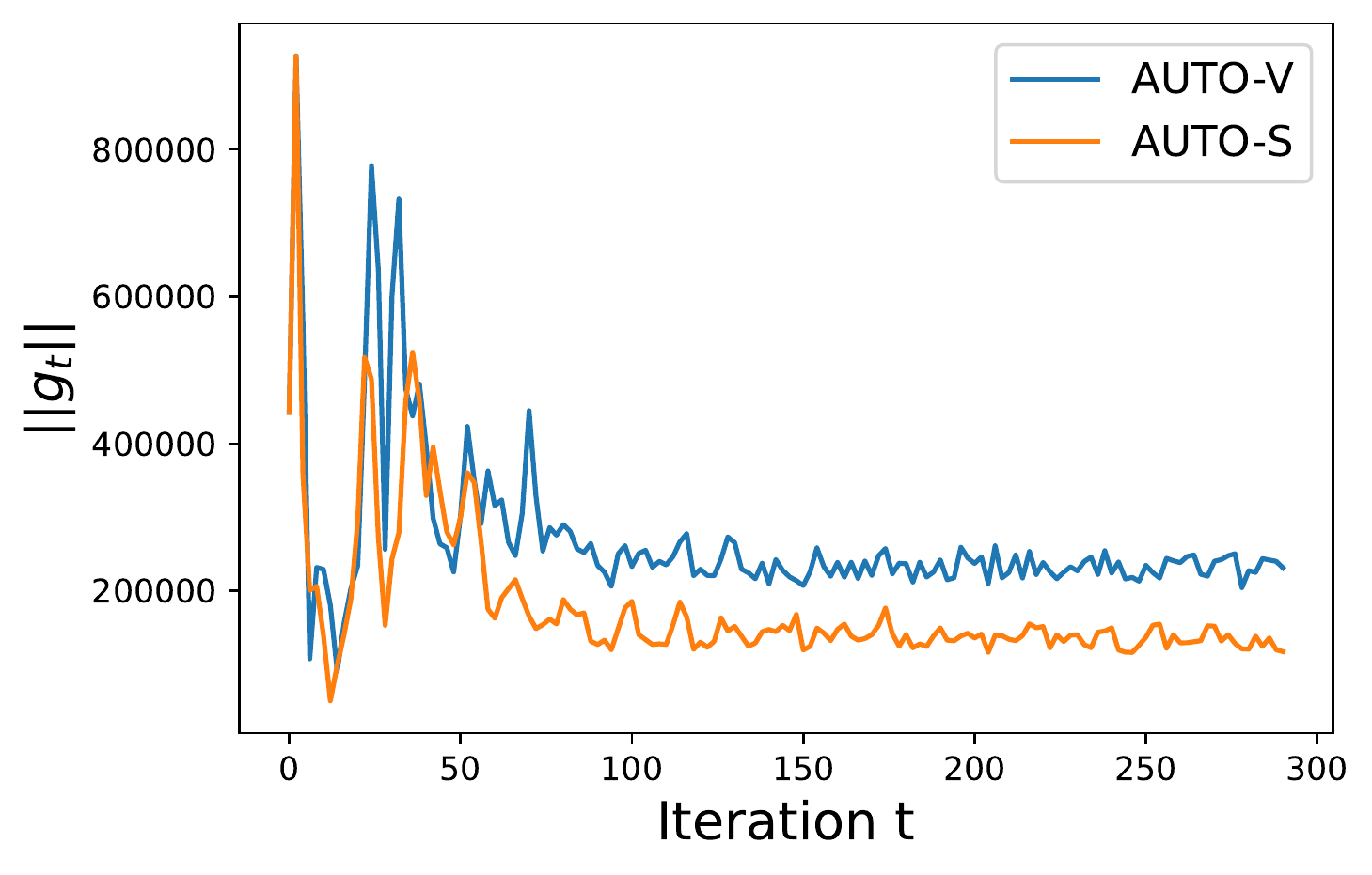}    \caption{Gradient norm by different automatic clipping methods, on SST2 (left) and MNLI (right), trained with RoBERTa-base.}
    \label{fig:grad norm reduced by stability}
\end{figure}

\subsection{Choice of stability constant is robust}
We claim in \Cref{thm: convergence DPSGD AUTO} that, as long as $\gamma>0$ in our automatic clipping, the asymptotic convergence rate of gradient norm is the same as that by standard non-private SGD. We plot the ablation study of learning rate and the stability constant $\gamma$ to show that it is easy to set $\gamma$: in \Cref{tab:sentence roberta base} and \Cref{tab:sentence roberta large}, we adopt learning rate 0.0005, under which a wide range of $0.0001<\gamma<1$ gives similar accuracy. Note that the largest good $\gamma$ is 1000 times bigger than the smallest good $\gamma$.

\begin{figure}[!htb]
    \centering
\includegraphics[width=0.38\linewidth]{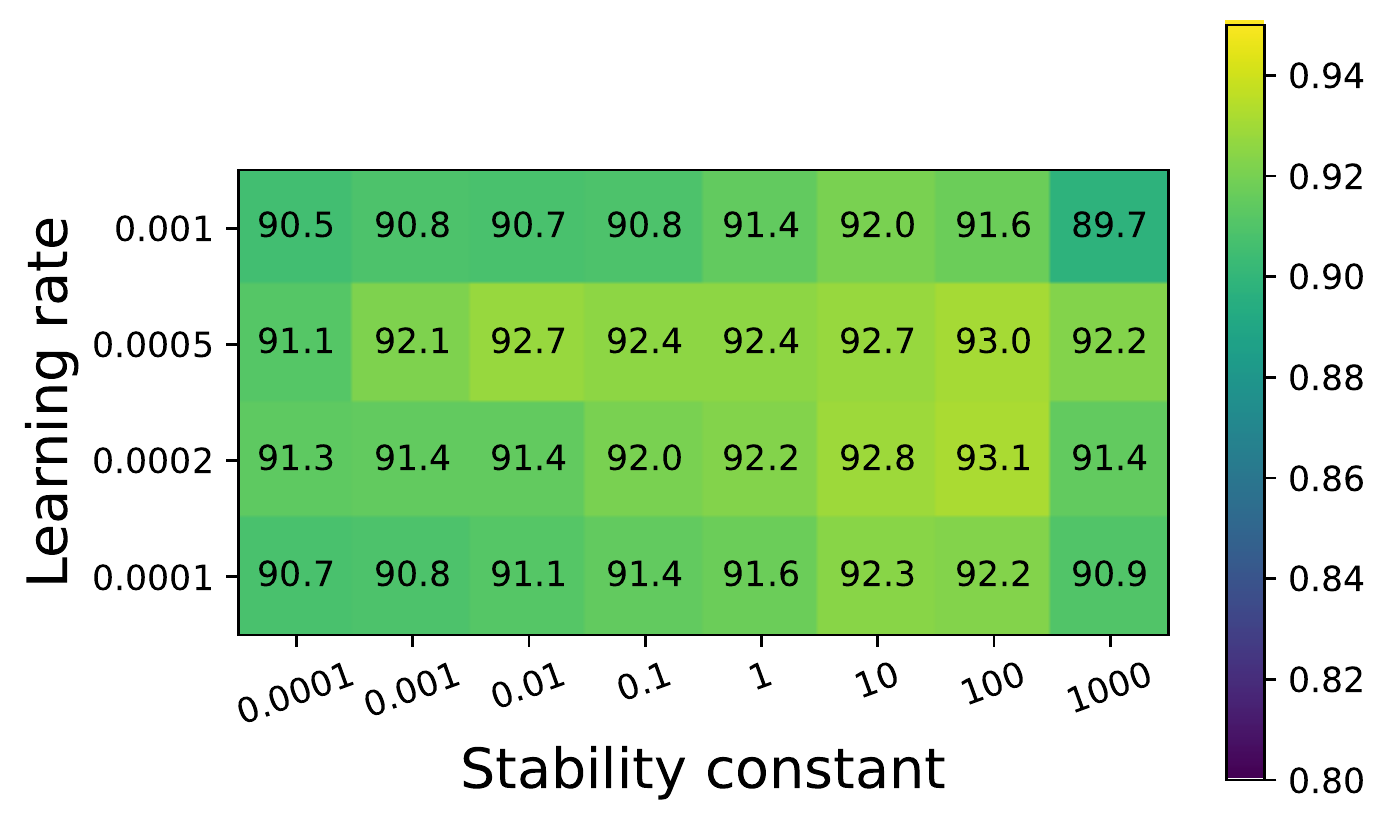}
    \includegraphics[width=0.38\linewidth]{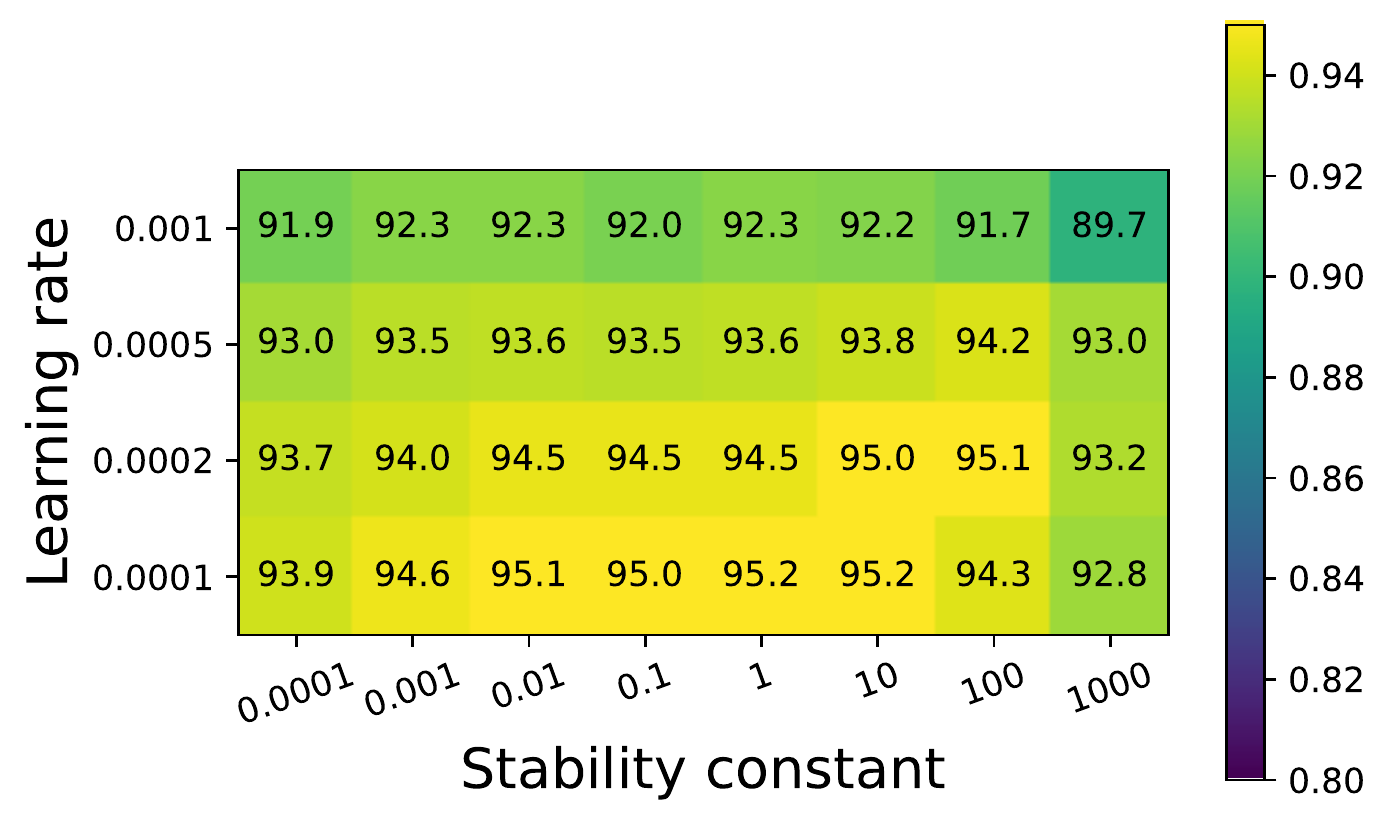}
    \includegraphics[width=0.38\linewidth]{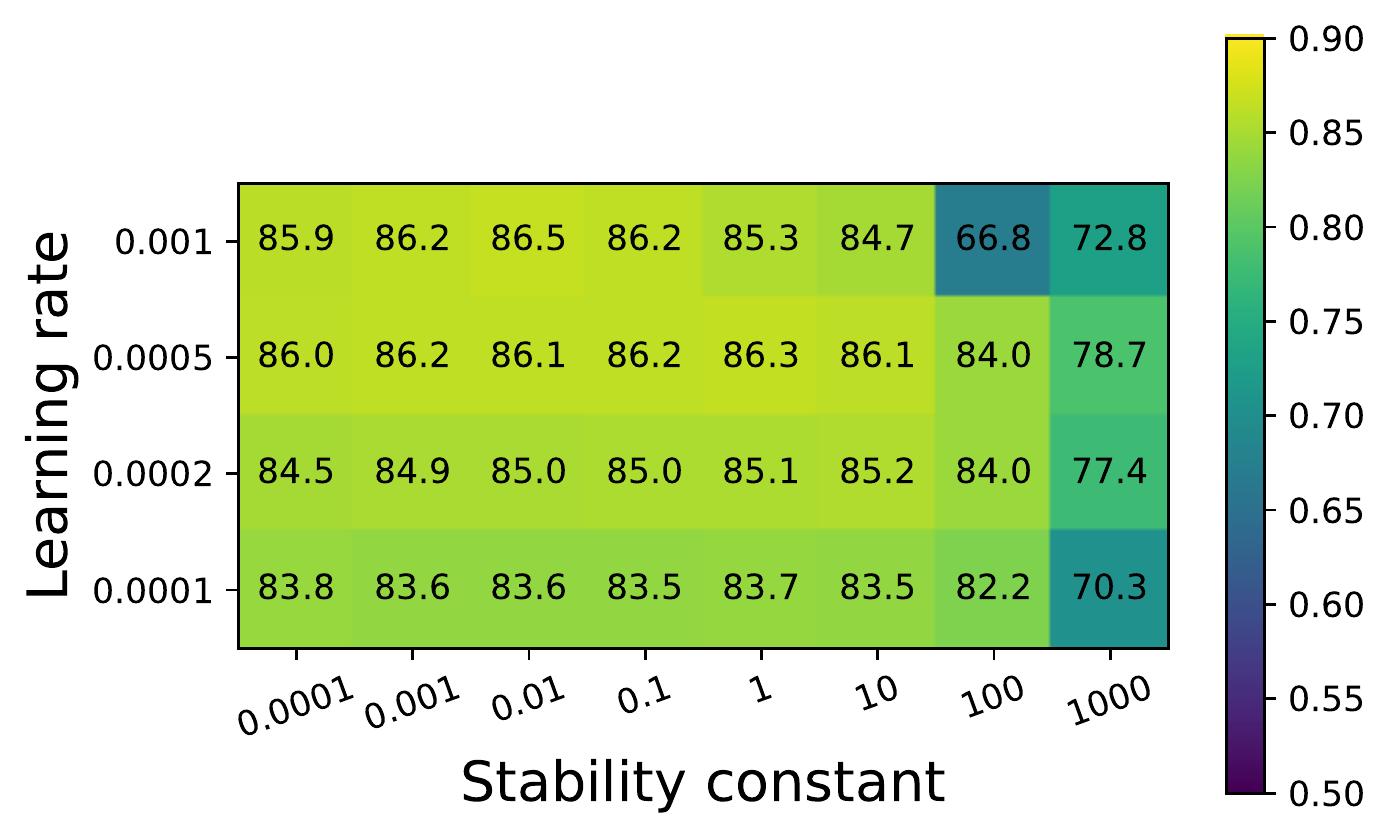}
    \includegraphics[width=0.38\linewidth]{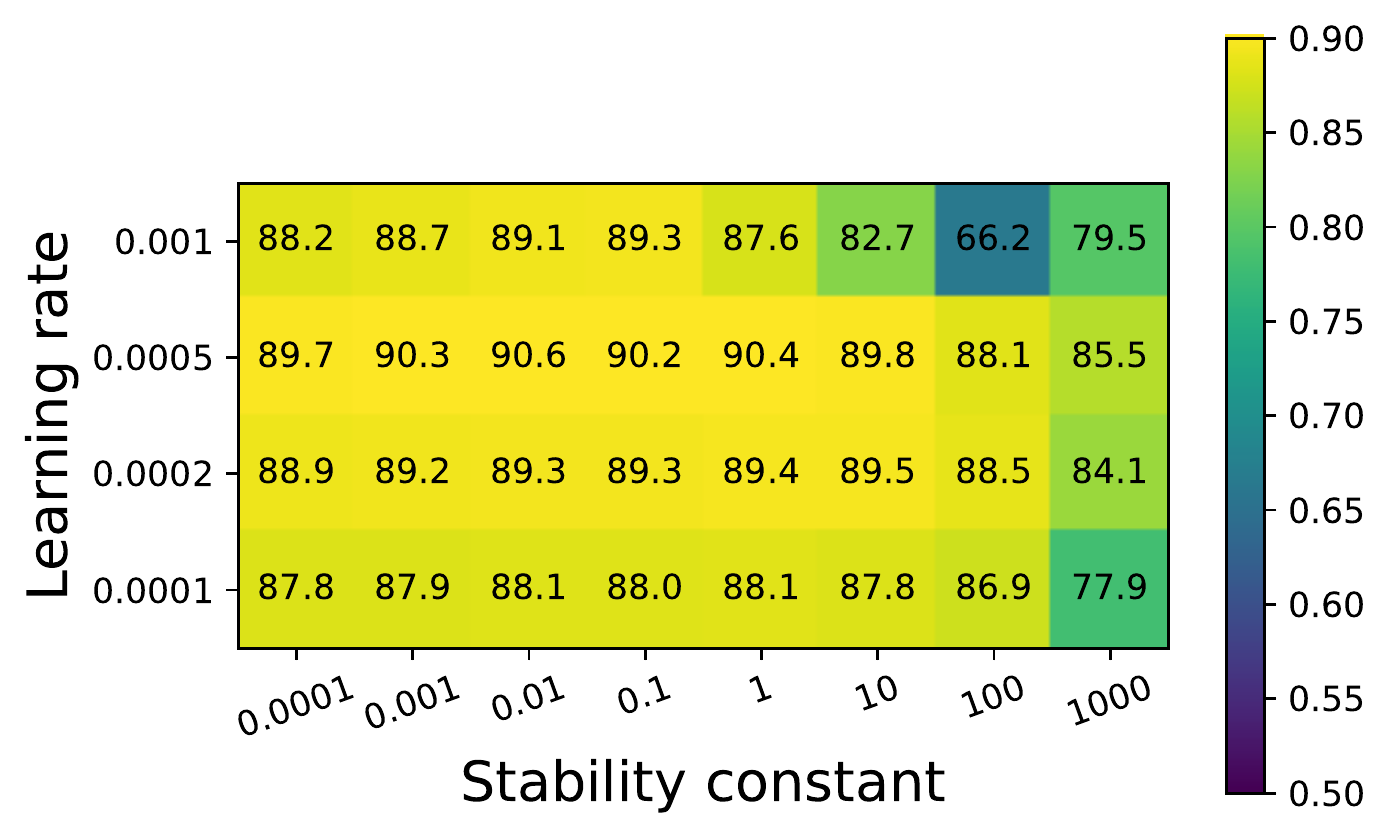}    \includegraphics[width=0.38\linewidth]{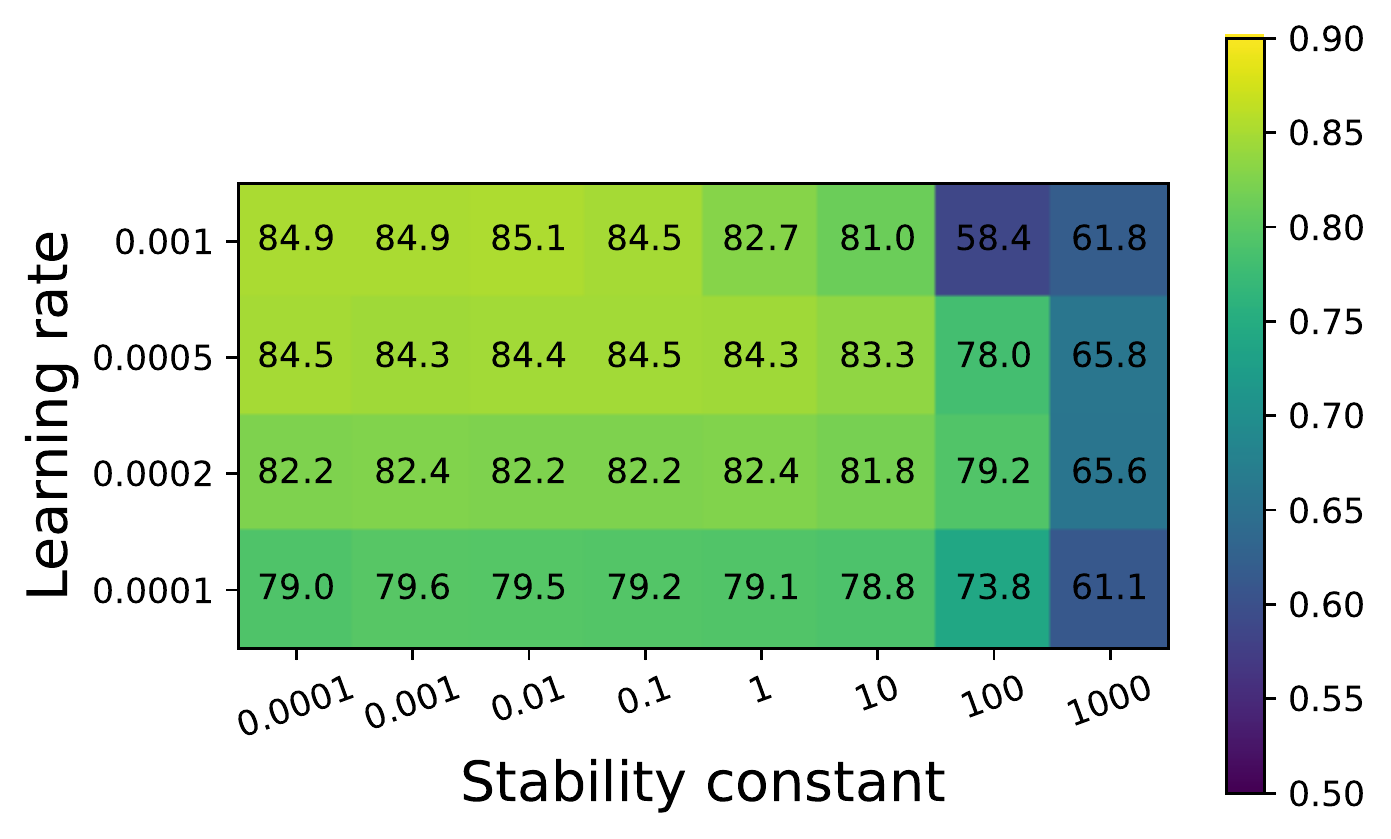}
    \includegraphics[width=0.38\linewidth]{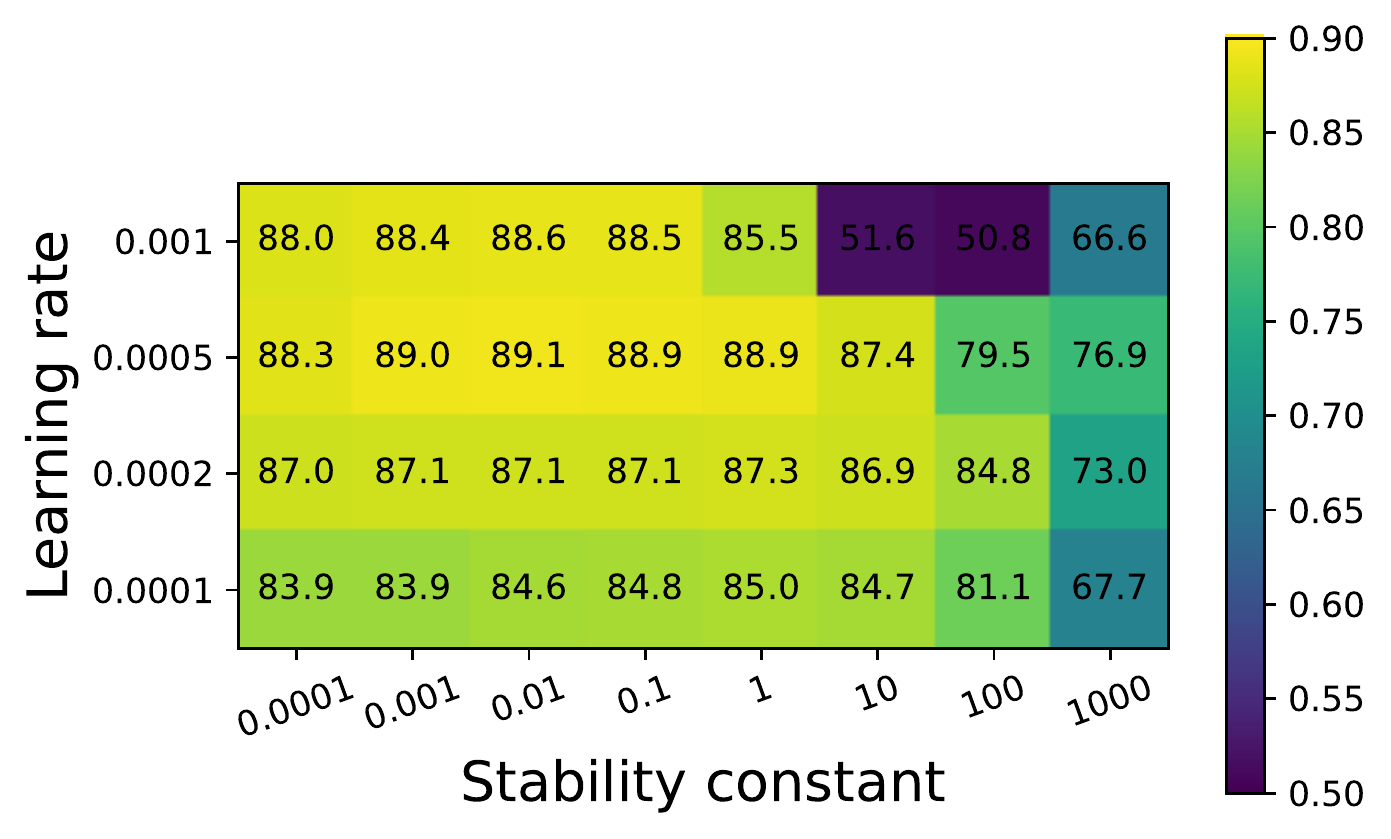}
    \caption{Test accuracy by different stability constant $\gamma$ and learning rate $\eta$ in automatic clipping ($\epsilon=3$). Upper row: SST2 for full 3 epochs. Middle row: QNLI for full 6 epochs. Lower row: QNLI for one epoch. Trained with RoBERTa-base (left) and RoBERTa-large (right).}
    \label{fig:easy to choose gamma}
\end{figure}


\section{Full table of GPT2 generation task on E2E dataset}
This is the extended version of \Cref{tab:E2E GPT selected} on E2E dataset. The performance measures are BLEU \cite{Papineni02bleu}, ROGUE-L \cite{lin-2004-rouge}, NIST \cite{sadjadi20182017}, METEOR \cite{banarjee2005}, and CIDEr \cite{vedantam2015cider} scores. Here $\epsilon$ is accounted by RDP \cite{mironov2017renyi}, where $\epsilon=3$ corresponds to 2.68 if accounted by Gaussian DP \cite{dong2019gaussian,bu2020deep} or to 2.75 if accounted by numerical composition \cite{gopi2021numerical}, and $\epsilon=8$ corresponds to 6.77 if accounted by Gaussian DP or to 7.27 if accounted by numerical composition.

\begin{table}[!htb]
    \centering
\resizebox{\linewidth}{!}{
\setlength\tabcolsep{2pt}
\begin{tabular}{|l|c|c|c|cccccccc|}
\hline
&DP&GPT2&GPT2&\multicolumn{8}{c|}{GPT2}
\\
Metric&guarantee&large&medium&&&&&&&&
\\\cline{3-12}
&&full&full&full&full&full&LoRA&RGP&prefix&top2&retrain
\\
&&AUTO-S&AUTO-S&AUTO-S&AUTO-V&\cite{li2021large}&\cite{hu2021lora}&\cite{yu2021large}&\cite{li2021prefix}&\cite{li2021large}&\cite{li2021large}
\\\hline
\multirow{3}{*}{BLEU}&$\epsilon=3$&\textbf{64.180}&\textbf{63.850}&\textbf{61.340}&\textbf{61.519}&\textbf{61.519}&58.153&58.482&47.772&25.920&15.457
\\
&$\epsilon=8$&\textbf{64.640}&\textbf{64.220}&\textbf{63.600}&63.189&63.189&\textbf{63.389}&58.455&49.263&26.885&24.247
\\
&non-DP&66.840&68.500&69.463&69.463&69.463&69.682&68.328&68.845&65.752&65.731
\\\hline\hline
\multirow{3}{*}{ROGUE-L}&$\epsilon=3$&\textbf{67.857}&\textbf{67.071}&\textbf{65.872}&65.670&65.670&\textbf{65.773}&65.560&58.964&44.536&35.240
\\
&$\epsilon=8$&\textbf{68.968}&\textbf{67.533}&\textbf{67.073}&66.429&66.429&\textbf{67.525}&65.030&60.730&46.421&39.951
\\
&non-DP&70.384&71.458&71.359&71.359&71.359&71.709&68.844&70.805&68.704&68.751
\\\hline\hline
\multirow{3}{*}{NIST}&$\epsilon=3$&\textbf{7.937}&\textbf{7.106}&\textbf{7.071}&\textbf{6.697}&\textbf{6.697}&5.463&5.775&5.249&1.510&0.376
\\
&$\epsilon=8$&\textbf{8.301}&\textbf{8.172}&\textbf{7.714}&7.444&7.444&\textbf{7.449}&6.276&5.525&1.547&1.01
\\
&non-DP&8.730&8.628&8.780&8.780&8.780&8.822&8.722&8.722&8.418&8.286
\\\hline\hline
\multirow{3}{*}{METEOR}&$\epsilon=3$&\textbf{0.403}&\textbf{0.387}&\textbf{0.387}&\textbf{0.384}&\textbf{0.384}&0.370&0.331&0.363&0.197&0.113
\\
&$\epsilon=8$&\textbf{0.420}&\textbf{0.418}&\textbf{0.404}&0.400&0.400&\textbf{0.407}&0.349&0.364&0.207&0.145
\\
&non-DP&0.460&0.449&0.461&0.461&0.461&0.463&0.456&0.445&0.443&0.429
\\\hline\hline
\multirow{3}{*}{CIDEr}&$\epsilon=3$&\textbf{2.008}&\textbf{1.754}&\textbf{1.801}&\textbf{1.761}&\textbf{1.761}&1.581&1.300&1.507&0.452&0.116
\\
&$\epsilon=8$&\textbf{2.163}&\textbf{2.081}&\textbf{1.938}&1.919&1.919&\textbf{1.948}&1.496&1.569&0.499&0.281
\\
&non-DP&2.356&2.137&2.422&2.422&2.422&2.491&2.418&2.345&2.180&2.004
\\\hline
\end{tabular}
}
\caption{Test performance on E2E dataset with GPT2. The best two GPT2 models for each row are marked in bold.}
    \label{tab:E2E GPT extended}
\end{table}

We observe that GPT2 (163 million parameters), GPT2-medium (406 million), and GPT2-large (838 million), \Cref{tab:E2E GPT selected} trained with our automatic clipping consistently perform better in comparison to other methods. In some cases, LoRA trained with Abadi's clipping also demonstrates strong performance and it would be interesting to see how LoRA trained with the automatic clipping will behave.

\section{Further experiments on CelebA dataset}
\label{app:celebA more experiments}
In this section, we present a complete summary of accuracy results, with DP constraint or not, for the CelebA dataset. We do not apply any data-preprocessing. In the first experiment, we apply a single ResNet on the 40 labels as the multi-task/multi-label learning. In the second experiment, we apply one ResNet on one label. As expected, our automatic DP optimizers have comparable test accuracy to the Abadi's DP optimizers, but we do not need to tune the clipping threshold for each individual task/label. We also notice that, learning different labels separately gives better accuracy than learning all labels together, though at the cost of heavier computational burden.
\newpage
\subsection{Multi-label classification}
We apply ResNet9 as in \Cref{app:CV settings} on the multi-label classification task. I.e. the output layer has 40 neurons, each corresponding to one sigmoid cross-entropy loss, that are summed to a single loss and all labels are learnt jointly.

\begin{table}[!htb]
    \centering
    \begin{tabular}{|c|c||c|c|c|c|c|}
         \hline
         \multirow{2}{*}{Index}&\multirow{2}{*}{Attributes}&Abadi's&AUTO-S&Abadi's&AUTO-S&non-DP\\
         &&$\epsilon=3$&$\epsilon=3$&$\epsilon=8$&$\epsilon=8$&$\epsilon=\infty$\\\hline
         0&5 o Clock Shadow&90.64&90.99$\color{green}{\uparrow}$&90.81&91.28$\color{green}{\uparrow}$&93.33\\
         1&Arched Eyebrows&75.15&76.31$\color{green}{\uparrow}$& 76.84 &77.11$\color{green}{\uparrow}$&81.52\\
         2&Attractive&75.85&76.10$\color{green}{\uparrow}$& 77.50 &77.74$\color{green}{\uparrow}$&81.15\\
         3&Bags Under Eyes&80.75&81.12$\color{green}{\uparrow}$& 82.15 &82.13$\color{red}{\downarrow}$&84.81\\
         4&Bald&97.84 &97.87$\color{green}{\uparrow}$&98.04 &97.98$\color{red}{\downarrow}$&98.58\\
         5&Bangs&92.71 &92.68$\color{red}{\downarrow}$&93.46 &93.55$\color{green}{\uparrow}$&95.50\\
         6&Big Lips&67.51&67.78$\color{green}{\uparrow}$& 68.34 &68.44$\color{green}{\uparrow}$&71.33\\
         7&Big Nose&78.01 &80.23$\color{green}{\uparrow}$&76.69 &80.59$\color{green}{\uparrow}$&83.54\\
         8&Black Hair&81.92&80.95$\color{red}{\downarrow}$& 83.33 &83.28$\color{red}{\downarrow}$&88.55\\
         9&Blond Hair&92.25 &92.38$\color{green}{\uparrow}$&93.52 &93.09$\color{red}{\downarrow}$&95.49\\
         10&Blurry&94.91 &94.82$\color{red}{\downarrow}$&95.08 &94.90$\color{red}{\downarrow}$&95.78\\
         11&Brown Hair&80.13&82.50$\color{green}{\uparrow}$& 83.74 &83.89$\color{green}{\uparrow}$&87.79\\
         12&Bushy Eyebrows&88.06&88.23$\color{green}{\uparrow}$& 89.72 &88.80$\color{red}{\downarrow}$&92.19\\
         13&Chubby&94.72&94.54$\color{red}{\downarrow}$& 94.54 &94.50$\color{red}{\downarrow}$&95.56\\
         14&Double Chin&95.19&95.49$\color{green}{\uparrow}$& 95.50 &95.51$\color{green}{\uparrow}$&96.09\\
         15&Eyeglasses&97.06&97.64$\color{green}{\uparrow}$& 98.32 &98.06$\color{red}{\downarrow}$&99.39\\
         16&Goatee&95.68&95.45$\color{red}{\downarrow}$& 95.84 &95.87$\color{green}{\uparrow}$&97.06\\
         17&Gray Hair&96.77&96.79$\color{green}{\uparrow}$& 97.02 &97.03$\color{green}{\uparrow}$&98.06\\
         18&Heavy Makeup&84.96&85.70$\color{green}{\uparrow}$& 87.58 &87.29$\color{red}{\downarrow}$&90.76\\
         19&High Cheekbones&81.46&81.42$\color{red}{\downarrow}$& 82.62 &82.72$\color{green}{\uparrow}$&86.62\\
         20&Male&92.05&92.17$\color{green}{\uparrow}$& 93.32 &93.17$\color{red}{\downarrow}$&97.46\\
         21&Mouth Slightly Open&86.20&86.32$\color{green}{\uparrow}$& 87.84 &88.48$\color{green}{\uparrow}$&93.07\\
         22&Mustache&96.05&95.96$\color{red}{\downarrow}$& 96.08 &95.99$\color{red}{\downarrow}$&96.74\\
         23&Narrow Eyes&84.90&84.78$\color{red}{\downarrow}$& 85.14 &85.18$\color{green}{\uparrow}$&86.98\\
         24&No Beard&91.55&91.67$\color{green}{\uparrow}$& 92.29 &92.45$\color{green}{\uparrow}$&95.18\\
         25&Oval Face&71.26&71.42$\color{green}{\uparrow}$& 71.98 &71.25$\color{red}{\downarrow}$&74.62\\
         26&Pale Skin&96.09&96.04$\color{red}{\downarrow}$& 96.15 &96.17$\color{green}{\uparrow}$&96.93\\
         27&Pointy Nose&70.34&72.11$\color{green}{\uparrow}$& 72.23 &73.01$\color{green}{\uparrow}$&75.68\\
         28&Receding Hairline&91.53&91.37$\color{red}{\downarrow}$& 91.75 &91.74$\color{red}{\downarrow}$&92.87\\
         29&Rosy Cheeks&93.26&93.02$\color{red}{\downarrow}$& 93.56 &93.35$\color{red}{\downarrow}$&94.86\\
         30&Sideburns&96.16&96.09$\color{red}{\downarrow}$& 96.27 &96.46$\color{green}{\uparrow}$&97.44\\
         31&Smiling&86.39&87.08$\color{green}{\uparrow}$& 88.87 &88.63$\color{red}{\downarrow}$&92.25\\
         32&Straight Hair&76.20&77.95$\color{green}{\uparrow}$& 78.78 &78.52$\color{red}{\downarrow}$&80.66\\
         33&Wavy Hair&70.30&71.79$\color{green}{\uparrow}$& 73.58 &73.19$\color{red}{\downarrow}$&79.15\\
         34&Wearing Earrings&80.53&81.52$\color{green}{\uparrow}$& 82.29 &82.20$\color{red}{\downarrow}$&87.56\\
         35&Wearing Hat&96.99&96.83$\color{red}{\downarrow}$& 97.46 &97.31$\color{red}{\downarrow}$&98.68\\
         36&Wearing Lipstick&88.95&88.04$\color{red}{\downarrow}$& 89.87 &90.72$\color{green}{\uparrow}$&93.49\\
         37&Wearing Necklace&84.59&85.83$\color{green}{\uparrow}$& 85.93 &85.42$\color{red}{\downarrow}$&86.61\\
         38&Wearing Necktie&93.91&93.91--& 94.43 &94.08$\color{red}{\downarrow}$&96.30\\
         39&Young&81.35&81.21$\color{red}{\downarrow}$& 82.18 &82.52$\color{green}{\uparrow}$&87.18\\\hline
    \end{tabular}
    \caption{Accuracy on CelebA dataset with settings in \Cref{app:CV settings} from one run. The green arrow indicates AUTO-S is better than Abadi's clipping under the same $\epsilon$; the red arrow indicates otherwise; the black bar indicates the same accuracy.}
\end{table}

\newpage
\subsection{Multiple binary classification}
For the second experiment, we apply ResNet9 on each label as a binary classification task. I.e. the output layer has 1 neuron and we run 40 different models for all labels separately.
\begin{table}[!htb]
    \centering
    \begin{tabular}{|c|c||c|c|c|c|c|}
         \hline
         \multirow{3}{*}{Index}&\multirow{3}{*}{Attributes}&Abadi's&AUTO-S&Abadi's&AUTO-S&non-DP\\
         &&Single&Single&Multi&Multi&Multi\\
         &&$\epsilon=8$&$\epsilon=8$&$\epsilon=8$&$\epsilon=8$&$\epsilon=\infty$\\\hline
         0&5 o Clock Shadow&92.15&92.29$\color{green}{\uparrow}$&90.81&91.28$\color{green}{\uparrow}$&93.33\\
         1&Arched Eyebrows&81.18&80.19$\color{red}{\downarrow}$& 76.84 &77.11$\color{green}{\uparrow}$&81.52\\
         2&Attractive&79.31&79.79$\color{green}{\uparrow}$& 77.50 &77.74$\color{green}{\uparrow}$&81.15\\
         3&Bags Under Eyes&83.52&83.48$\color{red}{\downarrow}$& 82.15 &82.13$\color{red}{\downarrow}$&84.81\\
         4&Bald&97.89 &97.88$\color{red}{\downarrow}$&98.04 &97.98$\color{red}{\downarrow}$&98.58\\
         5&Bangs&94.52 &94.83$\color{green}{\uparrow}$&93.46 &93.55$\color{green}{\uparrow}$&95.50\\
         6&Big Lips&67.32&67.53$\color{green}{\uparrow}$& 68.34 &68.44$\color{green}{\uparrow}$&71.33\\
         7&Big Nose&82.31 &82.36$\color{green}{\uparrow}$&76.69 &80.59$\color{green}{\uparrow}$&83.54\\
         8&Black Hair&87.08&86.93$\color{red}{\downarrow}$& 83.33 &83.28$\color{red}{\downarrow}$&88.55\\
         9&Blond Hair&94.29 &94.73$\color{green}{\uparrow}$&93.52 &93.09$\color{red}{\downarrow}$&95.49\\
         10&Blurry&94.95 &95.20$\color{green}{\uparrow}$&95.08 &94.90$\color{red}{\downarrow}$&95.78\\
         11&Brown Hair&87.41&87.19$\color{red}{\downarrow}$& 83.74 &83.89$\color{green}{\uparrow}$&87.79\\
         12&Bushy Eyebrows&91.23&91.43$\color{green}{\uparrow}$& 89.72 &88.80$\color{red}{\downarrow}$&92.19\\
         13&Chubby&94.70&94.70--& 94.54 &94.50$\color{red}{\downarrow}$&95.56\\
         14&Double Chin&95.43&95.43--& 95.50 &95.51$\color{green}{\uparrow}$&96.09\\
         15&Eyeglasses&98.88&99.14$\color{green}{\uparrow}$& 98.32 &98.06$\color{red}{\downarrow}$&99.39\\
         16&Goatee&96.12&96.07$\color{red}{\downarrow}$& 95.84 &95.87$\color{green}{\uparrow}$&97.06\\
         17&Gray Hair&97.48&97.34$\color{red}{\downarrow}$& 97.02 &97.03$\color{green}{\uparrow}$&98.06\\
         18&Heavy Makeup&88.85&88.72$\color{red}{\downarrow}$& 87.58 &87.29$\color{red}{\downarrow}$&90.76\\
         19&High Cheekbones&85.66&85.45$\color{red}{\downarrow}$& 82.62 &82.72$\color{green}{\uparrow}$&86.62\\
         20&Male&95.42&95.70$\color{green}{\uparrow}$& 95.53 &93.17$\color{red}{\downarrow}$&97.46\\
         21&Mouth Slightly Open&92.67&92.74$\color{green}{\uparrow}$& 87.84 &88.48$\color{green}{\uparrow}$&93.07\\
         22&Mustache&96.13&96.13--& 96.08 &95.99$\color{red}{\downarrow}$&96.74\\
         23&Narrow Eyes&85.13&85.13--& 85.14 &85.18$\color{green}{\uparrow}$&86.98\\
         24&No Beard&94.26&94.58$\color{green}{\uparrow}$& 92.29 &92.45$\color{green}{\uparrow}$&95.18\\
         25&Oval Face&70.77&73.05$\color{green}{\uparrow}$& 71.98 &71.25$\color{red}{\downarrow}$&74.62\\
         26&Pale Skin&96.38&96.34$\color{red}{\downarrow}$& 96.15 &96.17$\color{green}{\uparrow}$&96.93\\
         27&Pointy Nose&71.48&73.37$\color{green}{\uparrow}$& 72.23 &73.01$\color{green}{\uparrow}$&75.68\\
         28&Receding Hairline&91.51&91.51--& 91.75 &91.74$\color{red}{\downarrow}$&92.87\\
         29&Rosy Cheeks&93.26&93.35$\color{green}{\uparrow}$& 93.56 &93.35$\color{red}{\downarrow}$&94.86\\
         30&Sideburns&96.46&96.34$\color{red}{\downarrow}$& 96.27 &96.46$\color{green}{\uparrow}$&97.44\\
         31&Smiling&90.82&90.87$\color{green}{\uparrow}$& 88.87 &88.63$\color{red}{\downarrow}$&92.25\\
         32&Straight Hair&79.01&79.01--& 78.78 &78.52$\color{red}{\downarrow}$&80.66\\
         33&Wavy Hair&77.55&78.83$\color{green}{\uparrow}$& 73.58 &73.19$\color{red}{\downarrow}$&79.15\\
         34&Wearing Earrings&87.33&87.50$\color{green}{\uparrow}$& 82.29 &82.20$\color{red}{\downarrow}$&87.56\\
         35&Wearing Hat&98.04&98.11$\color{green}{\uparrow}$& 97.46 &97.31$\color{red}{\downarrow}$&98.68\\
         36&Wearing Lipstick&92.05&90.46$\color{red}{\downarrow}$& 89.87 &90.72$\color{green}{\uparrow}$&93.49\\
         37&Wearing Necklace&86.21&86.21--& 85.93 &85.42$\color{red}{\downarrow}$&86.61\\
         38&Wearing Necktie&95.85&95.94$\color{green}{\uparrow}$& 94.43 &94.08$\color{red}{\downarrow}$&96.30\\
         39&Young&85.19&84.12$\color{red}{\downarrow}$& 82.18 &82.52$\color{green}{\uparrow}$&87.18\\\hline
    \end{tabular}
    \caption{Accuracy on CelebA dataset with settings in \Cref{app:CV settings} from one run. `Single' means each attribute is learned separately as a binary classification task. `Multi' means all attributes are learned jointly as a multi-label classification task. The green arrow indicates AUTO-S is better than Abadi's clipping under the same $\epsilon$ and the same task; the red arrow indicates otherwise; the black bar indicates the same accuracy.}
\end{table}

\section{Code implementation of automatic clipping}\label{app:codebase}
Changing Abadi's clipping to automatic clipping is easy in available codebases. One can set the clipping $R=1$ or any other constant, as explained in \Cref{thm: non-adaptive automatic} and \Cref{thm: adaptive automatic}.

\subsection{Opacus}
For Opacus \cite{opacus} version 1.1.2 (latest), we can implement the all-layer automatic clipping by changing Line 399-401 in \url{https://github.com/pytorch/opacus/blob/main/opacus/optimizers/optimizer.py} to
\begin{verbatim}
per_sample_clip_factor = self.max_grad_norm /(per_sample_norms + 0.01)
\end{verbatim}

The per-layer automatic clipping requires changing Line 61-63 in \url{https://github.com/pytorch/opacus/blob/main/opacus/optimizers/perlayeroptimizer.py} to
\begin{verbatim}
per_sample_clip_factor =max_grad_norm / (per_sample_norms + 0.01)
\end{verbatim}

For older version ($< 1.0$, e.g. 0.15) of Opacus, we can implement the all-layer automatic clipping by changing Line 223-225 in \url{https://github.com/pytorch/opacus/blob/v0.15.0/opacus/utils/clipping.py} to
\begin{verbatim}
per_sample_clip_factor = self.flat_value / (norms[0] + 0.01)
\end{verbatim}
or implement the per-layer automatic clipping by changing Line 301-302 in \url{https://github.com/pytorch/opacus/blob/main/opacus/optimizers/perlayeroptimizer.py} to
\begin{verbatim}
per_sample_clip_factor = threshold / (norm + 0.01)
clipping_factor.append(per_sample_clip_factor)
\end{verbatim}

\subsection{ObJAX}
For ObJAX version 1.6.0 (latest), we can implement the automatic clipping in \url{https://github.com/google/objax/blob/master/objax/privacy/dpsgd/gradient.py} by changing Line 92 to
\begin{verbatim}
idivisor = self.l2_norm_clip / (total_grad_norm+0.01)
\end{verbatim}
and changing Line 145 to
\begin{verbatim}
idivisor = self.l2_norm_clip/(grad_norms+0.01) 
\end{verbatim}

\subsection{Private-transformers}
\label{app:xuechen modification}
To reproduce our experiments for sentence classification and table-to-text generation, we modify the `private-transformers' (version 0.1.0) codebase of \cite{li2021large}. The modification is in \url{https://github.com/lxuechen/private-transformers/blob/main/private_transformers/privacy_utils/privacy_engine.py}, by changing Line 349 to
\begin{verbatim}
return self.max_grad_norm / (norm_sample + 0.01)
\end{verbatim}
and Line 510-512 to
\begin{verbatim}
coef_sample = self.max_grad_norm * scale / (norm_sample + 0.01)
\end{verbatim}

\section{More on related works of per-sample clipping}
\label{app:more related}

We discuss the difference between our work and the related (see the table below).
\begin{table}[H]
\centering
\resizebox{\linewidth}{!}{
\begin{tabular}{c|c|c|c|c}
$C_i$&reference&clipping or not & convergence analysis& experiments \\\hline
$\min(1,\frac{R}{||\g_i||})$&\cite{abadi2016deep,li2021large}&clipping & None & CV and NLP
\\
$\min(\frac{1}{R},\frac{1}{||\g_i||})$&\cite{de2022unlocking}&clipping & None & CV only
\\
$\frac{R}{||\g_i||}$&\cite{das2021convergence}&normalization & convex and federated setting (not per-sample) & CV only
\\
$\frac{1}{||\g_i||+\gamma}$&\cite{yang2022normalized}&normalization & non-convex, relaxed Lipschitz smoothness & CV and NLP
\\
$\frac{1}{||\g_i||+\gamma}$&this work&normalization & non-convex, same smoothness as non-DP & CV and NLP
\end{tabular}
}
\caption{Comparison between clipping functions. CV means computer vision and NLP means natural language processing. Notice that any clipping function with $R$ is not automatic and requires tuning, and that the stability constant $\gamma$ enjoys theoretical and empirical benefits.}
\end{table}

Our work is very different to most works which do not analyze the convergence of DP deep learning in a non-convex setting, but it is very similar to \cite{yang2022normalized}\footnote{We emphasize that \cite{yang2022normalized} is a concurrent work with no known dependency either way, which goes public (to arXiv, on 27 Jun 2022) after ours (on 14 Jun 2022).}. However, \cite{yang2022normalized} assumes a relaxed Lipschitz smootheness in place of our \Cref{assumption: tilde g}, where we instead assume the symmetric gradient noise. In addition, our experiments are more comprehensive, covering over 10 tasks including DP-GPT2, while \cite{yang2022normalized} only experimented with 2 smaller models --- ResNet20 and Roberta-base. 

\subsection{Clarifications}
We now clarify some false or incomplete conclusion in previous literatures that apply the per-sample gradient clipping (re-parameterized or not).

1. Per-sample clipping is not robust to $R$, even with re-parameterization.

In \cite[Figure 8]{de2022unlocking} and our \Cref{fig:AUTO only 1D grid search}, the accuracy of DP optimizer with Abadi's clipping is insensitive to $R$ only if one has found a small enough region (e.g. $R\leq 1$), which takes effort to find or the accuracy will be unacceptably low out of the region. In particular, choosing $R=1$ as in \cite{de2022unlocking} is not universally proper, e.g. \cite{li2021large} uses $R=0.1$ for language models. This dependence on tasks, datasets and optimizers means per-sample clipping still requires the expensive hyperparameter tuning.

In other words, per-sample gradient clipping is at best an approximation of per-sample gradient normalization (i.e. our AUTO-V) and should be considered as semi-automatic, whereas AUTO-V/S is fully automatic in terms of tuning $R$. Although technically we introduce a new hyperparameter $\gamma$ in the place of $R$, we claim that automatic clipping is not sensitive to $\gamma$ (our only hyperparameter) for a large range, e.g. one can multiply $\gamma$ by 10000 times, going from $\gamma=$0.001
 to 10 with learning rate 0.0005 in \Cref{fig:easy to choose gamma}, and the accuracy is similar.

2. Per-sample clipping does not decouple $R$, especially for DP-Adam.

In general, $R$ is not completely decoupled from the re-parameterized per-sample clipping in \cite{de2022unlocking}:
$$C_\text{Abadi}=\min(\frac{R}{\|\g_i\|},1)=R\cdot C_\text{re-param}=R\cdot \min(\frac{1}{\|\g_i\|},\frac{1}{R})$$
Given that $R$ appears in both terms on the right hand side, one can at most say "... when the clipping norm is decreased k
times, the learning rate should be increased k times to maintain \textit{similar} accuracy." by \cite{kurakin2022toward} and "... Using this update, performance becomes \textit{less sensitive} to the choice of clipping norm." by \cite{de2022unlocking}. In contrast, we can state that adjusting the learning rate proportionally, our AUTO-V/S maintains \textit{exactly the same} accuracy and is \textit{completely insensitive} to the choice of $R$.

Additionally and importantly, the understanding in \cite{kurakin2022toward,de2022unlocking} is limited to DP-SGD (as they only experiment with the computer vision tasks), where "... the learning rate $\eta$ absorbs a factor of $R$." by \cite{de2022unlocking}. As rigorously proved in \Cref{thm: non-adaptive automatic} and \Cref{thm: adaptive automatic}, adaptive optimizers like Adam and AdaGrad do not absorb $R$ but rather cancel it. This is visualized in \Cref{fig:my motivation}, where the performance landscape is row-wise for DP-Adam and diagonal for DP-SGD.

3. Re-parameterized per-sample clipping unintentionally changes the weight decay.

Weight decay is a common technique used in any work that uses AdamW and in the re-parameterized trick by \cite{de2022unlocking}. We can see that 
$$\text{Before re-parameterization: }w_{t+1}=w_t-\eta\left(\frac{1}{B}\sum_i\min(1,\frac{R}{||g_i||})g_i+\lambda w_t+\frac{\sigma R}{B}N(0,I)\right)$$

$$\text{After re-parameterization: }w_{t+1}=w_t-\eta\left(\frac{1}{B}\sum_i\min(\frac{1}{R},\frac{1}{||g_i||})g_i+\frac{\lambda}{R} w_t+\frac{\sigma}{B}N(0,I)\right)$$

Therefore, when we move along $R$ like in \cite[Figure 8]{de2022unlocking}, from $R=1$ to $2^{-6}$, the weight decay increases from $\lambda$ to $2^6\cdot\lambda$ by 64 times, which may worsen the accuracy as seen in the blue curve of \cite[Figure 8]{de2022unlocking}! Again, this is due to the incomplete decoupling by per-sample clipping, which is only avoided in AUTO-V/S thanks to theoretical analysis in \Cref{thm: non-adaptive automatic} and \Cref{thm: adaptive automatic}.
\begin{align*}
\text{AUTO-V/S with weight decay: }w_{t+1}&=w_t-\eta\left(\frac{1}{B}\sum_i \frac{1}{\|g_i\|+\gamma}g_i+\lambda w_t+\frac{\sigma}{B} \mathcal{N}(0,I)\right).
\end{align*}

\subsection{Connections to normalized optimisation}

Variants of normalized gradient have been used in optimization \cite{mandic2004generalized,murray2019revisiting,zhao2021convergence,zhao2020stochastic,cutkosky2020momentum}. These normalized optimizers are fundamentally different to our automatic optimizers, because the normalization is on mini-batch not on each sample and noise is not involved:
\begin{align*}
\text{NSGD: }w_{t+1}&=w_t-\eta\left( \frac{\frac{1}{B}\sum_i g_i}{\|\frac{1}{B}\sum_i g_i\|}\right)
\\
\text{AUTO-V: }w_{t+1}&=w_t-\eta\left(\frac{1}{B}\sum_i \frac{g_i}{\|g_i\|}+\frac{\sigma}{B} \mathcal{N}(0,I)\right).
\end{align*}
The main difference lies in the challenge of analyzing per-sample normalization (which is biased) and the batch-gradient normalization (which is unbiased in the direction). That is,
$\frac{\frac{1}{B}\sum_i g_i}{\|\frac{1}{B}\sum_i g_i\|}$ is parallel to the mini-batch gradient $\frac{1}{B}\sum_i g_i$ but $\frac{1}{B}\sum_i \frac{g_i}{\|g_i\|}$ is generally not parallel to it (this conclusion also holds if the normalizaiton is replaced by the clipping). On a side note, it is interesting that \Cref{thm:upper bounding grad norm without r} indeed shows although a bias is introduced by the per-sample clipping, it is not fatal to the asymptotic convergence and hence may not be a concerning matter.

\end{document}